\documentclass[twoside]{article}

\usepackage{hyperref}

\usepackage[accepted]{aistats2026}

\usepackage[utf8]{inputenc} %
\usepackage[T1]{fontenc}    %
\usepackage{lmodern}
\usepackage{booktabs}       %
\usepackage{amsfonts}       %
\usepackage{nicefrac}       %
\usepackage{microtype}      %
\usepackage{xcolor}         %
\usepackage{amsthm}
\usepackage{amsmath}
\usepackage{extarrows}
\usepackage{enumitem}
\usepackage{verbatim}
\usepackage{wrapfig}
\usepackage{pifont}
\usepackage[textsize=tiny]{todonotes}
\usepackage{amssymb}
\usepackage{thmtools}
\usepackage{thm-restate}
\usepackage{notations}
\usepackage{placeins}
\usepackage{subcaption}
\usepackage{makecell}
\usepackage{float}
\usepackage{url}
\theoremstyle{plain}
\newtheorem{theorem}{Theorem}[section]
\newtheorem{proposition}[theorem]{Proposition}
\newtheorem{lemma}[theorem]{Lemma}

\theoremstyle{definition}
\newtheorem{definition}[theorem]{Definition}

\newtheorem{example}[theorem]{Example}
\theoremstyle{remark}

\usepackage[capitalize,noabbrev]{cleveref}
\crefname{example}{Example}{Examples}
\crefname{theorem}{Theorem}{Theorems}
\crefname{lemma}{Lemma}{Lemmas}
\Crefname{lemma}{Lemma}{Lemmas}
\crefname{remark}{Remark}{Remarks}
\crefname{proposition}{Proposition}{Propositions} 
\crefname{corollary}{Corollary}{Corollaries}
\crefname{definition}{Definition}{Definitions}
\crefname{assumption}{Assumption}{Assumptions}

\newcommand{\eqd}{\overset{d}{=}}

\usepackage[round]{natbib}

\begin{document}

\twocolumn[

\aistatstitle{Identifiability of Potentially Degenerate Gaussian Mixture Models With Piecewise Affine Mixing}

\aistatsauthor{ Danru Xu \And S\'ebastien Lachapelle \And  Sara Magliacane}

\aistatsaddress{ University of Amsterdam \And  Samsung AI Lab, Montreal \And Saarland University \& \\ University of Amsterdam} ]

\begin{abstract}
  Causal representation learning (CRL) aims to identify the underlying latent variables from high-dimensional observations, even when variables are dependent with each other. We study this problem for latent variables that follow a potentially degenerate Gaussian mixture distribution and that are only observed through the transformation via a piecewise affine mixing function. We provide a series of progressively stronger identifiability results for this challenging setting in which the probability density functions are ill-defined because of the potential degeneracy. For identifiability up to permutation and scaling, we leverage a sparsity regularization on the learned representation. Based on our theoretical results, we propose a two-stage method to estimate the latent variables by enforcing sparsity and Gaussianity in the learned representations. Experiments on synthetic and image data highlight our method’s effectiveness in recovering the ground-truth latent variables.
\end{abstract}

\section{INTRODUCTION}
\label{sec: introduction}

Recovering latent variables
from high-dimensional observations, such as images, videos, or text, is essential for building reliable and interpretable 
models \citep{bengio2013representation}. While traditional approaches, e.g., (nonlinear) ICA \citep{COMON1994,hyvarinen2013independent,hyvarinen2016unsupervised,hyvarinen2017nonlinear}, identify independent or conditionally independent variables, 
in many real-world settings latent variables often exhibit complex dependencies, e.g., because they are causally related.

Causal representation learning (CRL)~\citep{scholkopf2021toward} aims to identify latent variables with arbitrary dependencies from observations. 
Many CRL works consider additional data or information to identify causal variables, e.g., interventions \citep{brehmer2022weakly,lippe2022citris, varici2023score,von2023nonparametric, buchholz2023learning, zhang2023identifiability, ahuja2023interventional}, actions or temporal structure \citep {lippe2023biscuit,lachapelle2022disentanglement, lachapelle2024nonparametric}, multi-view structures \citep{von2021self,ahuja2023multi, yao2023multi,morioka2023causal}, grouping of observations by sparsity patterns \citep{xu2024sparsity}, hierarchical structures \citep{kong2023identification,kong2024learning} or conditions on the causal graphs \citep{zhang2024causal,dong2024on,song2024causal}.

In our case, we do not assume any additional data or information besides the observations, but we instead consider a set of parametric assumptions on the latent variables $\Zb$ and mixing function $\fb$ through which we get the observation $\Xb=\fb(\Zb)$. In particular, we focus on latent variables that follow a mixture of Gaussian distributions with a piecewise affine mixing function.
Similarly to us, \citet{kivva2022identifiability} studies the identifiability of non-degenerate Gaussian Mixture Model (GMM) latent variables with piecewise affine mixing, but to achieve element-wise identifiability they require that all pairs of latent variables must be conditionally independent given an (unobserved) auxiliary variable. 

We achieve the same identifiability result without requiring this auxiliary variable, while extending the setting to latent variables that follow a potentially degenerate Gaussian Mixture Model (pdGMM), i.e., a GMM in which the components can be potentially degenerate Gaussians, but we instead require a sparsity assumption.
This is inspired by the success of sparse (low-rank, i.e., degenerate) representations.
In particular, in many real-world settings, mixtures of low-dimensional subspaces provide a more accurate model than full-rank representations \citep{ldrdd2025}.

We provide a series of progressively stronger identifiability results under reasonable assumptions in this challenging setting, in which the probability density functions are ill-defined and hence  previous results cannot be applied. As a final step, we leverage a sparsity regularization on the learned representation, inspired by multi-task learning \citep{lachapelle2023synergies} and partially observable CRL \citep{xu2024sparsity} to prove identifiability up to permutation, scaling and translation, i.e., that we learn a disentangled representation of the latent variables.
We highlight our contributions:

\begin{itemize}[noitemsep, topsep=0pt]
\item We provide an identifiability result for pdGMMs, showing that the full distribution can be uniquely identified from an open subset that intersects the support of every component (Thm.~\ref{thm:iden-pdgmm-openset}). This is a critical step in our other results, but is also of independent interest due to its proof strategy;
\item We introduce a series of identifiability results for recovering latent variables that follow a pdGMMs distribution from observations (Thm.~\ref{thm:affine-identi-comp}, ~\ref{thm:global-affine-identi}, ~\ref{thm: PD identi}), showing that under progressively stronger assumptions, latent variables can be recovered up to (i) affine transformation within components (ATwC), (ii) global affine transformation (AT), and finally, (iii) permutation and scaling (PS), without interventions or supervision;

\item Finally, we propose a two-stage algorithm that implements these theoretical results and evaluate it on numerical and image datasets.
\end{itemize}

\section{PROBLEM SETTING AND BACKGROUND}
\label{sec: Problem set up}

We consider a setting with latent random variables $\Zb=(Z_1,...,Z_n) \in \Zcal \subseteq \RR^n$ that have arbitrary dependences with each other. Given only a set of \emph{observations} $\Xb \in \Xcal \subseteq \RR^d$ generated by an unknown mixing function $\fb: \Zcal \rightarrow \Xcal$, our goal is to recover the latent variables $\Zb$.
In general, we cannot fully recover these latent variables, but we can identify them up to an equivalence class. We first summarize two types of identifiability presented in previous work \citep{COMON1994, khemakhem2020variational,kivva2022identifiability, lachapelle2023synergies}. 
\begin{definition}
\label{def: idents}
Given an observation $\Xb = \fb(\Zb)$, a learned representation $\gb(\Xb)$ identifies 
$\Zb$ \emph{up to permutation and scaling (PS)} when there exists a permutation matrix $P$ and an invertible
diagonal matrix $D$ %
such that $\gb(\Xb)= PD\Zb$ almost surely. A weaker notion of identifiability is \emph{up to a global affine transformation (AT)}, which instead means that there exists a global invertible \emph{affine transformation} $\hb$ such that $\gb(\Xb)= \hb(\Zb)$ almost surely.
\end{definition} %

In this paper, we assume that the observations $\Xb$ are generated by an injective continuous piecewise affine mixing function $\fb$ (Def.~\ref{def: continue piecewise affine}). We also assume that $\Zb$ follows a \emph{potentially
degenerate Gaussian mixture model} (pdGMM) with unknown parameters, i.e., a Gaussian mixture model that explicitly allows for degenerate components (with singular covariance matrices). Gaussian Mixture Models (GMM) are a popular method in clustering. While the standard definition also allows for degenerate components, in practice there is often the assumption that the components are not degenerate. Since in this paper we focus on this setting, we make this choice explicit by slightly modifying their name and defining them as follows:

\begin{restatable}{definition}{defpdgmm}
\label{def: reduce gmm}
 A variable $\Zb$ follows a potentially degenerate Gaussian mixture model (pdGMM) if its distribution is $ P= \sum_{j=1}^J \lambda_j N(\mub_j, \Sigmab_j)$, where for each component $j\in [J]$ the mean is $\mub_j\in \RR^n$ and the covariance matrix is $\Sigmab_j\in \RR^{n\times n}$, such that determinant $|\Sigmab_j|\ge 0$. We assume that the weights $\lambda_j>0$ and $\sum_{j=1}^J \lambda_j=1$. If each component is distinct, i.e., for every $i \neq j \in [J]$ we have $(\mub_i,\Sigmab_i)\neq (\mub_j,\Sigmab_j)$, then we say that the pdGMM is in a \emph{reduced form}.
 \end{restatable}

Other works have already studied latent variables that follow a non-degenerate GMM distribution and considered the problem of identifying these variables from observation. In particular,
\citet{kivva2022identifiability} provide the identifiability up to AT of non-degenerate GMMs latent variables with piecewise linear mixing functions.  
For a stronger notion of identifiability,  identifiability up to PS, \citet{kivva2022identifiability} require that all pairs of latent variables must be conditionally independent given an auxiliary variable $U$, which is an assumption we will not make. 
Crucially, these results rely on the \emph{analyticity} of the probability density function (pdf) of Gaussian distributions, which does not hold for degenerate Gaussians, since their pdf is no longer well-defined. This means we cannot apply these results directly.

In this challenging setting, we instead leverage results from a different line of work focusing on sparsity
\citep{lachapelle2023synergies,xu2024sparsity} to prove the identifiability up to PS and a constant translation for piecewise affine mixing functions. \citet{lachapelle2023synergies} focus on multi-task learning with sparse task-specific predictors, while \citet{xu2024sparsity} focus on partially observable causal representation learning settings, where each observation only captures a small number of ``active'' latent variables. 
While these works focus on sparsity, their motivation is related to our choice to allow for potential degenerate representations, since sparse (i.e., low-rank) representations can be seen as representations in which some of the components are degenerate.
In many real-world settings we have (nearly)degenerate structures, such as high collinearity or strong dependencies among features, or intrinsically low-rank manifolds. This is typical for high-dimensional data that require large numbers of sparse features such as language models \citep{cunningham2023sparse,gao2024scaling,marks2024sparse}.

\section{IDENTIFIABILITY OF PDGMMS}
\label{sec: theorems}

\begin{figure}[t]
\centering
\includegraphics[width=0.25\textwidth]{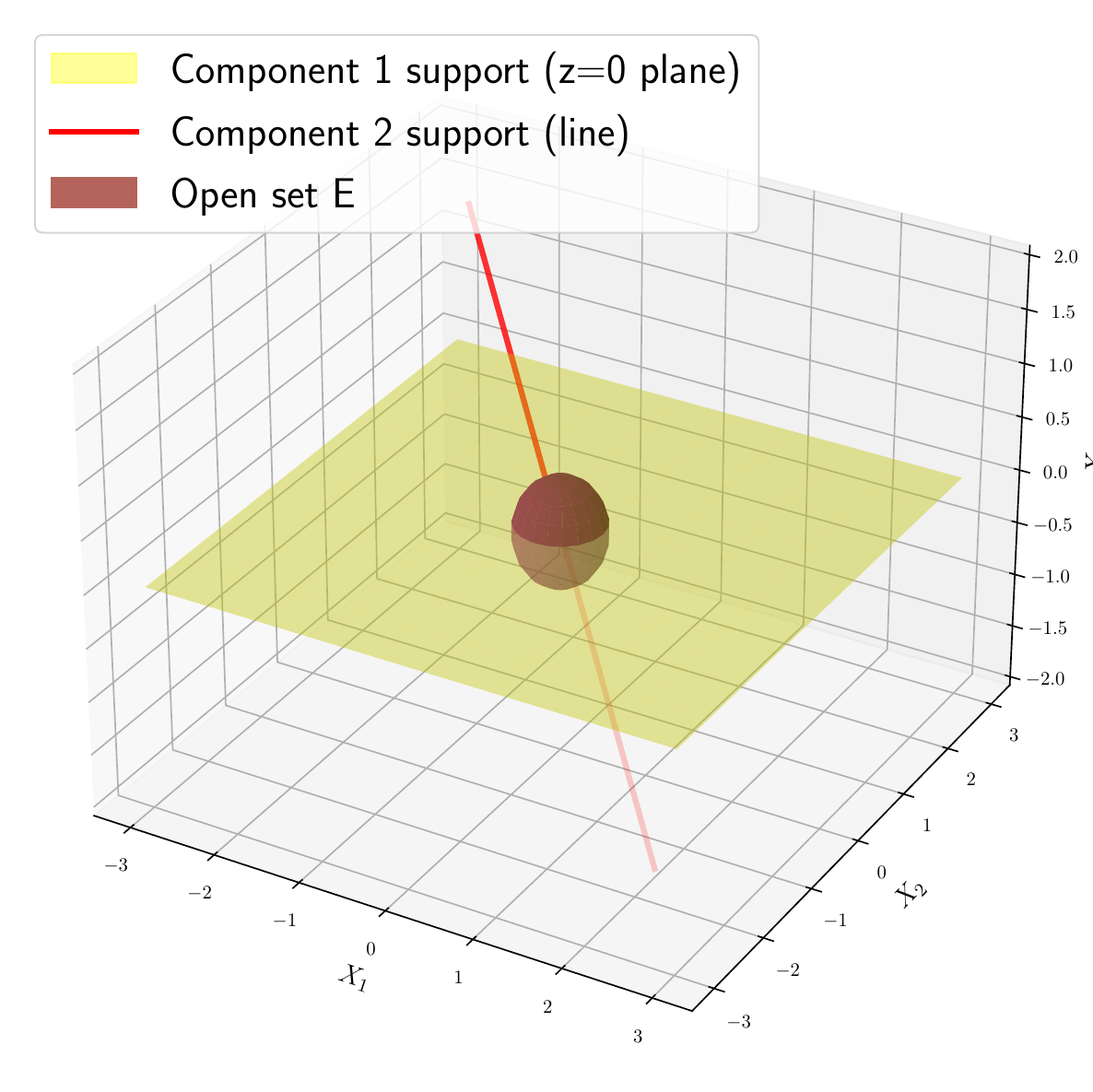}
\caption{
Example of a 3-dimensional pdGMM with 2 components: a rank-2 component with support represented by the yellow plane, and a rank-1 component with support represented by the red line. Knowing the distribution on the open set $E$ (in brown), intersecting both components, is sufficient to identify the pdGMM. 
}
\label{fig:visualization-thm2.3}
\end{figure}

In this section, we provide a series of theoretical results that prove progressively stronger notions of identifiability for latent variables following a potentially degenerate Gaussian Mixture Model (pdGMM) distribution with a piecewise affine mixing function. We provide all proofs in \cref{app: proof}. 
As a first step, we focus on the case in which the pdGMM variables are observed and prove that if two pdGMM variables have the same distribution on an open set, then they have the same distribution over all of $\mathbb{R}^n$, i.e., they are \emph{identified}. The intuition is illustrated in the example in Fig.~\ref{fig:visualization-thm2.3}. In addition to being of independent interest, this result is an important step for our other results.

\subsection{Identifiability from an open subset}
\label{subsec: identi pdGMM from open set}

Identifiability results for (pd)GMMs from the whole domain are well-established \citep{yakowitz1968identifiability}, proving that if two (pd)GMMs are the same on their domain, then they have the same components and weights up to a permutation. These results include the potentially degenerate case. On the other hand, in our quest to prove the identifiability of \emph{latent variables} that follow a pdGMM only from observations, we need an identifiability result for observed pdGMMs on a \emph{subdomain}, and more precisely, an open subset. We highlight this result in this subsection, because it is of independent interest beyond the CRL setting.

Since we cannot rely on equality of pdfs because of potential degeneracy, we first need a definition of when two distributions are considered equal on a subdomain. 

\begin{restatable}%
{definition}{defeqd}
\label{def: eqd}
    Let $\mu: \mathcal{F} \rightarrow [0,1]$ and $\mu': \mathcal{F} \rightarrow [0,1]$ be the probability measures induced by $\Xb$ and $\Xb'$, respectively, where $\mathcal{F}$ is the Borel sigma-algebra over $\mathbb{R}^n$. 
    Given a subset $E \in \Fcal$, we define the restriction of $\mu$ to $E$, denoted by $\mu_E: \Fcal \rightarrow [0,1]$, as 
    $\mu_E(A) := \mu(A \cap E), \forall A \in \mathcal{F}$.
    We say that $\Xb$ and $\Xb'$ are equal in distribution over a subdomain $E$, denoted as $\Xb \eqd \Xb'$ on $E$, when $\mu_E = \mu'_E$.
\end{restatable}

We show that the equality of two pdGMMs on an open subset of $\RR^n$ that intersects the support of every Gaussian component in a pdGMM implies the equality of pdGMMs on the whole domain.

\begin{restatable}[Identifiability of pdGMMs from open set]{theorem}{thmidenpdgmmopenset}
\label{thm:iden-pdgmm-openset}
Consider two pdGMMs in $\RR^n$ in reduced form,
\begin{equation}
    \Xb \sim \sum_{j=1}^J \lambda_j N(\mub_j, \Sigmab_j)\quad \text{and}\quad \Xb' \sim \sum_{j=1}^{J'} \lambda_j' N(\mub_j', \Sigmab_j'), 
\end{equation}
where $rank(\Sigmab_j)>0, \forall j\in[J]$ and $rank(\Sigmab'_j)>0, \forall j\in[J']$.
Let $\Xcal_j$ be the supports of $N(\mub_j, \Sigmab_j)$. If there exists an open set $E$ such that for all $j\in[J]$, $\Xcal_j\cap E\neq \emptyset$ and 
$\Xb\eqd\Xb'$ on $E$, then $\Xb\eqd\Xb'$ on the whole domain, $J=J'$, and there exists a permutation $\pi: [J] \rightarrow [J']$ such that $(\lambda_j,\mub_j, \Sigmab_j)=(\lambda'_{\pi(j)},\mub'_{\pi(j)}, \Sigmab'_{\pi(j)})$. %
\end{restatable}
We provide the proof in App.~\ref{app: identi pdgmm openset} and describe a proof sketch here. The main challenge is the ill-defined pdf of pdGMMs.  \cite{kivva2022identifiability} provides similar results for non-degenerate GMM, but their proof strategy relies on the analyticity of the pdf. However, in the degenerate case, the density does not exist on $\RR^n$, so their proof cannot be applied. 
We overcome this issue by projecting a pdGMM in $\RR^n$ into a sequence of lower-dimensional spaces $\RR^k$, where $k\in [n]$, such that each degenerate Gaussian component becomes non-degenerate in at least one image space. This enables us to apply classical identifiability results to recover the parameters of Gaussian components grouped by their ranks \citep{yakowitz1968identifiability,kivva2022identifiability}.

To make this strategy work, the projections must satisfy two key properties. First, they must preserve the rank of each component’s covariance matrix, thereby maintaining the intrinsic structure and information of the components. 
Second, a single projection is insufficient to fully recover the parameters of a pdGMM from lower-dimensional spaces; 
instead, a sufficiently rich set of rank-preserving projections is required to capture all distinguishable parameterizations. We therefore choose to work with a dense subset of the space of projections, and prove that such set of projections always exists.

\subsection{Identifiability of latent pdGMMs}
\label{subsec: latent identi}
In this section, we introduce our main results (Thm.~\ref{thm:affine-identi-comp}, \ref{thm:global-affine-identi}, and \ref{thm: PD identi}), which form a set of increasingly stronger assumptions on the distribution of $\Zb$, each yielding a correspondingly stronger identifiability.
As a first step we introduce a weak notion of identifiability for mixture-based latent variables (i.e., the distribution is a mixture of multiple components), \emph{identifiability up to affine transformations within components (ATwC)}.
This notion resembles the identifiability up to an affine transformation (AT) from literature, but it provides even less guarantees, since the recovered variables identify the ground truth latent variables only up to an affine function $\hb$ within the support of each component, but without any guarantee that $\hb$ is affine across all components, thus sacrificing the interpretability of the learned representations. 
We provide \cref{exam: local affine to global} to illustrate the difference between \emph{ATwC} and \emph{AT}.
Formally, we define identifiability up to ATwC as:

\begin{definition}
\label{def: ident_comp}
The ground truth mixture-based representation vector $\Zb$ is said to be \emph{identified up to affine transformations within components (ATwC)} by a learned representation vector %
$\gb(\Xb)$ when there exists an invertible transformation $\hb$, such that $\gb(\Xb)= \hb(\Zb)$ almost surely and $\hb$ is \emph{affine} on each $\Zcal_j$, where $\Zcal_j$ is the support of $j$-th component.
\end{definition}

To prove identifiability up to ATwC, we introduce a genericity condition on pdGMMs.
We consider pdGMMs for which the supports of components with the same rank $k$ may overlap, but there is at least one common point in the intersection where the components can be distinguished, i.e., %
their projections onto the non-degenerate subspace do not have the same norm.
This rules out special cases where two components are geometrically indistinguishable in their overlapping region, which we prove in App.~\ref{discussion:genericity} is a measure zero set in the parameter space under consideration.
Intuitively, this means that this assumption only rules out highly symmetric or finetuned cases.
We formalize it as follows:

\begin{restatable}[Genericity of pdGMMs]{assumption}{assequalnorm}
\label{ass: no equ norm}
Consider a pdGMM in reduced form $\Zb \sim \sum_{j=1}^J \lambda_j N(\mub_j, \Sigmab_j)$. 
Let $\Jcal_k:=\{j\in [J]: rank(\Sigmab_j)=k)\}$ be the indices of the components with rank $k$. For any $k\leq n$ and any subset $\Jcal_k^0\subseteq \Jcal_k$, suppose that the intersection of the supports of the components $\Zcal_j$ is non-empty, i.e., $\cap_{j\in \Jcal_k^0}\Zcal_j\neq \emptyset$. 
We assume there exists at least one point $\zb_0\in \cap_{j\in \Jcal_k^0}\Zcal_j$ s.t. the Mahalanobis distance \footnote{Mahalanobis distance of a point $\zb_0$ from a distribution $N(\mub_j, \Sigmab_j)$ is $\sqrt{(\zb_0-\mub_j)^T\Sigmab_j^{-1}(\zb_0-\mub_j)}$. When $\Sigmab_j$ is singular, we use the Moore-Penrose inverse for $\Sigmab_j^{-1}$} of $\zb_0$ to the distributions of any two distinct Gaussian components $i,j\in \Jcal_k^0$ is not the same.

\end{restatable}

Based on Ass.~\ref{ass: no equ norm}, we show that we can learn a pdGMM representation up to \emph{ATwC} with a perfect reconstruction and enforcing Gaussianity on representations. %

\begin{restatable}[Identifiability up to ATwC of pdGMM latent variables]{theorem}{thmaffineidenticomp}
\label{thm:affine-identi-comp}
Assume that $\Zb$ is a pdGMM and the observation $\Xb = \fb(\Zb)$ follows the data-generating process in Sec.~\ref{sec: Problem set up} for an injective continuous piecewise affine mixing function $\fb$, and $\Zb$ %
satisfies Ass.~\ref{ass: no equ norm}.
Let $\gb: \Xcal \rightarrow \mathbb{R}^n$ be an invertible continuous piecewise affine function and $\hat\fb: \mathbb{R}^n \rightarrow \mathbb{R}^d$ be an invertible piecewise affine function onto its image. 
If both conditions hold,
\begin{align}
    &\mathbb{E}\norm{\Xb - \hat\fb(\gb(\Xb))}^2_2 = 0 \,,
    \label{eq:zero_reconstruction} \text{and}\\ 
    &\gb(\Xb) \sim \sum_{j=1}^{J'} \lambda'_j N(\mub'_j, \Sigmab'_j),
    \label{eq:gaussianity}
\end{align}
for appropriate values of $\mub'_j, \Sigmab'_j$  (in reduced form), 
then J'=J and $\Zb$ is identified by $\gb(\Xb)$ up to affine transformation within components, i.e., %
$\gb \circ \fb$ is an invertible affine function on $\Zcal_j$, $j\in[J]$.
\end{restatable}
We provide a proof in App.~\ref{app: affine identi}.
While Thm~\ref{thm:affine-identi-comp} obtains the identifiability up to ATwC, this result does not guarantee that there is a single \emph{global affine transformation} $\hb$ such that $\gb(\Xb) = \hb(\Zb)$ for all components supports.
To achieve this stronger identifiability, the identifiability up to AT, %
we need an additional condition on $\Zb$. The following assumption requires 
that the supports of all components intersect at at least one point $\zb^0$ and that  
each of these subspaces is spanned by a subset of a shared global basis.

\begin{restatable}[Common basis and translation vector]{assumption}{asscommonbasis}
\label{ass: common basis}
Let $\Zcal_1, \dots, \Zcal_J \subseteq \RR^n$ be affine spaces such that there exists $\zb^0 \in \bigcap_{j=1}^J \Zcal_j$. We assume there exists a basis of $\RR^n$, denoted by $\{\zb^1, \dots, \zb^n\}$, such that for all $j \in [J]$, there exists $\Kcal_j \subseteq [n]$ s.t. $\{\zb^k \mid k \in \Kcal_j\}$ is a basis for the vector space $\Zcal_j-\zb^0$.
\end{restatable}

This assumption always holds for non-degenerate components. For degenerate components, a violation can happen because the components are not intersecting, or if their individual bases do not form a global shared basis. In App.~\ref{discussion:common basis} we discuss a concrete scenario in which this assumption holds, where only a subset of latent variables are ``active'' in each observation sample, e.g., they are captured in an image, while others are ``inactive'', i.e., they are masked by a constant. In this setting, the intersection of the supports is the vector of all the constant values for each variable, which will form a well-defined global shared basis.
If this assumption only holds for a subset of components, then we can achieve identifiability up to AT for this subset, while for the rest can only achieve identifiability up to ATwC using the previous result.
We provide an example to show the necessity of this assumption in App.~\ref{discussion:common basis} and use it to achieve identifiability up to AT:

\begin{restatable}[Identifiability up to AT of pdGMM latent variables]{theorem}{thmglobalaffineidenti}
\label{thm:global-affine-identi}
We assume the same conditions as Thm.~\ref{thm:affine-identi-comp}, and that the supports of components $\Zcal_1, \dots, \Zcal_J$ satisfy the common basis assumption (Ass.~\ref{ass: common basis}), 
then the representation $\gb(\Xb)$ identifies $\Zb$ up to a single global affine transformation across components, i.e., 
$\gb \circ \fb$ is affine on all of $\Zcal$.%
\end{restatable}
We provide a proof in App.~\ref{app: global affine}. Identifiability up to an affine transformation (\emph{AT}) still allows for arbitrary mixing of dimensions through general affine transformations. In many applications, we prefer representations where individual dimensions have separate and interpretable meanings. This motivates us to move to identifiability up to \emph{permutation and scaling} (PS). Permutation and scaling are fundamental indeterminacies
that typically cannot be solved without supervision or auxiliary information \citep{hyvarinen2024identifiability}. To get this stronger result, we require a stricter assumption on $\Zb$.
\begin{restatable}{assumption}{assstandcommonbasis}
\label{ass: standard common basis}
Let $\Zcal_1, \dots, \Zcal_J \subseteq \RR^n$ be the supports of the components of $\Zb$. We assume:
\begin{itemize}[noitemsep, topsep=0pt]
    \item{} [Common Standard Basis] For all $j \in [J]$, each $\Zcal_j$ is a vector space, and there exists an index set $\Kcal_j \subseteq [n]$ such that $\{\eb^k \mid k \in \Kcal_j\}$ forms a basis for $\Zcal_j$, where $\{\eb^1, \dots, \eb^n\}$ denotes the standard basis of $\RR^n$ (one-hot vectors).
    \item{} [Sufficient Support Basis Index Variability] For every $i \in [n]$, the union of the support index sets $\Kcal_j$ that do not include $i$ covers all other standard basis directions:
    $$\bigcup_{j \in [J] \,\mid\, i \notin \Kcal_j} \Kcal_j = [n] \setminus \{i\}.$$
\end{itemize}
\end{restatable}

This setting is consistent with the sparsity principle proposed by \cite{xu2024sparsity}, which shows that by enforcing sparsity of the transformed variables (as expressed in Eq.~\ref{eq:sparsity_constraint}), the corresponding transformation is an element-wise affine function. 
The Common Standard Basis assumption is a special case of the earlier Common Basis Assumption (Ass.~\ref{ass: common basis}) where the global basis is fixed to the standard basis and zero translation vector is applied. This implies that the realizations of $\Zb$ will have some coordinates equal to 0 with probability greater than zero. For example, assuming $n=3$ with $\Kcal_1 = \{1\}$ and $\Kcal_2 = \{2,3\}$, we have that samples from the component $j=1$ will have the form $(\zb_1, 0, 0)$ and samples from the component $j=2$ will have the form $(0, \zb_2, \zb_3)$. In other words, the realizations of $\Zb$ are sparse, with nonzero coordinates given by $\Kcal_j$. 

The Sufficient Support Basis Index Variability was originally proposed by \citet{lachapelle2023synergies} in the context of sparse multitask learning. 
Here, we adapt it into our context and use it to avoid cases in which certain latent variables are always degenerate together across all components, since then we would not be able to disentangle them from the observations. As we show in App.~\ref{discussion:standard common basisb}, even if this assumption does not hold strictly in a setting, our framework still supports a weaker form of
identifiability: block-wise identifiability, which allows disentanglement across blocks of latent variables, even if
variables within each block may remain entangled. This relaxation still enables meaningful interpretability and
structure in learned representations.
We can now prove our final theoretical result that shows when we can disentangle, or identify pdGMM latent variables up to PS. 

\begin{restatable}[Identifiability up to PS of pdGMM latent variables]{theorem}{thmPDidenti}
\label{thm: PD identi}
We assume the same conditions as Thm.~\ref{thm:affine-identi-comp} and that the supports of components $\Zcal_1, \dots, \Zcal_J$ satisfy Ass.~\ref{ass: standard common basis}. If the following holds:  
\begin{align}
    \mathbb{E}\norm{\gb(\Xb)}_0 \leq \mathbb{E}\norm{\Zb}_0,
    \label{eq:sparsity_constraint}
\end{align} 
where $\norm{\cdot}_0$ means L0 norm, then the representation $\gb(\Xb)$ identifies $\Zb$ up to PST, i.e., 
$\gb \circ \fb$ is a permutation combined with an element-wise linear transformation on $\Zcal$. %
\end{restatable}
We provide a proof in App.~\ref{app: PD identi}. Interestingly, a special case of this result also solves an open question in previous work on partially observed causal representation learning (CRL) \citep{xu2024sparsity}. In that setting the identifiability results required that the component index $j$ is observed for each observation (that is the group index), which is a strong assumption. Modelling this partially observable CRL problem with mixture models allows one to consider settings in which we do not know how to group the observations, i.e., a more realistic partially observable CRL setting.

\section{IMPLEMENTATION}
\label{sec: implementation}
We implement our theoretical results in two stages:
the identifiability up to AT in Thm.~\ref{thm:global-affine-identi} serves as a prerequisite for achieving the identifiability up to PS in Thm.~\ref{thm: PD identi}.
In the first stage, we use an autoencoder structure to estimate the latent variables directly from observations $\{\xb^i\}_{i\in[N]}$ by minimizing the reconstruction error (Eq.~\ref{eq:zero_reconstruction}), where $\gb_{\psi_1}$ and $\hat{\fb}_{\theta_1}$ denote the encoder and decoder, respectively. To prevent the latent codes from drifting arbitrarily and keep them clustered around the Gaussian prior,
we include the $L_2$ norm in the loss function. According to Thm.~\ref{thm:global-affine-identi}, the representations learned in this stage  correspond to an affine transformation of the ground truth latent variables.
In the second stage, we freeze the autoencoder model trained in the first stage and apply a second, inner autoencoder with affine transformations, where $\gb_{\psi_2}$ and $\hat{\fb}_{\theta_2}$ denote the affine encoder and decoder. Following \citet{xu2024sparsity}, we approximate the sparsity constraint in Eq.~\ref{eq:sparsity_constraint} of Thm.~\ref{thm: PD identi} by replacing the non-differentiable $L_0$ norm with the differentiable $L_1$ norm except at zero. As the ground truth sparsity level is not known, we use a tunable hyperparameter $\epsilon$ to control it (we provide sensitivity analysis of $\epsilon$ in App.~\ref{subapp: sensitive analysis}). Combined with the reconstruction loss, this encourages the model to converge to an appropriate sparsity level. To solve the optimization problem, we use Cooper~\citep{gallegoPosada2022cooper}, a Lagrangian-based optimization toolkit. The two-stage problems are solved using Adam \citep{diederik2014adam} and extra-gradient variant \citep{gidel2018variational}. 
\begin{align}
\label{equ: loss_pw}
\textbf{Stage 1: }
     &\min_{(\theta_1, \psi_1)}  \frac{1}{N} \sum_{i \in [N]}
     \norm{\xb^i - \hat\fb_{\theta_1}(\gb_{\psi_1}(\xb^i))}^2_2+\\
     &log(p(\gb_{\psi_1}(\xb^i);\mathbf{0}, I))
     \nonumber
     \\
\textbf{Stage 2: }
     &\min_{(\theta_2, \psi_2)} \frac{1}{N} \sum_{i \in [N]}
     \norm{\tilde{\xb}^i - \hat\fb_{\theta_2}(\gb_{\psi_2}(\tilde{\xb}^i))}^2_2  \\
     & \text{  subject to: }
    \frac{1}{N} \sum_{i \in [N]} \norm{\gb_{\psi_2}(\tilde{\xb}^i)}_1 \leq \epsilon, 
    \nonumber
\end{align}
where $\tilde{\xb}^i:=\gb_{\psi_1}(\xb^i)$. We provide all details in App.~\ref{app: implementation}.

\section{Experimental results}
\label{sec: experiments}

\paragraph{Numerical Experiments Setup.}
\label{sec:numericalexp}
We generate simulations with variations in the underlying causal model, the rank of Gaussian components, and mixing function $\fb$. We average results for each setup over 5 random seeds.
We consider $n=\{5,10,20,40\}$ latent variables $\Zb$. For each experiment, we sample a directed acyclic graph $\Dcal$ from an Erd\"os-R\'enyi (ER)-$k$ model with $k\in \{0,1,2,3\}$, where each ER-$k$ graph contains $n\cdot k$ directed edges. In particular, ER-0 corresponds to an empty graph, which implies independent latent variables. Given $\Dcal$, we simulate a linear Gaussian Structural Causal Model where edge weights are sampled uniformly from $[-1,-0.2]\cup[0.2,1]$, and standard Gaussian noise is added at each node.
To simulate pdGMMs, we define the number of mixture components $J=5n$ and introduce the ratio of non-degenerate dimensions $\rho \in \{1 \textsc{var}, 50\%, 75\%\}$, where $1 \textsc{var}$ indicates only one non-degenerate dimension, and $50\%$ or $75\%$ are the proportion of non-degenerate dimensions relative to $n$. 

Then, we randomly sample $J$ different $\rho n$-hot vectors to define a set of basis index sets $\Kcal_j$, which determines the subspace (i.e., the support) where the component has non-degenerate variance. For dimensions not in $\Kcal_j$ (i.e., those with a 0 in the multi-hot vector), we set their value to the corresponding entry in a predefined translation vector $\zb^0$.
To further show the different results for Ass.~\ref{ass: common basis} and Ass.~\ref{ass: standard common basis}, we vary the translation vector $\zb^0=\mub+ \delta \cdot \sigma$, where $\mub, \sigma \in \RR^n$ are the mean and standard deviation of $\Zb$, and $\delta \in \{0,1,2,3\}$. In addition, we apply a $n\times n $ rotation matrix on $\Zb$, parameterized by a rotation angle $\theta\in \{0^{\circ}, 15^{\circ}, 30^{\circ},45^{\circ}\}$, to control the deviations from standard basis. ll settings satisfy Ass.~\ref{ass: common basis}, which means we can always achieve identifiability up to AT, but Ass.~\ref{ass: standard common basis} is only satisfied in the setting with $\delta=0$ and $\theta=0$, and thus we expect identifiability up to PS to only be possible there.

To simulate the observed variables $\Xb$, we apply an invertible piecewise affine mixing function $\fb$ to the latent variables $\Zb$. We parameterize $\fb$ by a multi-layer perceptron (MLP) with $m \in \{3, 10, 20\}$ hidden layers,  $(m-1)$ Leaky-ReLU activation functions (with slope $\alpha \sim \text{Unif}(0.5, 1.5)$) and a final affine layer. The number of layers $m$ is a proxy for the non-linearity and complexity of the mixing function; the larger the $m$, the more complex the transformation from $\Zb$ to $\Xb$.

\begin{figure}[t]

\begin{minipage}{0.48\textwidth}
\centering
\footnotesize
\setlength{\tabcolsep}{3pt} %
\renewcommand{\arraystretch}{1.1} %

\begin{tabular}{@{\hskip 2pt}c@{\hskip 2pt} %
    @{\hskip 2pt}c@{\hskip 2pt} %
    @{\hskip 2pt}c@{\hskip 2pt} %
    @{\hskip 2pt}c@{\hskip 2pt} %
    @{\hskip 2pt}c@{\hskip 2pt} %
    @{\hskip 2pt}c@{\hskip 2pt} %
    c c c c}
\hline
$n$ & $k$ & $m$ & $\rho$ & $\delta$ & $\theta$ &
\makecell{$\Rb^2$ $\uparrow$\\ Stage1} &
\makecell{\textbf{MCC} \\ Stage1} &
\makecell{\textbf{MCC} $\uparrow$\\ Stage2 }&
\makecell{\textbf{MCC} $\uparrow$\\ VaDE }\\
\hline
\textbf{5}  & 1 & 10 & 50\% & 0 & 0 & 0.93 & 0.75 & 0.96 & 0.45\\
\textbf{10} & 1 & 10 & 50\% & 0 & 0 & 0.94 & 0.56 & 0.97 & 0.32\\
\textbf{20} & 1 & 10 & 50\% & 0 & 0 & 0.93 & 0.46 & 0.92 & 0.23\\
\textbf{40} & 1 & 10 & 50\% & 0 & 0 & 0.94 & 0.37 & 0.94 & 0.25\\
\hline
10 & \textbf{0} & 10 & 50\% & 0 & 0 & 0.94 & 0.51 & 0.97 & 0.33\\
10 & \textbf{1} & 10 & 50\% & 0 & 0 & 0.94 & 0.56 & 0.97 & 0.32\\
10 & \textbf{2} & 10 & 50\% & 0 & 0 & 0.93 & 0.58 & 0.95 & 0.31\\
10 & \textbf{3} & 10 & 50\% & 0 & 0 & 0.92 & 0.56 & 0.91 & 0.27\\
\hline
10 & 1 & \textbf{3} & 50\% & 0 & 0 & 0.99 & 0.50 & 0.96 & 0.33\\
10 & 1 & \textbf{10} & 50\% & 0 & 0 & 0.94 & 0.56 & 0.97 & 0.32\\
10 & 1 & \textbf{20} & 50\% & 0 & 0 & 0.88 & 0.50 & 0.93 & 0.35\\
\hline
10 & 1 & 10 & \textbf{1 var} & 0 & 0 & 0.75 & 0.46 & 0.45 & 0.28\\
10 & 1 & 10 & \textbf{50\%} & 0 & 0 & 0.94 & 0.56 & 0.97 & 0.32\\
10 & 1 & 10 & \textbf{75\%} & 0 & 0 & 0.95 & 0.57 & 0.93 & 0.31\\
\hline
10 & 1 & 10 & 50\% & \textbf{0} & 0 & 0.94 & 0.56 & 0.97 & 0.32\\
10 & 1 & 10 & 50\% & \textbf{1} & 0 & 0.95 & 0.59 & 0.88 & 0.30\\
10 & 1 & 10 & 50\% & \textbf{2} & 0 & 0.96 & 0.59 & 0.61 & 0.27\\
10 & 1 & 10 & 50\% & \textbf{3} & 0 & 0.96 & 0.58 & 0.49 & 0.28\\
\hline
10 & 1 & 10 & 50\% & 0 & \textbf{0} & 0.94 & 0.56 & 0.97 & 0.32\\
10 & 1 & 10 & 50\% & 0 & \textbf{15} & 0.93 & 0.58 & 0.83 & 0.32\\
10 & 1 & 10 & 50\% & 0 & \textbf{30} & 0.93 & 0.57 & 0.60 & 0.31\\
10 & 1 & 10 & 50\% & 0 & \textbf{45} & 0.93 & 0.57 & 0.62 & 0.32\\

\hline
\end{tabular}

\captionof{table}{
Results for the numerical experiments.
The bold font highlights the parameters that vary in each block. %
$\Rb^2$(Stage 1) shows identifiability up to AT after Stage 1 
\textbf{MCC}(Stage 1) reflects the necessity of sparsity to achieve identifiability up to PS; %
\textbf{MCC}(Stage 2) presents the identifiability up to PS after Stage 2; %
\textbf{MCC}(VaDE) reports the results of \citet{kivva2022identifiability}. Full results with Std. Dev. are in Tab.~\ref{tab: numerical_full}.}
\label{tab: numerical}
\end{minipage}
\end{figure}

\paragraph{Results for Stage 1 (Thm.~\ref{thm:global-affine-identi}).} 
For the encoder $\gb_{\psi_1}$ and the decoder $\hat{\fb}_{\theta_1}$ in the first stage, we employ 5-layer MLPs configured with $[50, 100, 100, 50] \times n$ units per layer, where $n$ is the number of latent variables. Each layer uses Leaky-ReLU activations, except for the final one. Batch normalization is applied to control the norm of $\gb_{\psi_1}(\Xb)$. To validate the identifiability up to \emph{AT} presented in Thm.~\ref{thm:global-affine-identi}, 
the metric we use is the $R^2$ of the linear regression model with $\Zb$ as explanatory variable and $\gb_{\psi_1}(\Xb)$ as response.
A high $R^2$ indicates that the ground truth latent variables can be linearly mapped to the learned representation with low error, consistent with our identifiability result. 
We show the average $R^2$
over five random seeds in Tab.~\ref{tab: numerical}. Our method remains robust with large $n$, dense causal graph $k$, varying translation vectors $\zb^0 (\delta)$ and basis direction $\theta$. Performance decreases slightly only with increased nonlinearity ($m$) and lower proportions of nondegenerate variables ($\rho$). This is expected: higher nonlinearity increases the complexity of achieving affine identifiability, while lower $\rho$ allows greater flexibility in how representations are distributed in $\RR^n$, making it more challenging to enforce Gaussianity.
However, a high $R^2$ does not imply disentanglement. Following prior work \citep{khemakhem2020variational,kivva2022identifiability}, we also evaluate identifiability up to \emph{PS} via \emph{Mean Correlation Coefficient} (MCC), a metric detailed in App.~\ref{app: metric}. %
Across all row groups, we observe that, as expected, Stage 1 \emph{alone} does not identify the latent variables up to PS accurately, as reflected by the low MCC values reported in the third right column of Tab.~\ref{tab: numerical}. This motivates us to move to Stage 2. %

\paragraph{Results for Stage 2 (Thm.~\ref{thm: PD identi}).} 
For the affine encoder $\gb_{\psi_2}$ and the decoder $\hat{\fb}_{\theta_2}$ in the second stage, we employ 7-layer MLPs configured with $[10, 50, 50, 50, 50, 10] \times n$ units per layer, where $n$ is the number of latent variables. Batch normalization is applied to control the norm of $\gb_{\psi_2}(\Xb)$. Since we generate a sufficient number of basis indices, our representation should be able to have a one-to-one correspondence with the ground truth latent variables. %
The last two columns of Tab.~\ref{tab: numerical} reports the average MCC scores for our method and VaDE \citep{kivva2022identifiability}. This is a misspecified setting for VaDE, because VaDE requires non-degeneracy of the data, but we report the results nevertheless as a baseline, and show that as expected we outperform it in all settings.
The results indicate consistently good performance across a range of $n$, $k$, and $m$. However, the overall performance is influenced by the quality of representations from Stage 1, with a noticeable decrease under a low nondegeneracy ratio $\rho$. This is expected, since large translation vectors $\zb^0 (\delta)$ and non-standard basis ($\theta$) violate Ass.~\ref{ass: standard common basis}.

\paragraph{Ablations.} We provide an evaluation on independent latent variables in App.~\ref{app: num independent},
showing comparable results, and an empirical complexity analysis in App.~\ref{subapp: complexity}, showing that we can scale our method up to $n=50$ and partially to $n=100$ variables. This is in line with other CRL works, which mostly focus on $10-20$ latent variables. On the other hand, the inference time increases exponentially with the number of latents.
In App.~\ref{subapp: sensitive analysis} we report a sensitivity analysis on our hyperparameters, in particular the sparsity level, which we set to $\epsilon=0.01$ for the main experiments, and the learning rate, which we set to $lr=1e-4$. 
In App.~\ref{subapp: baselines} we report an ablation on the sparsity constraint, showing its importance for the second stage. We also compare with a related method \citep{xu2024sparsity} that requires additional information (i.e., knowing the label of the component) and show that with the same information, our method provides comparable results. For completeness, we also provide results in App.~\ref{app: without misspecifiction} with a setting that is fair to the baseline VaDE \citep{kivva2022identifiability}, where we generate the latent with a non-degenerate GMM, showing that our method outperforms VaDE even in these settings. Similarly to other work in CRL, our theoretical results assume that we know the ground truth $n$, but empirically our method performs well even if we overestimate it, as shown in App.~\ref{subapp: over n}. As an additional measure of the accuracy of our reconstruction, in App.~\ref{subapp: graph} we evaluate the Structural Hamming Distance (SHD) of the causal graph that we learn with a standard causal discovery method, PC \citep{spirtes2000causation}, on our recovered latent variables, showing that the SHD is only slightly worse than the SHD for the ground truth variables.

\paragraph{Settings when the assumptions are violated.}
In App.~\ref{app: assump} we discuss the intuition, implications and practical applicability of our assumptions, providing also ablations for their violations.
In particular, in App.~\ref{app:misspecified} we show empirically that our method is often still robust in a misspecified setting, i.e., when the variables are not Gaussian (but instead they follow and exponential or Gumbel distribution) or the mixing function is not piecewise affine (but instead is a smooth Leaky-ReLU or sigmoid). On the other hand, we do need all other assumptions. We provide evaluation showing the necessity of Ass.~\ref{ass: no equ norm} (genericity of pdGMMs) in App.~\ref{discussion:genericity}, showing that when it does not hold the results are much worse. Similarly, results are much worse when Ass.~\ref{ass: common basis} does not hold, so there does not exists a unique shared basis for all Gaussian components. To show a violation of this, we generate the data with a mixture of two different bases: a standard basis and one with rotations, and provide results in App.~\ref{discussion:common basis}, showing the necessity of this assumption. Table~\ref{tab: numerical} already shows violations of Ass.~\ref{ass: standard common basis}a when $\delta$ (the masking value) and 
$\theta$ (the rotation of the standard basis) are non-zero, showcasing a decrease in performance when the assumption does not hold.
 When we have violations of Ass.~\ref{ass: standard common basis}b, we can still achieve block-identifiability, as discussed and shown empirically in App.~\ref{discussion:standard common basisb}.

\paragraph{Image dataset: Multiple Balls.}
\label{sec:multiballexp}
To show potential applications of our approach in partially observable causal representation learning, 
we evaluate our method on an image dataset from \citet{ahuja2022weakly,xu2024sparsity}, which consists of $\bb$ moving balls rendered in a 2D space, as shown in App.~\ref{app: balls}. 
The latent causal variables correspond to the $(x,y)$ positions of each ball, which are modeled as Gaussian. We focus on a \emph{fixed position} setting, where the balls usually move freely within the frame but occasionally remain stationary at unknown fixed positions. When a ball is stationary, the corresponding latent dimensions are degenerate.
We generate datasets with varying numbers of balls: $\bb = 2, 4, 6$. The $(x, y)$ coordinates of each ball $i\in[\bb]$ are sampled independently from a truncated 2-dimensional normal distribution $\Ncal(\mub_i,\Sigmab_i)$ bounded within the square $(0.05,0.95)^2$, where $\mub_i\sim \text{Unif}(0.3,0.7)^2$, and $\Sigmab_i= ((0.01,0.00)(0.00,0.01))$. To avoid the high ratio of occlusion, we use rejection sampling to sample $\mub_i$ to ensure the Euclidean distance between any pair of $\mub_i$ values is at least $0.2$. Each ball has a probability $p=0.1$ to be at a fixed position $\mub_i$. 

\begin{figure}
\centering
\includegraphics[width=0.45\textwidth]{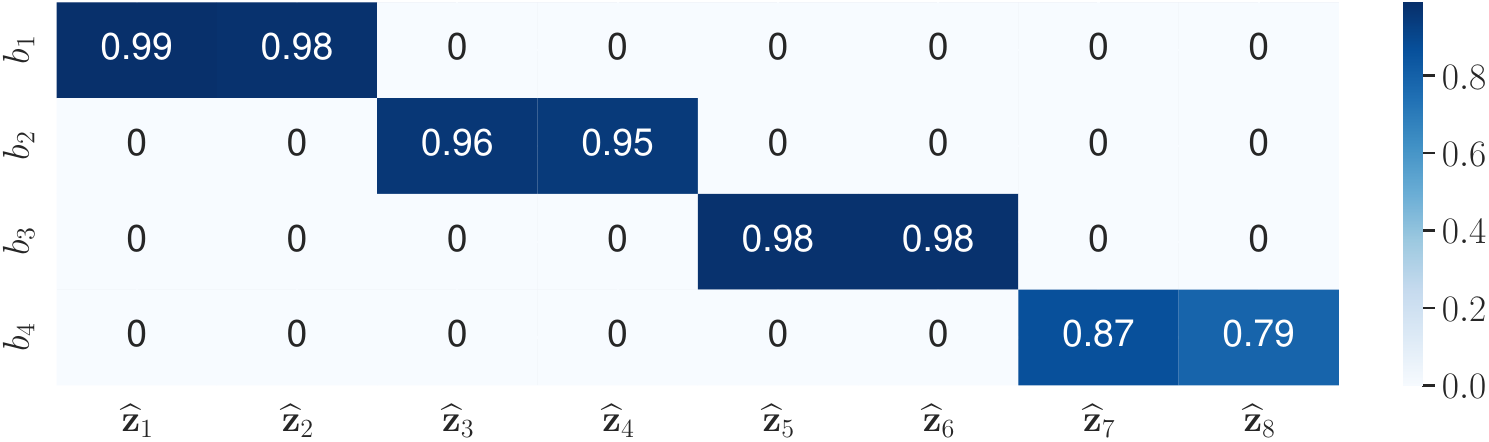}
\caption{
$R^2$ of the linear regression between ball positions $b_1, \dots, b_4$ and learned representations $\hat{z}_1, \dots, \hat{z}_8 $. We use two variables for each $(x,y)$ position.}
\label{fig:ball4} 
\end{figure}

Before we train the model described in Sec.~\ref{sec: implementation}, we pre-train a CNN model to reduce the input image dimension from $64\times 64$ to a vector of size $16\bb\times1$. More details about the network architecture are provided in App.~\ref{app: balls}. Instead of a fixed-size training dataset, we generate images online until convergence.
Since the variables describing the $x$ and $y$ position of the same ball are always degenerate together, they cannot be disentangled in our method (as sufficient variability in Ass.~\ref{ass: standard common basis} does not hold). To evaluate the performance, we perform regression from each representation separately to the corresponding ground-truth position pairs, as illustrated in Fig.~\ref{fig:ball4}. The results show that the position of each ball can be recovered through pairs of representations, which is consistent with our theoretical results. Additional results are in App.~\ref{app: balls}.

\section{RELATED WORK}
Our work is inspired by previous work in nonlinear independent component analysis (ICA), causal representation learning (CRL) and sparsity-based methods. ICA \citep{COMON1994,hyvarinen2013independent} is a seminal approach that identifies latent variables under the assumption that the latent variables are independent and that the mixing function is linear. In the non-linear ICA setting, \citet{hyvarinen2016unsupervised,hyvarinen2017nonlinear} reformulate this standard ICA independence assumption as conditional independence given observed auxiliary variables. \citet{khemakhem2020variational} show the identifiability of the VAE model with a conditionally independent prior. \citet{willetts2021don} show that the auxilliary variables can be unobservable and learned; however, this still requires the conditional independence assumption.

Causal representation learning (CRL) extends (nonlinear) ICA to the case in which latent variables have causal relations between them and hence arbitrary dependencies. Our approach fits in this line of work. Many CRL works consider additional data or information to identify causal variables, e.g., interventions \citep{brehmer2022weakly,lippe2022citris, varici2023score,von2023nonparametric, buchholz2023learning, zhang2023identifiability, ahuja2023interventional}, actions or temporal structure \citep {lippe2023biscuit,lachapelle2022disentanglement, lachapelle2024nonparametric}, multi-view settings \citep{von2021self,ahuja2023multi, yao2023multi,morioka2023causal}, multiple environments \citep{liu2024identifiable}, grouping of observations by sparsity patterns \citep{xu2024sparsity}, hierarchical structures \citep{kong2023identification,kong2024learning} or  conditions on causal graphs \citep{zhang2024causal,dong2024on,song2024causal}. %

In our case, we do not assume any additional data or information besides the observations, but instead consider parametric assumptions on the latent variables and mixing function.  
Similarly to us, \citet{kivva2022identifiability} studies the identifiability of non-degenerate GMM latents with piecewise affine mixing. For PS identifiability, they require that all pairs of latent variables must be conditionally independent given an auxiliary variable.
Other works consider restrictions on the mixing function, e.g., \citet{squires2023linear,bing2024identifying} assume a linear mixing function, \citet{ahuja2023interventional} assumes a finite-degree polynomial and either independent supports or perfect interventions, and \citet{lachapelle2023additive} assumes an additive mixing function.
Our results do not require conditional independence, independent support or interventions and allow mixing functions composed of an infinite number of affine pieces, allowing the potential possibility to approximate any nonlinear functions up to an arbitrary precision.

Our focus on low-rank representations is also inspired by prior work that leverages sparsity in (causal) representation learning, e.g.,  structural sparsity in Jacobians of nonlinear mixing functions \citep{zheng2022identifiability, zheng2023generalizing}, sparsity constraints on affine mixing \citep{ng2023identifiability}, sparsity on causal graphs \citep{zhang2024causal,song2024causal}, sparse decoders \citep{moran2022identifiable} and sparsity task-specific predictors \citep{lachapelle2023synergies, fumero2023leveraging}. For time-series data, \citet{lachapelle2022disentanglement, lachapelle2024nonparametric} propose sparse actions and sparse temporal dependencies, while \citet{li2025on} introduce latent process sparsity, referring to sparse connections between latents both within and across time.
Differently from these approaches, which impose sparsity on the model structure or latent interactions, we are enforcing \emph{sparsity of the representation itself}. This aligns with prior work \citep{tonolini2020variational}, as well as classical methods \textit{sparse component analysis} \citep{gribonval2006survey}, and \textit{sparse dictionary learning} \citep{Mairal_Bach_Ponce2009}. Most works show empirical success on benchmark tasks and do not provide guarantees for identifiability, which is our main focus. \citet{moran2022identifiable} focuses on independent latent variables, while \citet{hu2023global} studies the identifiability of dictionary learning, but they assume \emph{affine mixing}. Instead, we focus on \emph{piecewise-affine mixing} for dependent latent variables.

\section{CONCLUSIONS AND LIMITATIONS}
\label{sec: limitation}
In this work, we focus on identifying pdGMM latent variables with piecewise affine mixing functions. We first show that the entire pdGMM can be identified via an open set. This result is a critical step for a series of increasingly stronger identifiability results, ultimately proving identifiability up to permutation and element-wise transformations, i.e. completely disentangled representations.
While our experiments validate our results, there
are still several limitations. As discussed in App.~\ref{app: assump}, some of our assumptions might not hold in some settings, potentially leading to partial or no theoretical guarantees. 
In particular, the Gaussian assumption for mixture components may not hold, so extending our approach to other parametric models is an exciting avenue for future research. Finally, although our method encourages Gaussianity on representations, ensuring this in practice remains challenging.

\subsubsection*{Acknowledgements}
We thank the anonymous reviewers for their helpful and constructive comments. The research of DX and SM was supported by the Air Force Office of Scientific Research under award number FA8655-22-1-7155. Any opinions, findings, and conclusions or recommendations expressed in this material are those of the author(s) and do not necessarily reflect the views of the United States Air Force. We also thank SURF for the support in using the Dutch National Supercomputer Snellius. 

\bibliography{ref}

@article{yakowitz1968identifiability,
  title={On the identifiability of finite mixtures},
  author={Yakowitz, Sidney J and Spragins, John D},
  journal={The Annals of Mathematical Statistics},
  volume={39},
  number={1},
  pages={209--214},
  year={1968},
  publisher={Institute of Mathematical Statistics}
}

@article{kivva2022identifiability,
  title={Identifiability of deep generative models without auxiliary information},
  author={Kivva, Bohdan and Rajendran, Goutham and Ravikumar, Pradeep and Aragam, Bryon},
  journal={Advances in Neural Information Processing Systems},
  volume={35},
  pages={15687--15701},
  year={2022}
}

@article{xu2024sparsity,
  title={A sparsity principle for partially observable causal representation learning},
  author={Xu, Danru and Yao, Dingling and Lachapelle, S{\'e}bastien and Taslakian, Perouz and von K{\"u}gelgen, Julius and Locatello, Francesco and Magliacane, Sara},
  journal={International Conference on Machine Learning},
  year={2024}
}

@inproceedings{khemakhem2020variational,
  title={Variational autoencoders and nonlinear ica: A unifying framework},
  author={Khemakhem, Ilyes and Kingma, Diederik and Monti, Ricardo and Hyvarinen, Aapo},
  booktitle={International Conference on Artificial Intelligence and Statistics},
  pages={2207--2217},
  year={2020},
  organization={PMLR}
}

@inproceedings{lachapelle2022disentanglement,
  title={Disentanglement via mechanism sparsity regularization: A new principle for nonlinear ICA},
  author={Lachapelle, S{\'e}bastien and Rodriguez, Pau and Sharma, Yash and Everett, Katie E and Le Priol, R{\'e}mi and Lacoste, Alexandre and Lacoste-Julien, Simon},
  booktitle={Conference on Causal Learning and Reasoning},
  pages={428--484},
  year={2022},
  organization={PMLR}
}

@inproceedings{lachapelle2023synergies,
  title={Synergies between Disentanglement and Sparsity: Generalization and Identifiability in Multi-Task Learning},
  author={Lachapelle, S{\'e}bastien and Deleu, Tristan and Mahajan, Divyat and Mitliagkas, Ioannis and Bengio, Yoshua and Lacoste-Julien, Simon and Bertrand, Quentin},
  booktitle={International Conference on Machine Learning},
  pages={18171--18206},
  year={2023},
  organization={PMLR}
}

@inproceedings{lachapelle2023additive,
      title={Additive Decoders for Latent Variables Identification and Cartesian-Product Extrapolation}, 
      author={S. Lachapelle and D. Mahajan and I. Mitliagkas and S. Lacoste-Julien},
      year={2023},
      booktitle={Advances in Neural Information Processing Systems},
}

@article{COMON1994,
title = {Independent component analysis, A new concept?},
journal = {Signal Processing},
year = {1994},
author = {P. Comon}
}

@misc{gallegoPosada2022cooper,
    author={Gallego-Posada, Jose and Ramirez, Juan},
    title={Cooper: a toolkit for Lagrangian-based constrained optimization},
    year={2022}
}

@article{gidel2018variational,
  title={A variational inequality perspective on generative adversarial networks},
  author={Gidel, Gauthier and Berard, Hugo and Vignoud, Ga{\"e}tan and Vincent, Pascal and Lacoste-Julien, Simon},
  journal={International Conference on Learning Representations},
  year={2019}
}

@inproceedings{diederik2014adam,
  author = {Kingma, Diederik P. and Ba, Jimmy},
  booktitle = {International Conference on Learning Representations},
  editor = {Bengio, Yoshua and LeCun, Yann},
  title = {Adam: A Method for Stochastic Optimization.},
  year = {2015}
}

@article{von2021self,
  title={Self-supervised learning with data augmentations provably isolates content from style},
  author={von K{\"u}gelgen, Julius and Sharma, Yash and Gresele, Luigi and Brendel, Wieland and Sch{\"o}lkopf, Bernhard and Besserve, Michel and Locatello, Francesco},
  journal={Advances in neural information processing systems},
  volume={34},
  pages={16451--16467},
  year={2021}
}

@article{scholkopf2021toward,
  title={Toward causal representation learning},
  author={Sch{\"o}lkopf, Bernhard and Locatello, Francesco and Bauer, Stefan and Ke, Nan Rosemary and Kalchbrenner, Nal and Goyal, Anirudh and Bengio, Yoshua},
  journal={Proceedings of the IEEE},
  volume={109},
  number={5},
  pages={612--634},
  year={2021},
  publisher={IEEE}
}

@article{hyvarinen2016unsupervised,
  title={Unsupervised feature extraction by time-contrastive learning and nonlinear ica},
  author={Hyvarinen, Aapo and Morioka, Hiroshi},
  journal={Advances in neural information processing systems},
  volume={29},
  year={2016}
}

@inproceedings{hyvarinen2017nonlinear,
  title={Nonlinear ICA of temporally dependent stationary sources},
  author={Hyvarinen, Aapo and Morioka, Hiroshi},
  booktitle={Artificial Intelligence and Statistics},
  pages={460--469},
  year={2017},
  organization={PMLR}
}

@article{brehmer2022weakly,
  title={Weakly supervised causal representation learning},
  author={Brehmer, Johann and De Haan, Pim and Lippe, Phillip and Cohen, Taco S},
  journal={Advances in Neural Information Processing Systems},
  volume={35},
  pages={38319--38331},
  year={2022}
}

@inproceedings{lippe2022citris,
  title={Citris: Causal identifiability from temporal intervened sequences},
  author={Lippe, Phillip and Magliacane, Sara and L{\"o}we, Sindy and Asano, Yuki M and Cohen, Taco and Gavves, Stratis},
  booktitle={International Conference on Machine Learning},
  pages={13557--13603},
  year={2022},
  organization={PMLR}
}

@article{lippe2023biscuit,
  title={BISCUIT: Causal Representation Learning from Binary Interactions},
  author={Lippe, Phillip and Magliacane, Sara and L{\"o}we, Sindy and Asano, Yuki M and Cohen, Taco and Gavves, Efstratios},
  journal={Proceedings of the Thirty-Ninth Conference on Uncertainty in Artificial Intelligence},
  year={2023}
}

@inproceedings{ahuja2023interventional,
  title={Interventional causal representation learning},
  author={Ahuja, Kartik and Mahajan, Divyat and Wang, Yixin and Bengio, Yoshua},
  booktitle={International Conference on Machine Learning},
  pages={372--407},
  year={2023},
  organization={PMLR}
}

@inproceedings{von2023nonparametric,
  title={Nonparametric Identifiability of Causal Representations from Unknown Interventions},
  author={von K{\"u}gelgen, Julius and Besserve, Michel and Liang, Wendong and Gresele, Luigi and Keki{\'c}, Armin and Bareinboim, Elias and Blei, David M and Sch{\"o}lkopf, Bernhard},
  booktitle={Advances in Neural Information Processing 36},
  year={2023}
}

@inproceedings{
zhang2023identifiability,
title={Identifiability Guarantees for Causal Disentanglement from Soft Interventions},
author={Jiaqi Zhang and Kristjan Greenewald and Chandler Squires and Akash Srivastava and Karthikeyan Shanmugam and Caroline Uhler},
booktitle={Advances in Neural Information Processing Systems},
year={2023},
}

@inproceedings{ahuja2023multi,
  title={Multi-Domain Causal Representation Learning via Weak Distributional Invariances},
  author={Ahuja, Kartik and Mansouri, Amin and Wang, Yixin},
  booktitle={Causal Representation Learning Workshop at NeurIPS 2023},
  year={2023}
}

@inproceedings{squires2023linear,
  title={Linear Causal Disentanglement via Interventions}, 
  author={Chandler Squires and Anna Seigal and Salil Bhate and Caroline Uhler},
  year={2023},
  booktitle = {40th International Conference on Machine Learning},
}

@inproceedings{buchholz2023learning,
  title={Learning Linear Causal Representations from Interventions under General Nonlinear Mixing},
  author={Buchholz, Simon and Rajendran, Goutham and Rosenfeld, Elan and Aragam, Bryon and Sch{\"o}lkopf, Bernhard and Ravikumar, Pradeep},
  booktitle={Advances in Neural Information Processing Systems},
  year={2023}
}

@article{yao2023multi,
  title={Multi-View Causal Representation Learning with Partial Observability},
  author={Yao, Dingling and Xu, Danru and Lachapelle, S{\'e}bastien and Magliacane, Sara and Taslakian, Perouz and Martius, Georg and von K{\"u}gelgen, Julius and Locatello, Francesco},
  journal={International Conference on Learning Representations},
  year={2024}
}

@misc{lachapelle2024nonparametric,
      title={Nonparametric Partial Disentanglement via Mechanism Sparsity: Sparse Actions, Interventions and Sparse Temporal Dependencies}, 
      author={Sébastien Lachapelle and Pau Rodríguez López and Yash Sharma and Katie Everett and Rémi Le Priol and Alexandre Lacoste and Simon Lacoste-Julien},
      year={2024},
      eprint={2401.04890},
      archivePrefix={arXiv},
      primaryClass={stat.ML}
}

@article{kong2023identification,
  title={Identification of nonlinear latent hierarchical models},
  author={Kong, Lingjing and Huang, Biwei and Xie, Feng and Xing, Eric and Chi, Yuejie and Zhang, Kun},
  journal={Advances in Neural Information Processing Systems},
  volume={36},
  pages={2010--2032},
  year={2023}
}

@article{
kong2024learning,
title={Learning Discrete Concepts in Latent Hierarchical Models},
author={Lingjing Kong and Guangyi Chen and Biwei Huang and Eric P. Xing and Yuejie Chi and Kun Zhang},
journal={Advances in Neural Information Processing Systems},
year={2024}
}

@inproceedings{
dong2024on,
title={On the Parameter Identifiability of Partially Observed Linear Causal Models},
author={Xinshuai Dong and Ignavier Ng and Biwei Huang and Yuewen Sun and Songyao Jin and Roberto Legaspi and Peter Spirtes and Kun Zhang},
booktitle={The Thirty-eighth Annual Conference on Neural Information Processing Systems},
year={2024}
}

@article{zhang2024causal,
  title={Causal representation learning from multiple distributions: A general setting},
  author={Zhang, Kun and Xie, Shaoan and Ng, Ignavier and Zheng, Yujia},
  journal={International Conference on Machine Learning},
  year={2024}
}

@inproceedings{
song2024causal,
title={Causal Temporal Representation Learning with Nonstationary Sparse Transition},
author={Xiangchen Song and Zijian Li and Guangyi Chen and Yujia Zheng and Yewen Fan and Xinshuai Dong and Kun Zhang},
booktitle={The Thirty-eighth Annual Conference on Neural Information Processing Systems},
year={2024}
}

@article{morioka2023causal,
  title={Causal representation learning made identifiable by grouping of observational variables},
  author={Morioka, Hiroshi and Hyv{\"a}rinen, Aapo},
  journal={International Conference on Machine Learning},
  year={2024}
}

@inproceedings{tonolini2020variational,
  title={Variational sparse coding},
  author={Tonolini, Francesco and Jensen, Bj{\o}rn Sand and Murray-Smith, Roderick},
  booktitle={Uncertainty in Artificial Intelligence},
  pages={690--700},
  year={2020},
  organization={PMLR}
}

@article{moran2022identifiable,
  title={Identifiable Deep Generative Models via Sparse Decoding},
  author={Moran, G and Sridhar, D and Wang, Y and Blei, D},
  journal={Transactions on machine learning research},
  year={2022}
}

@article{zheng2022identifiability,
  title={On the identifiability of nonlinear ICA: Sparsity and beyond},
  author={Zheng, Yujia and Ng, Ignavier and Zhang, Kun},
  journal={Advances in Neural Information Processing Systems},
  volume={35},
  pages={16411--16422},
  year={2022}
}

@article{zheng2023generalizing,
  title={Generalizing Nonlinear ICA Beyond Structural Sparsity},
  author={Zheng, Yujia and Zhang, Kun},
  journal={Advances in Neural Information Processing Systems},
  volume={36},
  pages={13326--13355},
  year={2023}
}

@article{ng2023identifiability,
  title={On the Identifiability of Sparse ICA without Assuming Non-Gaussianity},
  author={Ng, Ignavier and Zheng, Yujia and Dong, Xinshuai and Zhang, Kun},
  journal={Advances in Neural Information Processing Systems},
  volume={36},
  year={2023}
}

@inproceedings{
fumero2023leveraging,
title={Leveraging sparse and shared feature activations for disentangled representation learning},
author={Marco Fumero and Florian Wenzel and Luca Zancato and Alessandro Achille and Emanuele Rodol{\`a} and Stefano Soatto and Bernhard Sch{\"o}lkopf and Francesco Locatello},
booktitle={Thirty-seventh Conference on Neural Information Processing Systems},
year={2023}
}

@inproceedings{gribonval2006survey,
  title={A survey of sparse component analysis for blind source separation: principles, perspectives, and new challenges},
  author={Gribonval, R{\'e}mi and Lesage, Sylvain},
  booktitle={ESANN'06 proceedings-14th European Symposium on Artificial Neural Networks},
  pages={323--330},
  year={2006},
  organization={d-side publi.}
}

@inproceedings{Mairal_Bach_Ponce2009,
  title={Online dictionary learning for sparse coding},
  author={J. Mairal and F. Bach and J. Ponce and G. Sapiro},
  booktitle={Proceedings of the 26th annual international conference on machine learning},
  pages={689--696},
  year={2009}
}

@inproceedings{
hu2023global,
title={Global Identifiability of  \${\textbackslash}ell\_1\$-based Dictionary Learning via Matrix Volume Optimization},
author={Jingzhou Hu and Kejun Huang},
booktitle={Thirty-seventh Conference on Neural Information Processing Systems},
year={2023}
}

@article{caron2005zero,
  title={The zero set of a polynomial},
  author={Caron, Richard and Traynor, Tim},
  journal={WSMR Report},
  pages={05--02},
  year={2005}
}

@article{bona2023parameter,
  title={Parameter identifiability of a deep feedforward ReLU neural network},
  author={Bona-Pellissier, Joachim and Bachoc, Fran{\c{c}}ois and Malgouyres, Fran{\c{c}}ois},
  journal={Machine Learning},
  volume={112},
  number={11},
  pages={4431--4493},
  year={2023},
  publisher={Springer}
}

@book{Munkres2000Topology,
  author = {Munkres, J. R.},
  edition = 2,
  publisher = {Prentice Hall, Inc.},
  title = {{Topology}},
  year = 2000
}

@article{hyvarinen2024identifiability,
  title={Identifiability of latent-variable and structural-equation models: from linear to nonlinear},
  author={Hyv{\"a}rinen, Aapo and Khemakhem, Ilyes and Monti, Ricardo},
  journal={Annals of the Institute of Statistical Mathematics},
  volume={76},
  number={1},
  pages={1--33},
  year={2024},
  publisher={Springer}
}

@article{gao2024scaling,
  title={Scaling and evaluating sparse autoencoders},
  author={Gao, Leo and la Tour, Tom Dupr{\'e} and Tillman, Henk and Goh, Gabriel and Troll, Rajan and Radford, Alec and Sutskever, Ilya and Leike, Jan and Wu, Jeffrey},
  journal={International Conference on Learning Representations},
  year={2025}
}

@article{ahuja2022weakly,
  title={Weakly supervised representation learning with sparse perturbations},
  author={Ahuja, Kartik and Hartford, Jason S and Bengio, Yoshua},
  journal={Advances in Neural Information Processing Systems},
  volume={35},
  pages={15516--15528},
  year={2022}
}

@article{marks2024sparse,
  title={Sparse feature circuits: Discovering and editing interpretable causal graphs in language models},
  author={Marks, Samuel and Rager, Can and Michaud, Eric J and Belinkov, Yonatan and Bau, David and Mueller, Aaron},
  journal={International Conference on Learning Representations},
  year={2025}
}

@article{cunningham2023sparse,
  title={Sparse autoencoders find highly interpretable features in language models},
  author={Cunningham, Hoagy and Ewart, Aidan and Riggs, Logan and Huben, Robert and Sharkey, Lee},
  journal={International Conference on Learning Representations},
  year={2024}
}

@inproceedings{
li2025on,
title={On the Identification of Temporal Causal Representation with Instantaneous Dependence},
author={Zijian Li and Yifan Shen and Kaitao Zheng and Ruichu Cai and Xiangchen Song and Mingming Gong and Guangyi Chen and Kun Zhang},
booktitle={The Thirteenth International Conference on Learning Representations},
year={2025}
}

@article{bengio2013representation,
  title={Representation learning: A review and new perspectives},
  author={Bengio, Yoshua and Courville, Aaron and Vincent, Pascal},
  journal={IEEE transactions on pattern analysis and machine intelligence},
  volume={35},
  number={8},
  pages={1798--1828},
  year={2013},
  publisher={IEEE}
}

@article{hyvarinen2013independent,
  title={Independent component analysis: recent advances},
  author={Hyv{\"a}rinen, Aapo},
  journal={Philosophical Transactions of the Royal Society A: Mathematical, Physical and Engineering Sciences},
  volume={371},
  number={1984},
  pages={20110534},
  year={2013},
  publisher={The Royal Society Publishing}
}

@article{willetts2021don,
  title={I don't need u: Identifiable non-linear ica without side information},
  author={Willetts, Matthew and Paige, Brooks},
  journal={arXiv preprint arXiv:2106.05238},
  year={2021}
}

@article{varici2023score,
  title={Score-based causal representation learning with interventions},
  author={Varici, Burak and Acarturk, Emre and Shanmugam, Karthikeyan and Kumar, Abhishek and Tajer, Ali},
  journal={Causal Representation Learning Workshop at NeurIPS 2023},
  year={2023}
}

@inproceedings{bing2024identifying,
  title={Identifying linearly-mixed causal representations from multi-node interventions},
  author={Bing, Simon and Ninad, Urmi and Wahl, Jonas and Runge, Jakob},
  booktitle={Causal Learning and Reasoning},
  pages={843--867},
  year={2024},
  organization={PMLR}
}

@article{zheng2018dags,
  title={Dags with no tears: Continuous optimization for structure learning},
  author={Zheng, Xun and Aragam, Bryon and Ravikumar, Pradeep K and Xing, Eric P},
  journal={Advances in neural information processing systems},
  volume={31},
  year={2018}
}

@book{ldrdd2025,
  title={Learning Deep Representations of Data Distributions},
  author={Buchanan, Sam and Pai, Druv and Wang, Peng and Ma, Yi},
  month=aug,
  year={2025},
  publisher={Online},
  note={https://ma-lab-berkeley.github.io/deep-representation-learning-book/.}
}

@book{spirtes2000causation,
  title={Causation, prediction, and search},
  author={Spirtes, Peter and Glymour, Clark N and Scheines, Richard},
  year={2000},
  publisher={MIT press}
}

@inproceedings{
liu2024identifiable,
title={Identifiable Latent Polynomial Causal Models through the Lens of Change},
author={Yuhang Liu and Zhen Zhang and Dong Gong and Mingming Gong and Biwei Huang and Anton van den Hengel and Kun Zhang and Javen Qinfeng Shi},
booktitle={The Twelfth International Conference on Learning Representations},
year={2024}
}

@article{mityagin2015zero,
  title={The Zero Set of a Real Analytic Function},
  author={Mityagin, BS},
  journal={Mathematical Notes},
  volume={107},
  number={3},
  pages={529--530},
  year={2020},
  publisher={Springer}
}
\bibliographystyle{unsrtnat}

\clearpage

\appendix
\thispagestyle{empty}

\onecolumn
\aistatstitle{Identifiability of potentially degenerate Gaussian Mixture Models with piecewise affine mixing: Supplementary Materials}

\section{Notation}
 Throughout this paper, we adopt the following conventions for mathematical notation: We use lowercase italic letters such as $x$ to denote scalars, bold lowercase letters such as $\xb$ to denote deterministic vectors, bold uppercase letters such as $\Xb$ to denote random vectors, and uppercase italic letters such as $X$ to denote matrices, and calligraphic uppercase letters such as $\Xcal$ to denote sets (e.g., a set of inputs or a domain). We list all specific notations in the following table:

\begin{table}[h]
\label{tab: notation}
\begin{center}
\begin{tabular}{cc}
\hline
\textbf{Symbol}  & \textbf{Description} \\ \hline
$\Zb$   & Causal latent variables \\
$\Xb$   & Observations \\
$\fb$   & Mixing function \\
$n$   & Latent space dimension \\
$d$   & Observation space dimension \\
$J$   & Number of components in mixture models\\
$\Zcal$   & Support of pdGMM latent \\
$\Zcal_j$   & Support of $j$-th component of pdGMM latent \\
$\Xcal$   & Support of observation \\
$\Kcal_j$   & Basis index for $j$-th component \\
$\zb^0$   & Translation vector \\
\hline
\end{tabular}
\end{center}
\end{table}

\section{Proofs}
\label{app: proof}

Our work focuses on Gaussian Mixture Models (GMMs), which are frequently used in machine learning models and possess a number of nice theoretical properties like closure under affine transformation etc.
Before presenting our main results, we recall one of the fundamental properties of GMMs, stated in \cref{thm:gmm-iden},  which will serve as a recurring tool in our derivations. This theorem builds a one-to-one correspondence-commonly referred to as \emph{identifiability of distributions} between the probability measure and the parameterization of a GMM. Notably, this property holds even when the mixture includes potentially degenerate components (i.e., components with singular covariance matrices), which we refer to as \emph{potentially degenerate GMMs} (pdGMMs).

To proceed, we begin by reporting again the formal definition of pdGMM in a reduced form, i.e., with distinct components. Note that $N(\mub, \Sigmab)$ stands for the Gaussian probability measure associated to parameters $\mub, \Sigmab$.
\defpdgmm

We now report a classic result on the identifiability of Gaussian Mixture Models (GMMs) on the whole domain by \citet{yakowitz1968identifiability}. While this result holds also for pdGMMs, it only proves identifiability when two GMMs are equal in distribution on the whole domain, while in this paper we will focus on identifiability from an open subset of the domain.

\begin{theorem}[Identifiable Gaussian Mixture Models - Proposition 2 in \citet{yakowitz1968identifiability}]
\label{thm:gmm-iden}
Consider a pair of finite GMMs (in reduced form) in $\RR^n$,
\begin{equation}
    P= \sum_{j=1}^J \lambda_j N(\mub_j, \Sigmab_j)\quad \text{and}\quad P'= \sum_{j=1}^{J'} \lambda_j' N(\mub_j', \Sigmab_j') \, .
\end{equation}
we have that $P = P'$ if and only if the number of mixture components are the same, i.e., $J=J'$, and for some permutation $\pi$ for all $j\in [J]$ the mixture components have the same parameters, i.e., $\lambda_j=\lambda_{\pi(j)}$ and $(\mub_j, \Sigmab_j) = (\mub_{\pi(j)}', \Sigmab_{\pi(j)}')$.
\end{theorem}

\subsection{Identifiability of pdGMMs from an open subset}
\label{app: identi pdgmm openset}
In this section, we address the following question: \emph{Can a potentially degenerate Gaussian Mixture Model (pdGMM) be identified using only information from an open subset of its domain?}

Inspired by \citet{kivva2022identifiability}, which show that a non-degenerate GMM (so a GMM with $|\Sigmab_j| > 0$ for all $j$) can be uniquely identified from its restriction to a single open subset of the domain (\cref{thm:iden-gmm-openset}), we show that a similar identifiability result also holds for pdGMMs (\cref{thm:iden-pdgmm-openset}).

However, the extension from GMMs to pdGMMs is nontrivial. The main challenge arises from the fact that the proof technique in \citet{kivva2022identifiability} critically relies on the analyticity of the probability density function (w.r.t. the Lebesgue measure) of a non-degenerate Gaussian mixture. However, mixtures of degenerate Gaussians do not have a density function, which means this strategy cannot be applied. This necessitates the development of a new proof strategy to overcome the lack of density function.

As a first step, we introduce a formal definition of equivalence in distribution over a constrained subdomain.

\defeqd

We then recall the results from \citet{kivva2022identifiability}, which show that for non-degenerate GMMs, we can identify them via an open subdomain.

\begin{theorem}[Identifiability of non-degenerate GMMs from open set \citep{kivva2022identifiability}]
\label{thm:iden-gmm-openset}
Consider a pair of finite non-degenerate GMMs in $\RR^n$ in reduced form,
\begin{equation}
    \Xb \sim \sum_{j=1}^J \lambda_j N(\mub_j, \Sigmab_j)\quad \text{and}\quad \Xb' \sim \sum_{j=1}^{J'} \lambda_j' N(\mub_j', \Sigmab_j'), 
\end{equation}
where $|\det(\Sigmab_j)|>0, \forall j\in[J]$ and $|\det(\Sigmab'_j)|>0, \forall j\in[J']$. If $\Xb\eqd\Xb'$ on $E$, where $E \subseteq \RR^n$ is an open set, then, $\Xb\eqd\Xb'$.
\end{theorem}

Next, we focus on extending these results to pdGMMs, which we will formalize in the following theorem. 

\thmidenpdgmmopenset*

As previously mentioned, proving the above theorem is challenging because, in the degenerate case, the probability density function no longer exists, making the techniques used in \citet{kivva2022identifiability} inapplicable. To overcome this difficulty, we first provide a weaker result, Theorem~\ref{thm:iden-pdgmm-openset-signlepoint}, which is requiring all supports intersect at one point at least, as a middle step to get \cref{thm:iden-pdgmm-openset}.

\begin{restatable}[Identifiability of pdGMMs from open set with common intersection]{theorem}{thmidentifopencommonintersection}
\label{thm:iden-pdgmm-openset-signlepoint}
Under the same conditions as \cref{thm:iden-pdgmm-openset}, and further assuming that there exists a point $\zb_0$, such that $\zb_0\in \cap_{j\in[J]}\Xcal_j$, 
then, the same results stated in \cref{thm:iden-pdgmm-openset} hold.
\end{restatable}

To prove Theorem~\ref{thm:iden-pdgmm-openset-signlepoint}, we build upon four key technical lemmas that we introduce and prove below.
We now provide a brief overview of our proof strategy. The core idea is to project a pdGMM in $\RR^n$ into a sequence of lower-dimensional spaces $\RR^k$, where $k\in [n]$. Intuitively, the projection serves to “resolve” the degeneracy by embedding each low-rank component into a subspace where it appears non-degenerate, thus enabling the application of classical identifiability arguments in Theorem~\ref{thm:iden-gmm-openset}.

To make this strategy work, the projections must satisfy two key properties. First, to ensure that no critical information is lost during projection (intuitively, the intrinsic structure of the component is retained under projection), the rank of the covariance matrix for the targeted Gaussian component must be preserved. Second, a single projection is insufficient to fully recover the parameters of a pdGMM; instead, we require a sufficiently rich set of rank-preserving projections to capture all distinguishable parameterizations. The existence of such a family of rank-preserving mappings is guaranteed by Lemma~\ref{lemma: keep low rank}.

\begin{lemma}
\label{lemma: keep low rank}

Suppose a positive semi-definite and symmetric matrix $\Sigmab\in \RR^{n\times n}$ has rank $k\le n$.
Let 
\begin{align}
\Bcal:=\{A\in \RR^{n\times m}: rank(A^T\Sigmab A)=k\},
\end{align}
where $k\le m\le n$. Then, we have that i) the Lebesgue measure of its complement $\Bcal^c$, denoted by $\Lcal(\Bcal^c)$ is zero, and ii) $\Bcal$ is open.
\end{lemma}
\begin{proof}
The fact that $rank(A\Sigmab A^T) = k$ means there exists at least one $k \times k$ submatrix with full rank $k$.
Consider the sets
\begin{align}
    \Acal_{I,J}:=\{A\in \RR^{n\times m}: \det((A^T\Sigmab A)_{I,J})= 0\}
\end{align}
which are the sets of matrices such that $A\Sigmab_j A^T$ has a submatrix with index sets $I,J \subseteq [m]$ with $|I| = |J|$ that is full rank. Then the complement of $\Bcal$, denoted as  $\Bcal^c$, can be rewritten as
\begin{align}
\label{equ: BA cal}
\Bcal^c=\cap_{\{I,J\subseteq [m]: |I|=|J|=k\}}\Acal_{I,J}.
\end{align}

This implies $\Lcal(\Bcal^c)\leq \Lcal(\Acal_{I,J})$, $\forall (I,J)\in \{I,J\subseteq [n]: |I|=|J|=k\}$. Therefore, if we can find one $\Acal_{I,J}$ which is measure zero, $\Bcal^c$ has Lebesgue measure zero as well.

Note that $\det((A^T\Sigmab A)_{I,J})$ is a polynomial in $A_{i,j}$, where $i \in I,j \in J$ are individual row and column indices. This is because $\det$ is a polynomial and $(A^T\Sigmab A)_{I,J}$ is a polynomial and thus composition of polynomials is a polynomial. Additionally, $\Acal_{I,J}$ is the set of the roots of this polynomial. By \cite{caron2005zero}, a polynomial from $\mathbb{R}^n$ to $\mathbb{R}$ is either zero everywhere or non-zero almost everywhere. In our context, it means if the polynomial $\det((A^T\Sigmab A)_{I,J})$ is not zero everywhere, then it is nonzero almost everywhere, i.e. its roots $\Acal_{I,J}$ has zero Lebesgue measure in $\mathbb{R}^{n \times m}$. 
Thus, all we need to show is that the polynomial $\det((A^T\Sigmab A)_{I,J})$ is not always zero. In other words, we must find a $A_0\in \RR^{n\times m}$, such that for one pair $I,J$, $\det((A^T\Sigmab A)_{I,J})\neq 0$. Then in the following, we show how to construct such a $A_0$.

Since $\rank(\Sigmab) = k$, there exists an invertible $k \times k$ submatrix $\Sigmab_{I,J}$. Let $\tilde\Ib := [\eb_1\ \cdots\ \eb_m] \in \RR^{n \times m}$ where the vectors $\eb_k$ are the elements of the standard basis of $\RR^n$ (one-hot vector). Notice that 
\begin{align}
    (\tilde\Ib^{T} \Sigmab \tilde\Ib)_{I,J} = (\tilde\Ib^{T})_{I, \cdot} \Sigmab \tilde\Ib_{\cdot, J} = (\tilde\Ib_{\cdot, I})^T \Sigmab \tilde\Ib_{\cdot, J} = \Sigmab_{I, J}  \,.
\end{align}
So we found index sets $I, J \subseteq [m]$ with $|I|=|J|=k$ and a matrix $A_0 := \tilde\Ib \in \RR^{n\times m}$ such that 
$$\det((A_0^T\Sigmab A_0)_{I,J}) = \det((\tilde\Ib^T\Sigmab \tilde\Ib)_{I,J}) = \det(\Sigmab_{I,J}) \neq 0 \,,$$
as desired. 

Hence, the polynomial function $p(A) :=  \det((A^T\Sigmab A)_{I,J})$ is not everywhere zero and thus, by \cite{caron2005zero}, it is nonzero almost everywhere. In other words, $\Acal_{I,J}$ has zero Lebesgue measure. We can conclude that $\Bcal^c$ has zero Lebesgue measure.

We now show that $\Bcal$ is open. Take the complement set on \cref{equ: BA cal}, we have 
\begin{align}
\label{equ: BA cal2}
\Bcal=\cup_{\{I,J\subseteq [n]: |I|=|J|=k\}}\Acal^c_{I,J},\\
\Acal^c_{I,J}=\{A\in \RR^{n\times m}: \det((A^T\Sigmab A)_{[I,J]})\neq 0\}.
\end{align}
Since $\det((A^T\Sigmab A)_{[I,J]})$ is a continuous function and $\Acal^c_{I,J}$ is the preimage of open set $\RR\backslash \{0\}$, $\Acal^c_{I,J}$ is open as well. As the union of finite open sets, we can conclude $\Bcal$ is open.
\end{proof}

 The following lemma will be useful many times in what follows. For a probability measure $\mu$ and a function $\fb$, $\fb(\mu)$ designates the pushforward of $\mu$ under $\fb$. The pushforward measure is defined as $\fb(\mu)(E) := \mu(\fb^{-1}(E))$. 
 We recall that as defined in Def.~\ref{def: eqd}, $\mu_E$ is a restriction of $\mu$ to $E$.

\begin{lemma}
\label{lemma: pushforward}
Let $\mu: \Fcal^{(n)} \rightarrow \RR$ be a measure on $\Fcal^{(n)}$, the Borel sigma-algebra of $\RR^n$. Let $\fb:\RR^n \rightarrow \RR^m$ be an injective function %
for any set $E\in \Fcal^{(n)}$, we have $\fb(\mu_E)=\fb(\mu)_{\fb(E)}$.
\end{lemma}
\begin{proof}
Take an arbitrary set $B \in \Fcal^{(m)}$, the Borel sigma-algebra over $\RR^m$. We have
\begin{align}
    \fb(\mu_E)(B) = \mu_E(\fb^{-1}(B)) = \mu(\fb^{-1}(B) \cap E)
\end{align}
Since $\fb$ is injective, we have that $E = \fb^{-1}(\fb(E))$.
This means
\begin{align}
    \fb(\mu_E)(B) &= \mu(\fb^{-1}(B) \cap \fb^{-1}(\fb(E))) \\
&=\mu(\fb^{-1}(B\cap \fb(E)))\\
&=\fb(\mu)(B\cap \fb(E))\\
&=\fb(\mu)_{\fb(E)}(B) \,,
\end{align}
where the second equality uses the general fact that taking the preimage preserves intersections. 

\end{proof}

The next lemma is a slightly different version from the one showed above where the function is required to be injective only on the support of the measure $\mu$.

\begin{lemma}
\label{lemma: pushforward restriction measure}
Let $\mu: \Fcal^{(n)} \rightarrow \RR$ be a measure on $\Fcal^{(n)}$, the Borel sigma-algebra of $\RR^n$. Let $\Xcal\subseteq \RR^n$ be the support of $\mu$, and let $\fb:\RR^n \rightarrow \RR^m$ be such that its restriction to $\Xcal$ is injective, %
then for any set $E\in \Fcal^{(n)}$, we have $\fb(\mu_E)=\fb(\mu)_{\fb(E \cap \Xcal)}$.
\end{lemma}
\begin{proof}
Take an arbitrary set $B \in \Fcal^{(m)}$, the Borel sigma-algebra over $\RR^m$. We have
\begin{align}
    \fb(\mu_E)(B) &= \mu_E(\fb^{-1}(B)) \\
    &= \mu(\fb^{-1}(B) \cap E) \\
    &= \mu(\fb^{-1}(B) \cap E \cap \Xcal) + \mu(\fb^{-1}(B) \cap E \cap \Xcal^c) \\
    &= \mu(\fb^{-1}(B) \cap (E \cap \Xcal)) \,,
\end{align}
where the last equality used the fact that $\mu(\Xcal^c) = 0$. Since $\fb$ is injective on $\Xcal$ and $(E \cap \Xcal) \subseteq \Xcal$, we have that $(E \cap \Xcal) = \fb^{-1}(\fb(E \cap \Xcal)) \cap \Xcal$. %
This means
\begin{align}
    \fb(\mu_E)(B) &= \mu(\fb^{-1}(B) \cap \fb^{-1}(\fb(E \cap \Xcal)) \cap \Xcal) \\
    &= \mu(\fb^{-1}(B) \cap \fb^{-1}(\fb(E \cap \Xcal)) \cap \Xcal) + \underbrace{\mu(\fb^{-1}(B) \cap \fb^{-1}(\fb(E \cap \Xcal)) \cap \Xcal^c)}_{=0} \\
    &= \mu(\fb^{-1}(B) \cap \fb^{-1}(\fb(E \cap \Xcal))) \\
&=\mu(\fb^{-1}(B\cap \fb(E \cap \Xcal)))\\
&=\fb(\mu)(B\cap \fb(E \cap \Xcal))\\
&=\fb(\mu)_{\fb(E \cap \Xcal)}(B) \,,
\end{align}
where the second equality uses the fact that $\mu(\Xcal^c)=0$ and fourth equality uses the general fact that taking the preimage preserves intersections. 

\end{proof}

The previous two lemmas show that the injectivity can ensure the exchange order of pushforward and restricted measures. In the following lemma, we show that the rank-preserving projections are injective on the domain we are interested in, and hence we can apply these lemmas.

\begin{lemma}
\label{lemma: inv on column space}
For $A\in \RR^{n\times m}$, where $m\leq n$, if $rank(\Sigmab)=\rank(A^T\Sigmab A)=k\leq m$, then, $A^T$ is injective on $col(\Sigmab)+\mub$, $\forall \mub\in \RR^n$.
\end{lemma}
\begin{proof}
Since $rank(A^T\Sigmab A)\leq rank(A^T\Sigmab)$, we get $k\leq rank(A^T\Sigmab)\leq m$. Now we consider the null space for $A^T\Sigmab$ and $\Sigmab$.

First, we know dim$(null(A^T\Sigmab))\in[n-m,n-k]$ and dim$(null(\Sigmab))=n-k$. Second, $\forall x\in \RR^n$, if $\Sigmab x=0$, then $A^T\Sigmab x=0$. This means $null(\Sigmab)\subseteq null(A^T\Sigmab)$. Combined with the first point, we can get dim$(null(A^T\Sigmab))=n-k$. Since both of them are affine subspaces with the same dimension and one is the subset of the other, we get $null(\Sigmab)= null(A^T\Sigmab)$. This means
\begin{align}
\forall y\in \RR^n, \text{ if }A^T\Sigmab y=0 \quad \Rightarrow \quad \Sigmab y=0
\end{align}
Denote $x:=\Sigmab y\in col(\Sigmab)$, then we have
\begin{align}
\label{equ: x=0}
\forall x\in col(\Sigmab), \text{ if }A^Tx=0 \quad \Rightarrow \quad x=0.
\end{align}
Now consider two distinct points $x_1, x_2\in col(\Sigmab)+\mub$. As $x_1-x_2\in col(\Sigmab)$ but not equal to 0, by \cref{equ: x=0}, we have 
\begin{align}
A^T (x_1-x_2)\neq 0 \quad\Rightarrow \quad A^T x_1\neq A^T x_2,
\end{align}
which means $A^T$ is injective on $col(\Sigmab)+\mub$.
\end{proof}

In the following lemma, we show that only the components with well-defined density can be equal to each other, and the same for the components without density. Intuitively, we can group the Gaussian components by rank and analyze each group separately. To discuss the existence of well-defined density, we recall the Radon-Nikodym theorem: informally, if measure $\mu$ is absolutely continuous with respect to measure $\Lcal$, denoted as $\mu \ll \Lcal$, then there exists a density function of $\mu$ w.r.t. $\Lcal$. In other words, $\mu \ll \Lcal$ means for any set $A$, if $\Lcal(A)=0$, then $\mu(A)=0$.

\begin{lemma}
\label{lemma: w/wo density}
Let $\mu$, $\tilde{\mu}$, $\lambda$,  $\tilde{\lambda}$ be nonnegative finite measures on a common measure space such that
\begin{align}
\label{equ: condi mix}
\mu + \lambda=\tilde{\mu}+\tilde{\lambda}.
\end{align}
Let $\Lambda$ and $\tilde{\Lambda}$ be the supports of $\lambda$ and $\tilde{\lambda}$, respectively.
If there is a third measure $\Lcal$ over the same space such that
(i) $\mu, \tilde{\mu} \ll \Lcal$, and
(ii) $\Lcal(\Lambda)=\Lcal(\tilde{\Lambda})=0$,
then $\mu=\tilde{\mu}$ and $\lambda=\tilde{\lambda}$.
\end{lemma}
\begin{proof}
First, we rewrite the equation equivalently as:
\begin{align}
\label{equ: compo measures}
\mu - \tilde{\mu}=\tilde{\lambda}-\lambda.
\end{align}
Let $B$ be an arbitrary measurable set. We have
\begin{align}
(\mu - \tilde{\mu})(B)&=(\tilde{\lambda}-\lambda)(B) \\
&= (\tilde{\lambda}-\lambda)(B \cap (\Lambda \cup \tilde\Lambda)) + (\tilde{\lambda}-\lambda)(B \cap (\Lambda \cup \tilde\Lambda)^c) \label{eq:03ijds}
\end{align}
Since $B\cap(\Lambda\cup \tilde{\Lambda})^c$ has no overlap with the support of neither $\lambda$ nor $\tilde \lambda$, we have  $(\tilde{\lambda}-\lambda)(B\cap(\Lambda\cup \tilde{\Lambda})^c)=0$. This means
\begin{align}
(\mu - \tilde{\mu})(B)&= (\tilde{\lambda}-\lambda)(B \cap (\Lambda \cup \tilde\Lambda)) \\
&= (\mu-\tilde{\mu})(B \cap (\Lambda \cup \tilde\Lambda))\\
&= \mu(B \cap (\Lambda \cup \tilde\Lambda))-\tilde{\mu}(B \cap (\Lambda \cup \tilde\Lambda))\,,
\end{align}
where the second equality holds by \cref{equ: compo measures}. Note that 
$$\Lcal(B \cap (\Lambda \cup \tilde\Lambda)) \leq \Lcal(\Lambda \cup \tilde\Lambda)) \leq \Lcal(\Lambda) + \Lcal(\tilde\Lambda) = 0\, ,$$
which means $\Lcal(B \cap (\Lambda \cup \tilde\Lambda)) = 0$. Since $\mu, \tilde\mu \ll \Lcal$, we must have $\mu(B \cap (\Lambda \cup \tilde\Lambda)) = \tilde\mu(B \cap (\Lambda \cup \tilde\Lambda)) =0$. This combined with \cref{eq:03ijds} implies $(\mu - \tilde\mu)(B) = 0$, i.e. $\mu(B) = \tilde\mu(B)$. This is true for any $B$, thus $\mu = \tilde\mu$. Plugging this last fact into \cref{equ: compo measures} yields, $\tilde\lambda = \lambda$.

\end{proof}

Now we are ready to prove our intermediate result,  \cref{thm:iden-pdgmm-openset-signlepoint}, that proves the identifiability of pdGMMs from an open set with a common intersection.

\thmidentifopencommonintersection*

\begin{proof}[Proof]
Let $k_j:=rank(\Sigmab_j)$ and $k'_j:=rank(\Sigmab'_j)$. W.l.o.g., suppose that the components in each of the two pdGMM are ordered by rank in a decreasing fashion, i.e., $k_1\geq k_2\geq...\geq k_J$, $k'_1\geq k'_2\geq...\geq k'_{J'}$ and that the rank of the first component of the first pdGMM $\Xb$ is larger or equal to the rank of the first component of the second pdGMM $\Xb'$, i.e., 
 $k_1\geq k'_1$.

We define 
\begin{align}\Scal_k:=\{j\in[J]: k_j\leq k\}
\end{align}
as the index set of all components in the first pdGMM $\Xb$, which have rank lower than or equal to $k$. Similarly we define the same index set for $\Xb'$ as \begin{align}
\Scal'_k:=\{j\in[J']: k'_j\leq k\}.
\end{align}
We define the set of linear mappings that preserve the rank for the component $j$ of $\Xb'$ as
\begin{align}\Bcal_{k,j}:=\{A\in \RR^{n\times k}\mid rank(A^T\Sigmab_jA)=k_j\}.
\end{align}

We now define the intersection of all the rank-preserving matrices for the components that have rank lower than or equal to $k$ as
\begin{align}\Bcal_k:=\cap_{\{j\in\Scal_k\}}\Bcal_{k,j}.
\end{align}
Similarly for $\Xb'$, we define 
\begin{align}
\Bcal'_k :=\cap_{\{j\in\Scal'_k\}}\Bcal'_{k,j} \text{ with  }
\Bcal'_{k,j}:=\{A\in \RR^{n\times k}\mid rank(A^T\Sigmab'_jA)=k'_j\}.
\end{align}

$\Bcal_k$ and $\Bcal'_k$ involve the linear mappings that preserve the ranks for the components with rank no larger than $k$. Furthermore, to ensure the components in $\Xb$ are still different from each other after mapping, we first consider a pair of components in $\Xb$ and define
\begin{align}
\Ccal_{k,i,j}:=\{A\in \RR^{n\times k}\mid A^T(\mub_i-\mub_j)\neq \mathbf{0} \text{ or } A^T(\Sigmab_i-\Sigmab_j)A \neq \mathbf{0}\},
\end{align}
and then go through all possible pairs and define
\begin{align}
\Ccal_k:=\cap_{\{i\neq j\in\Scal_k:k_i,k_j= k\}}\Ccal_{k,i,j}.
\end{align}
Similarly for $\Xb'$, we define 
\begin{align}
\Ccal_k' &:=\cap_{\{i\neq j\in\Scal'_k:k'_i,k'_j=k\}}\Ccal'_{k,i,j}\\
\text{ with }\Ccal'_{k,i,j}&:=\{A\in \RR^{n\times k}\mid A^T(\mub'_i-\mub'_j)\neq \mathbf{0} \text{ or } A^T(\Sigmab'_i-\Sigmab'_j)A \neq \mathbf{0}\}.
\end{align}

We now analyze the set $\Acal_k\subseteq \RR^{n\times k}$ where $k>0$. Define
\begin{align}
&\Acal_k:=\Bcal_k \cap \Bcal'_k\cap \Ccal_k \cap \Ccal'_k.
\end{align}

We will show that i) $\RR^{n\times k}\backslash \Acal_k$ has Lebesgue measure zero, and ii) $\Acal$ is an open set in $\RR^{n\times k}$. We analyze each of the components 

Let $\Lcal$ be the Lebesgue measure. For $\Bcal_k$, by Lemma~\ref{lemma: keep low rank}, we know $\Lcal(\Bcal_{k,j}^c)=0$ for all $j\in[J]$. Thus, we have 
\begin{align}
0 \leq\Lcal(\Bcal_k^c)=\Lcal(\cup_{j\in\Scal_k}\Bcal_{k,j}^c) \leq \sum_{j\in\Scal_k}\Lcal(\Bcal_{k,j}^c)=0
\end{align}
which of course implies $\Lcal(\Bcal_k^c) = 0$.
Next, from Lemma~\ref{lemma: keep low rank}, we also know $\Bcal_{k,j}$ is open for all $j\in[J]$.
As $\Bcal_k$ is the finite intersection of $\Bcal_{k,j}$'s, we know $\Bcal_k$ is open as well.
We can apply the same arguments and obtain i) $\Lcal(\Bcal_k'^c)=0$, and ii) $\Bcal_k'$ is an open set.

For $\Ccal_{k,i,j}$, we separate it into two cases:
\begin{align}
&\Ccal_{k,i,j}=\Ccal_{k,i,j}^{\mu}\cup \Ccal_{k,i.j}^{\sigma},
\label{equ: def Ccal}\\
\text{where  }&\Ccal_{k,i,j}^{\mu}:=\{A\in\RR^{n\times k}\mid A^T(\mub_i-\mub_j)\neq \mathbf{0}\},\\
\text{and  }&\Ccal_{k,i,j}^{\sigma}:=\{A\in\RR^{n\times k}\mid A^T(\Sigmab_i-\Sigmab_j)A\neq \mathbf{0}\}.
\end{align}
If $\mub_i\neq\mub_j$, $A^T(\mub_i-\mub_j)$ can be regarded as $m$ polynomial functions of $A$ from $\RR^n$ to $\RR$, i.e.
\begin{align}
A^T(\mub_i-\mub_j)=
\begin{bmatrix}
    A^T_{\cdot1}(\mub_i-\mub_j)\\
    A^T_{\cdot2}(\mub_i-\mub_j)\\
    \vdots\\
    A^T_{\cdot k}(\mub_i-\mub_j)
\end{bmatrix}.
\end{align}

Each $A_{\cdot i}(\mub_i-\mub_j)$, $i\in[k]$ is not zero everywhere as $(\mub_i-\mub_j)\neq\mathbf{0}$. By \cite{caron2005zero}, a polynomial from $\mathbb{R}^n$ to $\mathbb{R}$ is either zero everywhere or non-zero almost everywhere. Apply to our functions, it means they are all nonzero almost everywhere. Denote the zero set as $\Ocal_{\mu}$ and we know $\Lcal(\Ocal_{\mu})=0$. Since these polynomial functions all share the same coefficients, we have 
\begin{align}(\Ccal_{k,i,j}^{\mu})^c=\underbrace{\Ocal_{\mu}\times \cdots \times \Ocal_{\mu} }_{m}=\Ocal_{\mu}^m,
\end{align}
which means
\begin{align}
\Lcal((\Ccal_{k,i,j}^{\mu})^c)=\Lcal(\Ocal_{\mu}^m)=\Lcal(\Ocal_{\mu})^m=0.
\end{align}
If $\Sigmab_i\neq \Sigmab_j$, then consider the diagonal elements of $ A^T(\Sigmab_i-\Sigmab_j)A$ as $n$ polynomial functions of $A$, denote as $p_{ii}:=A_{\cdot i}^T(\Sigmab_i-\Sigmab_j)A_{\cdot i}$, $\forall i\in[m]$. Since $\Sigmab_i-\Sigmab_j\neq \mathbf{0}$, $p_{ii}$ is not zero everywhere. By \cite{caron2005zero}, they are all nonzero almost everywhere. Denote the zero set as $\Ocal_{\sigma}$ and we know $\Lcal(\Ocal_{\sigma})=0$. Since they are only the diagonal elements, we have
\begin{align}
(\Ccal_{k,i,j}^{\sigma})^c \subseteq \Ocal_{\sigma}^m \quad \Rightarrow \quad 0 \leq \Lcal((\Ccal_{k,i,j}^{\sigma})^c)\leq \Lcal(\Ocal_{\sigma}^m)=\Lcal(\Ocal_{\sigma})^m=0,
\end{align}
which means $\Lcal((\Ccal_{k,i,j}^{\sigma})^c)=0$.

By \cref{equ: def Ccal}, we have 
\begin{align}
\Ccal_{k,i,j}^c=(\Ccal_{k,i,j}^{\mu})^c\cap (\Ccal_{k,i,j}^{\sigma})^c.
\end{align}
As we assume, $\Xb$ is in reduced form, which means $\mub_i\neq\mub_j$ or $\Sigmab_i\neq \Sigmab_j$. This implies at least one of $(\Ccal_{ij}^{\mu})^c$ and $(\Ccal_{ij}^{\sigma})^c$ is Lebesgue measure zero. Then, as the intersection of these two sets, we have $\Lcal(\Ccal_{ij}^c)=0$. Since the union of measure zero sets is measure zero, we have $\Lcal(\Ccal^c)=0$.

Since $A^T(\mub_i-\mub_j)$ is continuous function and $(\Ccal_{k,i,j}^{\mu})$ is the preimage of open set $\RR^{k}\backslash \{\mathbf{0}$\}, then $(\Ccal_{k,i,j}^{\mu})$ is open. Similarly, we have $(\Ccal_{k,i,j}^{\sigma})$ is open as well. As the union of two open sets, $\Ccal_{k,i,j}$ is also open. Then as the finite intersection of $\Ccal_{k,i,j}$, we obtain $\Ccal)k$ is an open set.

For $\Ccal'_k$, we can apply the same arguments and obtain i) $\Lcal(\Ccal'^c_k)=0$, and ii) $\Ccal_k'$ is an open set.

To sum up, all the above results mean that $\Lcal(\Acal^c_k)=\Lcal(\Bcal_k^c \cup \Bcal'^c_k\cup \Ccal^c_k \cup \Ccal'^c_k)=0$ and thus $\Acal_k$ is not empty and, furthermore, $\Acal_k=\Bcal_k \cap \Bcal'_k\cap \Ccal_k \cap \Ccal'_k$ is an open set in $\RR^{n\times k}$.

Now, we have $\Xb \eqd \Xb'$ on $E$, by Def.~\ref{def: eqd},  which means
\begin{align}
\left(\sum_{j=1}^J \lambda_j N(\mub_j, \Sigmab_j)\right)_E&= \left(\sum_{j=1}^{J'} \lambda_j' N(\mub_j', \Sigmab_j')\right)_E\\
 \Rightarrow \quad\sum_{j=1}^J \lambda_j N(\mub_j, \Sigmab_j)_E&= \sum_{j=1}^{J'} \lambda_j' N(\mub_j', \Sigmab_j')_E. 
 \label{equ: constraint E0}
\end{align}

We start from the highest rank $k_1$. Since $\Acal_{k_1}$ is non-empty, it has an element $A_1\in \Acal_{k_1}$. Apply $A_1$ on both sides of \cref{equ: constraint E0} and we get
\begin{align}
A_1^T\sum_{j=1}^J \lambda_j N(\mub_j, \Sigmab_j)_E&=  A_1^T\sum_{j=1}^{J'} \lambda_j' N(\mub_j', \Sigmab_j')_E\\
\sum_{j=1}^J \lambda_j A_1^T (N(\mub_j, \Sigmab_j)_E)&=  \sum_{j=1}^{J'} \lambda_j' A_1^T( N(\mub_j', \Sigmab_j')_E)\\
\Rightarrow \quad \sum_{j=1}^J \lambda_j A_1^T (N(\mub_j, \Sigmab_j)_{\Xcal_j\cap E})&=  \sum_{j=1}^{J'} \lambda_j' A_1^T (N(\mub_j', \Sigmab_j')_{\Xcal'_j\cap E}). \label{equ: constraint E}
\end{align}

Due to $\Bcal_k\cap\Bcal_k'$, we have $rank(A_1^T\Sigmab_j A_1)=k_j$ for all $j \in [J]$ and $rank(A_1^T\Sigmab'_j A_1)=k'_j$ for all $j \in [J']$ . Thus, we can apply Lemma~\ref{lemma: inv on column space}, and obtain $\A_1$ is injective on $col(\Sigmab_j)+\mub_j$, which is exactly $\Xcal_j$.

Then, we can apply Lemma~\ref{lemma: pushforward restriction measure} on \cref{equ: constraint E} and get 
\begin{align}
\label{equ: constraint E 2}
\sum_{j=1}^J \lambda_j A_1^T (N(\mub_j, \Sigmab_j)_{\Xcal_j\cap E})&=  \sum_{j=1}^{J'} \lambda_j' A_1^T (N(\mub_j', \Sigmab_j')_{\Xcal'_j\cap E}) \\
\sum_{j=1}^J \lambda_j (A_1^TN(\mub_j, \Sigmab_j))_{A_1^T(\Xcal_j\cap E)}&=  \sum_{j=1}^{J'} \lambda_j' (A_1^T N(\mub_j', \Sigmab_j'))_{A_1^T(\Xcal'_j\cap E)} \\
\sum_{j=1}^J \lambda_j  N(A_1^T\mub_j, A_1^T\Sigmab_jA_1)_{A_1^T(\Xcal_j\cap E)}&=
\sum_{j=1}^{J'} \lambda_j' N(A_1^T\mub_j', A_1^T\Sigmab_j'A_1)_{A_1^T(\Xcal'_j\cap E)}.
\end{align}

Next, we split both sides into two terms by the rank value $k_j$ and $k'_j$
\begin{align}
\label{equ: constraint E 4_1}
&\sum_{\{j\in[J]: k_j=k_1\}} \lambda_j  N(A_1^T\mub_j, A_1^T\Sigmab_jA_1)_{A_1^T(\Xcal_j\cap E)}+
\sum_{\{j\in[J]: k_j<k_1\}} \lambda_j  N(A_1^T\mub_j, A_1^T\Sigmab_jA_1)_{A_1^T(\Xcal_j\cap E)}\\
=
&\sum_{\{j\in[J']:k_j'=k_1\}} \lambda_j' N(A_1^T\mub_j', A_1^T\Sigmab_j'A_1)_{A_1^T(\Xcal'_j\cap E)}+
\sum_{\{j\in[J']:k_j'<k_1\}} \lambda_j' N(A_1^T\mub_j', A_1^T\Sigmab_j'A_1)_{A_1^T(\Xcal'_j\cap E)} \,.
\end{align}

For the terms in the first sum on both sides, we have $\det(A^T\Sigmab_jA)\neq0$, which implies non-singular covariance matrices after transformation. This means $N(A_1^T\mub_j, A_1^T\Sigmab_jA_1) \ll \Lcal$ and $N(A_1^T\mub_j', A_1^T\Sigmab_j'A_1) \ll \Lcal$, where $\Lcal$ is the Lebesgue measure\footnote{This is another of saying the distribution has a density w.r.t. $\Lcal$, by the Radon-Nikodym theorem}. In general, if $\lambda \ll \mu$, we also have $\lambda_E \ll \mu$ and thus we have 
\begin{align}
    N(A_1^T\mub_j, A_1^T\Sigmab_jA_1)_{A_1^T(\Xcal_j\cap E)} \ll \Lcal \quad \text{and} \quad N(A_1^T\mub_j', A_1^T\Sigmab_j'A_1)_{A_1^T(\Xcal'_j\cap E)} \ll \Lcal
\end{align}

For the terms in the second sum of both sides, we have 
\begin{align}
rank(A_1^T\Sigmab_jA_1)= k_j <k_1, \quad rank(A_1^T\Sigmab'_jA_1)= k'_j <k_1,
\end{align}
which means after transformation, these components are still degenerate. Importantly, it means the support of $N(A_1^T\mub_j, A_1^T\Sigmab_jA_1)$ and $N(A_1^T\mub_j', A_1^T\Sigmab_j'A_1)$ have zero Lebesgue measure. By restricting further, the support of $N(A_1^T\mub_j, A_1^T\Sigmab_jA_1)_{A_1^T(\Xcal_j\cap E)}$ and $N(A_1^T\mub_j', A_1^T\Sigmab_j'A_1)_{A_1^T(\Xcal'_j\cap E)}$ still have measure zero.

Therefore, we can apply Lemma~\ref{lemma: w/wo density} and obtain the equality for the terms with density
\begin{align}
\label{equ: constraint E 3}
\sum_{\{j\in[J]:k_j= k_1\}} \lambda_j  N(A_1^T\mub_j, A_1^T\Sigmab_jA_1)_{A_1^T(\Xcal_j\cap E)} =
\sum_{\{j\in[J']:k'_j= k_1\}} \lambda_j' N(A_1^T\mub_j', A_1^T\Sigmab_j'A_1)_{A_1^T(\Xcal'_j\cap E)}.
\end{align}

Let $U:=\cap_{\{j\in [J]:k_j=k_1\}}A_1^T(\Xcal_j\cap E)$. Since $A_1^T(\Xcal_j\cap E)$ is an open set in $\RR^{k_1}$ for all $\{j\in [J]:k_j=k_1\}$, and have common element $A_1\zb_0$ by condition. Thus $U$ is an open and non-empty set.
We further restrict the measure of \cref{equ: constraint E 3} to $U$:
\begin{align}
\sum_{\{j\in [J]:k_j=k_1\}} \lambda_j  N_{U}(A_1^T\mub_j, A_1^T\Sigmab_jA_1)&=
\sum_{\{j\in [J]:k'_j=k_1\}} \lambda_j' N_{U}(A_1^T\mub_j', A_1^T\Sigmab_j'A_1) \\
\left(\sum_{\{j\in [J]:k_j=k_1\}} \lambda_j  N(A_1^T\mub_j, A_1^T\Sigmab_jA_1) \right)_{U}&= \left(
\sum_{\{j\in [J]:k'_j=k_1\}} \lambda_j' N(A_1^T\mub_j', A_1^T\Sigmab_j'A_1) \right)_{U}.
\label{equ: constraint E 4}
\end{align}
Since each component in \cref{equ: constraint E 4} has density functions, and they are all different due to $\Ccal\cap \Ccal'$. As we discussed before, $U$ is an open and non-empty set, thus by  \cref{thm:iden-gmm-openset}, we can get 
\begin{align}
\label{equ: constraint E 5}
\sum_{\{j\in [J]:k_j=k_1\}} \lambda_j  N(A_1^T\mub_j, A_1^T\Sigmab_jA_1) = 
\sum_{\{j\in [J]:k'_j=k_1\}} \lambda_j' N(A_1^T\mub_j', A_1^T\Sigmab_j'A_1),
\end{align}
which means $|\{j\in[J]: k_j=k_1\}|=|\{j\in[J']: k_j'=k_1\}|$, and there exits a permutation $\pi:\{j\in[J]: k_j=k_1\} \rightarrow \{j\in[J']: k_j'=k_1\}$, such that 
\begin{align}
(A_1^T\mub_j, A_1^T\Sigmab_jA_1)=(A_1^T\mub_{\pi(j)}', A_1^T\Sigmab_{\pi(j)}'A_1), \quad \lambda_j =\lambda'_{\pi(j)}, \quad \forall A_1\in \Acal_{k_1}.
\end{align}

Consider two functions:
\begin{align}
f_1(A_1):=(A_1^T(\mub_j-\mub'_{\pi(j)}))_1=0,\quad 
f_2(A_1):=\left(A_1^T(\Sigmab_j-\Sigmab'_{\pi(j)})A_1\right)_{1,1}=0, \quad \forall A_1\in \Acal_{k_1}.
\end{align}
Since these hold for all $A_1\in \Acal_{k_1}$, and $\Acal_{k_1}$ is open and non-empty. Thus, we can take the first and second derivatives on $f_1$ and $f_2$, respectively, w.r.t. the elements in the first column of $A_1$. We get
\begin{align}
f'_1(A_1):=\mub_j-\mub'_{\pi(j)}=\mathbf{0}^{n\times 1},\quad 
f''_2(A_1)=\Sigmab_j-\Sigmab'_{\pi(j)}=\mathbf{0}^{n\times n}, \quad \forall A_1\in \Acal_{k_1}.
\end{align}
Thus we conclude 
\begin{align}
\exists \pi:\{j\in[J]:k_j=k_1\}\rightarrow \{j\in[J']:k'_j= k_1\}, s.t. \mub_j=\mub'_{\pi(j)},\quad 
\Sigmab_j=\Sigmab'_{\pi(j)}, \quad \lambda_j =\lambda'_{\pi(j)}.
\end{align}
Now go back to \cref{equ: constraint E}, we can cancel out the terms with rank $k_1$ both sides and get 
\begin{align}
\sum_{\{j\in[J]:k_j< k_1\}}  \lambda_j  N_E(\mub_j, \Sigmab_j)=  \sum_{\{j\in[J']:k'_j< k_1\}}  \lambda_j' N_E(\mub_j', \Sigmab_j')
\end{align}

W.l.o.g, suppose $k_2\geq k'_2$, the same argument can apply to $k_2$. And continue progressing from the largest to the smallest rank to cover all elements in both $\Xb$ and $\Xb'$, and finally, we can get 
\begin{align} 
J=J', \text{ and }
\exists \pi:[J]\rightarrow [J'], s.t. \mub_j=\mub'_{\pi(j)},\quad 
\Sigmab_j=\Sigmab'_{\pi(j)}, \quad \lambda_j =\lambda'_{\pi(j)},
\end{align}
and by \cref{thm:gmm-iden}, we can conclude $\Xb\eqd \Xb'$.
\end{proof}

We have concluded the proof of \cref{thm:iden-pdgmm-openset-signlepoint}, which proves the identifiability of two pdGMMs from an open set assuming a common non-empty intersection set for all components.
Based on these results, we can prove our final goal of this section: \cref{thm:iden-pdgmm-openset}, which removes the assumption of the common non-empty intersection set for all components. The basic idea is to group the components by whether they intersect on the same point and apply \cref{thm:iden-pdgmm-openset-signlepoint} on each group. 

\thmidenpdgmmopenset*

\begin{proof}[Proof]
We first focus on $N(\mub_1, \Sigmab_1)$. Pick up one point $\zb_1\in \Xcal_1\cap E$, and define index set
\begin{align}
u_1:=\{j\in [J]:\zb_1\in \Xcal_j\}, \quad u'_1:=\{j\in [J']:\zb_1\in \Xcal'_j\}.
\end{align}
Then, we can construct a open ball $B(\zb_1,\delta)\subset E$, such that $\forall j\in[J]\backslash u_1$, $\Xcal_j\cap B(\zb_1,\delta)=\emptyset$. By condition we have $\Xb\eqd \Xb'$ on $E$, we restrict it on $B(\zb_1,\delta)$ and eliminate the zero measures on both sides, then we get  
\begin{align}
 &\sum_{j\in u_1} \lambda_j N_B(\mub_j, \Sigmab_j)=\sum_{j\in u_1'} \lambda_j' N_B(\mub_j', \Sigmab_j')\\
 &\left(\sum_{j\in u_1} \lambda_j N(\mub_j, \Sigmab_j)\right)_B=\left(\sum_{j\in u_1'} \lambda_j' N(\mub_j', \Sigmab_j')\right)_B.
\end{align}
By \cref{thm:iden-pdgmm-openset-signlepoint}, we have 
\begin{align}
 \sum_{j\in u_1} \lambda_j N(\mub_j, \Sigmab_j)=\sum_{j\in u_1'} \lambda_j' N(\mub_j', \Sigmab_j'), 
\end{align}
also $|u_1|=|u_1'|$, and there exists a permutation $\pi_1:u_1\rightarrow u_1'$, such that $(\lambda_j,\mub_j,\Sigmab_j)=(\lambda'_{\pi_1(j)},\mub'_{\pi_1(j)},\Sigmab'_{\pi_1(j)})$.

Then, we move to the rest $t\in[J]\backslash \{1\}$. Pick up one point $\zb_t\in \Xcal_t\cap E$, and define index set
\begin{align}
u_t:=\{j\in [J]\backslash (\cup_{j<t}u_j):\zb_t\in \Xcal_j\}, \quad u'_t:=\{j\in [J']\backslash (\cup_{j<t}u'_j):\zb_t\in \Xcal'_j\}.
\end{align}
If $u_t$ is not empty, this means we haven't idenfity component $t$ yet. Then as before, we construct a open ball $B(\zb_t,\delta)\subset E$, such that $\forall j\in[J]\backslash u_t$, $\Xcal_j\cap B(\zb_t,\delta)=\emptyset$. We restrict it on $B(\zb_t,\delta)$ and apply \cref{thm:iden-pdgmm-openset-signlepoint} to get   
\begin{align}
 \sum_{j\in u_t} \lambda_j N(\mub_j, \Sigmab_j)=\sum_{j\in u_t'} \lambda_j' N(\mub_j', \Sigmab_j'), 
\end{align}
also $|u_t|=|u_t'|$, and there exists a permutation $\pi_t:u_t\rightarrow u_t'$, such that $(\lambda_j,\mub_j,\Sigmab_j)=(\lambda'_{\pi_t(j)},\mub'_{\pi_t(j)},\Sigmab'_{\pi_t(j)})$.

By going through all $t\in[J]$, we can combine all results to conclude that $\Xb\eqd \Xb'$. Also $J=J'$, and there exists a permutation $\pi:J\rightarrow J'$, such that $(\lambda_j,\mub_j,\Sigmab_j)=(\lambda'_{\pi(j)},\mub'_{\pi(j)},\Sigmab'_{\pi(j)})$.
\end{proof}

\subsection{Proofs for the identifiability of pdGMM latent variables}
\label{app: identis compare}
In this section, we consider the setting where latent variables follow pdGMMs and aim to identify the them from observations. In \cref{app: affine identi} we provide the proof of \cref{thm:affine-identi-comp}, which shows the identifiability up to affine transformation within components. In \cref{app: global affine}  we provide the proof of \cref{thm:global-affine-identi}, which shows the identifiability up to affine transformation. Finally, in \cref{app: PD identi}  we provide the proof of \cref{thm: PD identi}, which shows the identifiability up to permutation and element-wise affine transformation.

\subsubsection{Affine Identifiability within component of pdGMM latent variables mixed by piecewise affine function}
\label{app: affine identi}
In this section, we consider a set of latent variables that follow pdGMM and are transformed into observations via a continuous, invertible, piecewise affine mixing function. Our objective is to recover the latent variables from the observed data, without any auxiliary information, except for knowing that the latent variables are pdGMM.

First, following \citet{bona2023parameter}, we define continuous piecewise affine functions by \emph{closed polyhedra} as follows:

\begin{definition}[Closed polyhedron \citep{bona2023parameter}]
\label{def: polyhedron}
We say $P\subset \RR^n$ is a \emph{closed polyhedron} if and only if there exists $q\in \NN$, $\ab_1,...,\ab_q\in \RR^n$ and $b_1,...,b_q\in \RR$ such that
\begin{align}
\xb\in P \quad \Leftrightarrow\quad \left\{
\begin{aligned}
&\ab_1^T\xb+b_1 \leq 0 \\
&\vdots \\
&\ab_q^T\xb+b_q \leq 0.
\end{aligned}
\right.
\end{align}
\end{definition}

\begin{definition}[Continuous piecewise affine function \citep{bona2023parameter}]
\label{def: continue piecewise affine}
We say that a function $\fb: \RR^n\rightarrow \RR^m$ is a \emph{continuous piecewise affine} function if there exists a finite set of
closed polyhedra whose union is $\RR^n$ and such that $\fb$ is affine over each polyhedron.
\end{definition}

A simple application of the pasting lemma (see e.g. \citet{Munkres2000Topology}) shows that a continuous piecewise affine function is indeed continuous. For a continuous piecewise affine function $\fb$, there are infinitely many finite covers of $\mathbb{R}^n$ by closed polyhedra such that $\fb$ is affine over each polyhedron. We will be interested into a specific type of covers called \textit{admissible}.

\begin{definition}
\label{def:admissible}
[\citep{bona2023parameter}]
Let $\fb$ be a continuous piecewise affine function. Let $\Pi$ be a finite set of closed polyhedra of $\RR^n$. We say that $\Pi$ is \textit{admissible} with respect to $\fb$ if and only if
\begin{enumerate}
    \item $\bigcup_{P \in \Pi} P = \RR^n$;
    \item for all $P \in \Pi$, $\fb$ is affine on $P$; and
    \item for all $P \in \Pi$, the interior of $P$, denoted by $P^\circ$, is not empty.
\end{enumerate}
\end{definition}

\citet[Proposition 22]{bona2023parameter} show that such an admissible set of closed polyhedra exists for all continuous piecewise affine functions.  

We first present a result by \citet{kivva2022identifiability} that studies the identifiability of non-degenerate GMM latent variables with piecewise affine mixing function, which serves as a primary inspiration for our approach.
\begin{theorem}[Affine Identifiability of GMM latent variables \citep{kivva2022identifiability}]
\label{thm: affine identi gmm}
Assume the observation $\Xb = \fb (\Zb)$ follows the data-generating process in Sec.~\ref{sec: Problem set up}. Suppose $\Zb$ follows a Gaussian mixture model 
\begin{equation}
    \Zb \sim \sum_{j=1}^J \lambda_j N(\mub_j, \Sigmab_j)
\end{equation}
    with $\mub_j\in \RR^n$, $\Sigmab_j\in \RR^{n\times n}$ and $|\Sigmab_j|> 0$
    , $\lambda_j>0$, and $\sum_{j=1}^J \lambda_j=1$. For every $i \neq j \in [J]$ we have $(\mub_i,\Sigmab_i)\neq(\mub_j,\Sigmab_j)$. Then $\Zb$ is identifiable from $\Xb$ up to an affine transformation.
   
\end{theorem}
However, the analysis of \cref{thm: affine identi gmm} is restricted to non-degenerate Gaussian mixture models, assuming full-rank covariance matrices for all components.

We begin by presenting an intermediate theorem (Thm.~\ref{thm: affine for components}), which shows that any continuous invertible piecewise affine function capable of transforming one pdGMM into another pdGMM must be affine on the support of each component in pdGMM. 
However, this property does not hold universally. We show now a counterexample in Example~\ref{exa: violate genericity}, where $\fb$ is not affine over the support of any Gaussian components but still preserves the same distribution.

\begin{example}
\label{exa: violate genericity}
Consider a pdGMM with dimension $n=2$ and Gaussian components $J=2$,
\begin{equation}
  \Zb \sim 0.5N((0,0),\begin{pmatrix}1 &0\\0 &0\end{pmatrix} )+0.5N((0,0),\begin{pmatrix}0 &0\\0 &1\end{pmatrix} )
\end{equation}
Denote the two dimensions as $\zb_1$ and $\zb_2$. One can construct a function $\fb: \Zcal \rightarrow \RR^2$ that is clearly not affine on $\Zcal$:
\begin{equation}
  \fb(\zb) = \begin{cases}
   \begin{pmatrix}1 &0\\0 &1\end{pmatrix}\zb, \text{ if $\zb_1<0$ and $\zb_2=0$,}\\
   \begin{pmatrix}-1 &0\\0 &-1\end{pmatrix}\zb, \text{ if $\zb_1=0$ and $\zb_2>0$,}\\
  \begin{pmatrix}0 &-1\\1 &0\end{pmatrix}\zb, \text{ if $\zb_1=0$ and $\zb_2<0$ or $\zb_1>0$ and $\zb_2=0$.} 
  \end{cases}
\end{equation}
In other words, $\fb$ switches the relative position of axes, while retaining the same distribution, i.e. $\fb(\Zb) \eqd \Zb$. 
\end{example}
\begin{figure}[t]
\centering
\includegraphics[width=0.8\textwidth]{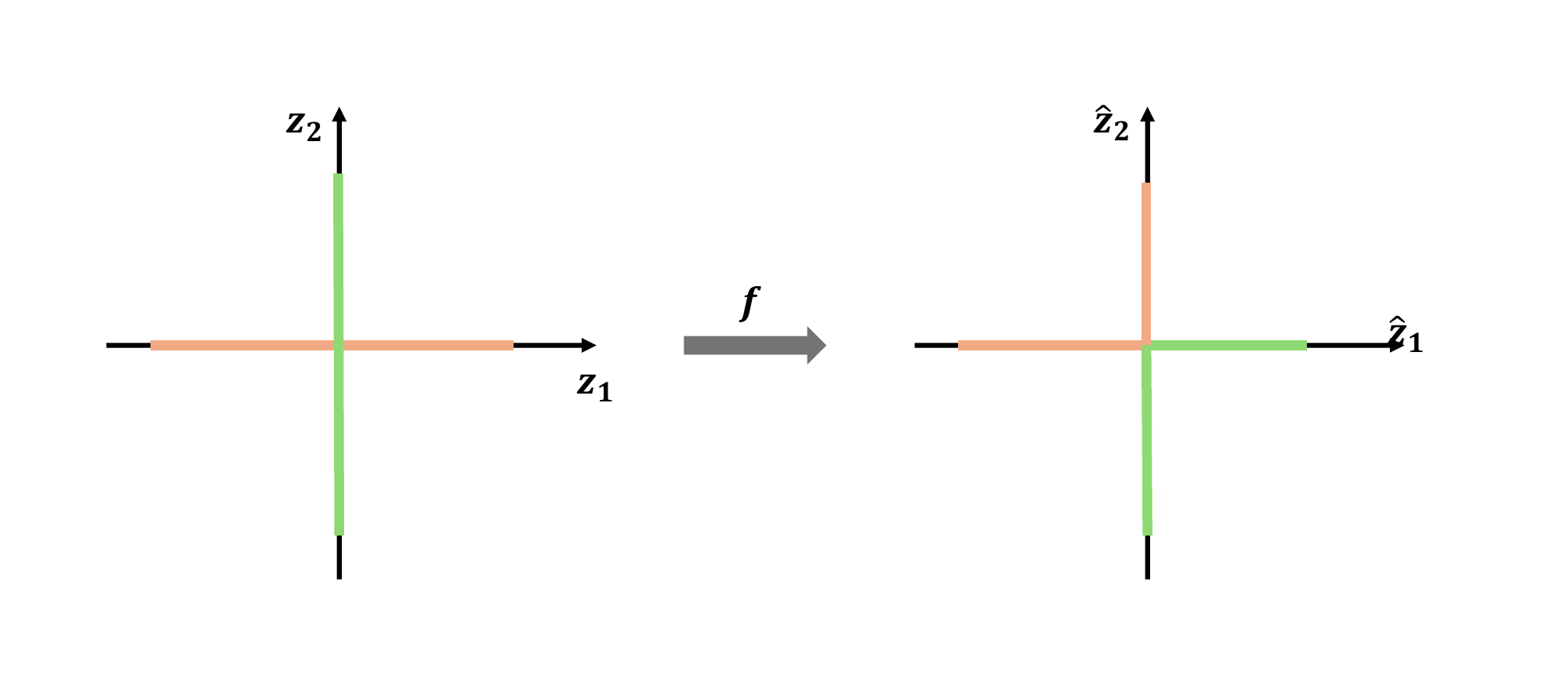}
\caption{Visualization of Example~\ref{exa: violate genericity}: $\fb$ switches half axes of $\zb_1$ and $\zb_2$. The orange area refers to the support of component 1, and the green area refers to the support of component 2 before transformation.}
\label{fig: Visualization_of_exampleB12}   
\end{figure}

To avoid these special cases, we propose Ass.~\ref{ass: no equ norm}, which describes a general class of models.

\assequalnorm*

The idea is: the supports of components with the same rank $k$ may overlap, but there is at least one common point in the intersection where the components can be distinguished, i.e., %
their projections onto the non-degenerate subspace do not have the same norm.

We provide proof and example to illustrate the genericity of Ass.~\ref{ass: no equ norm} in \cref{app: assump}.

The following theorem shows that any continuous invertible piecewise affine function capable of transforming one pdGMM into another pdGMM must be affine on the support of each component in pdGMM.

\begin{restatable}[Identifiability of pdGMMs from open set with common intersection]{theorem}{thmaffineforcomponents}
\label{thm: affine for components}
Consider a pair of finite pdGMMs in $\RR^n$ in reduced form,
\begin{equation}
    \Zb \sim \sum_{j=1}^J \lambda_j N(\mub_j, \Sigmab_j)\quad \text{and}\quad \Zb' \sim \sum_{j=1}^{J'} \lambda_j' N(\mub_j', \Sigmab_j'), 
\end{equation}
where $|\det(\Sigmab_j)|\geq0, \forall j\in[J]$ and $|\det(\Sigmab'_j)|\geq0, \forall j\in[J']$. Suppose $\Zb$ satisfies Assumption~\ref{ass: no equ norm}. 
If there exists a continuous, invertible, piecewise affine function $\fb:\RR^n\rightarrow \RR^n$ satisfying $\Zb'\eqd \fb(\Zb)$, then for all $j_0 \in[J]$, $\fb$ is affine on $\Zcal_{j_0}$, where $\Zcal_{j_0}$ is the support of $N(\mub_{j_0}, \Sigmab_{j_0})$.
\end{restatable}

We introduce some useful theoretical results before we start the proof of \cref{thm: affine for components}. The first lemma shows that the full-ranked matrices sharing the same Gram matrix can be converted to each other via rotations.

\begin{lemma}
\label{lemma: orthogonal matrix}
Consider two matrices $A, B\in \RR^{n\times k}$ with $n\geq k$ and A is full-column rank with $rank(A)=k$. If $AA^T=BB^T$, then there exists an orthogonal matrix $Q\in \RR^{k\times k}$ such that $A=BQ$.
\end{lemma}
\begin{proof}
Since $AA^T=BB^T$ and $A$ is full-column-rank, we can derive $A=BB^TA(A^TA)^{-1}$. Then, we aim to prove $Q:=B^TA(A^TA)^{-1}$ is an orthogonal matrix by showing $Q^TQ=I$, where $I$ is a identity matrix:
\begin{equation}
    Q^TQ=[B^TA(A^TA)^{-1}]^TB^TA(A^TA)^{-1}=(A^TA)^{-1}A^TBB^TA(A^TA)^{-1}=I.
\end{equation}
\end{proof}

In the following lemma, we show that if we can find an open cover, such that two probability measure are the same over any open set in this open cover, then they are the same measure.

\begin{lemma}
    \label{lemma: pasting measures}
    Let $\mu: \Fcal \rightarrow [0,1]$ and $\lambda: \Fcal \rightarrow [0,1]$ be two probability measures on the Borel sigma-algebra $\Fcal$ over $\RR^n$. Let $\Xcal \subseteq \RR^n$ be the support of $\mu$. Let $\Ocal$ be an open cover of $\Xcal$. If for all $O \in \Ocal$, we have
    $$\mu_{O \cap \Xcal} = \lambda_{O \cap \Xcal} \,,$$
    then $\mu = \lambda$.
\end{lemma}
\begin{proof}
    Let $B \in \Fcal$. Let $\Ocal_B$ be a cover of $B$. Since $B$ is a subset of a second-countable space, one can find a countable open subcover $\{O_1, O_2, \dots \}$ of $B$. We construct a partition $B$ as follow: For all $k \in \NN$, define
    $$B_1 := B \cap O_1 \quad\text{and}\quad B_k := B \cap O_k \cap \left( \bigcap_{i=1}^{k-1} B_i^c\right) \, .$$
    We now show that $\{B_1, B_2, \dots\}$ is indeed a partition of $B$. For any $x \in B$, we show that there exists a unique $B_k$ such that $x \in B_k$. Let $k_x := \min \{k \in \NN \mid x \in O_k\}$. Clearly, $x \not\in B_{i}$ for $i < k_x$ since $x \not\in O_i$ for $i < k_x$. Moreover, $x \in B_{k_x}$ since $x \in O_{k_x}$ and $x \in B_i^c$ for all $i < k_x$. It is then clear that $x\not\in B_i$ for $i > k_x$ since $x \not\in B_{k_x}^c$. 

    We write 
    \begin{align}
        \mu(B) &= \mu(B \cap \Xcal) = \mu(\left( \bigcup_{k \geq 1} B_k\right) \cap \Xcal) = \mu( \bigcup_{k \geq 1} (B_k \cap \Xcal)) = \sum_{k=1}^\infty \mu(B_k \cap \Xcal) \\
        &=  \sum_{k=1}^\infty \mu( B \cap O_k \cap \left( \bigcap_{i=1}^{k-1} B_i^c\right) \cap \Xcal) = \sum_{k=1}^\infty \mu_{O_k \cap \Xcal}( B \cap \left( \bigcap_{i=1}^{k-1} B_i^c\right)) \\
        &= \sum_{k=1}^\infty \lambda_{O_k \cap \Xcal}( B \cap \left( \bigcap_{i=1}^{k-1} B_i^c\right)) = \cdots = \lambda(B \cap \Xcal) \, , \\
    \end{align}
    where the $\cdots$ means we apply the exact same steps as before in reverse order.
    By taking $B = \Xcal$, we get $1 = \mu(\Xcal) = \lambda(\Xcal \cap \Xcal) = \lambda(\Xcal)$, which means $\lambda(\Xcal^c) = 0$. This means 
    $$\mu(B) = \lambda(B \cap \Xcal) = \lambda(B \cap \Xcal) + \lambda(B \cap \Xcal^c) = \lambda(B) \,,$$
    which is what we intended to show.
\end{proof}

Next, we report two results from \citet{xu2024sparsity}, which show the identifiability results of potentially degenerate multivariate normal distributions (pdMVNs), as we will use them as important final steps in the proof of \cref{thm: affine for components} later.

\begin{theorem}[Affine Identifiability for pdMVNs with Piecewise Affine $\fb$ (\cite{xu2024sparsity})]
\label{adap_thm:identif-inv-affine}
Assume $\fb, \hat{\fb}:\mathbb{R}^n \rightarrow \mathbb{R}^d$ are injective and piecewise affine. We assume $\Zb$ and $\hat{\Zb}$ follow a (degenerate) multivariate normal distribution.
If $\fb(\Zb) \eqd \hat{\fb}(\hat{\Zb})$, then there exists an invertible affine transformation $\hb:\RR^n \rightarrow \mathbb{R}^n$ such that $\hb(\Zb) \eqd \hat{\Zb}$.  
\end{theorem}

\begin{lemma}[Linearity of $\fb$ for pdMVNs (\cite{xu2024sparsity})]
\label{lemma: linearity of f for degenerate}
Let $\Zb \sim N(\mub, \Sigmab)$, where $|\Sigmab|\geq 0$, i.e. $\Zb$ is potentially a degenerate multivariate normal. Assume that $\fb: \RR^n \rightarrow \RR^n$ is a continuous  piecewise affine function such that $\fb(\Zb) \eqd \Zb$. Then $\fb$ is affine over $\Zcal$, the support of $\Zb$.
\end{lemma}

Now we are ready to prove \cref{thm: affine for components}. %
Below, we outline the structure of the proof. We begin by focusing on a fixed Gaussian component $j_0$ from $\Zb$ and partitioning its support into polyhedral regions such that the function $\fb$ is affine within each subdomain. Within each polyhedral region, by applying \cref{thm:iden-gmm-openset}, we show that there exists a component $j'_0$ in $\Zb'$ such that component $j_0$ is mapped to $j'_0$.
Next, under Ass.~\ref{ass: no equ norm}, we show that $\fb$ consistently assigns $j_0$ to the same $j'_0$ across all polyhedra. This successfully builds a one-to-one correspondence between the Gaussian components on both sides via $\fb$. Finally, using the results from \citet{xu2024sparsity}, we conclude that $\fb$ must be affine on the support of each component.

\thmaffineforcomponents*

\begin{proof}[Proof]
Fix $j_0 \in [J]$. We must show that $\fb$ is affine on $\Zcal_{j_0}$, the support of the $j_0$th component. 
To achieve this, we will first show that,
there exists an $j'_0\in[J']$, such that $$\fb(N(\mub_{j_0},\Sigmab_{j_0}))= N(\mub'_{j'_0},\Sigmab'_{j'_0}).$$ To get this, we will show that for all $\zb_0 \in \Zcal_{j_0}$, there exists an open ball $B$ centered at $\zb_0$ such that $$\left(\fb(N(\mub_{j_0},\Sigmab_{j_0}))\right)_{\fb(B)}= (N(\mub_{j'_0},\Sigmab_{j'_0}))_{\fb(B)}.$$

Let us fix some $\zb_0 \in \Zcal_{j_0}$ and let $\Pi$ be an admissible set of closed polyhedra for $\fb$ (\cref{def:admissible}). The existence of admissible set is ensured by Proposition 22 from \citet{bona2023parameter}. We consider two cases:

\begin{itemize}
    \item Case 1. There exists $P \in \Pi$ such that $\zb_0\in P^{\circ}$, where   $P^{\circ}$ is the interior part of $P$ as defined in Def.~\ref{def:admissible}.
    \item Case 2. For all $P \in \Pi$, $\zb_0 \not\in P^\circ$, where $P^{\circ}$ is  the interior part of $P$. 
\end{itemize} 

\paragraph{Case 1.}
Case 1 is simple. Since $\zb_0$ is in the interior of $P$, we can find an open set with interior point $\zb_0$, $B$, such that $B \subset P$. By definition of an admissible cover, $\fb$ is affine on $P$, and thus on $B$, and thus on $B \cap \Zcal_j$. Since $\fb$ is injective, we know there exists an invertible matrix $H$ and weight vector $\bb$ such that $\fb(\zb)=H\zb+\bb$, $\forall \zb\in B$. Define $\phi(\zb) := H\zb+\bb$ for all $\zb \in \RR^n$. We can thus constrain the probability measure to $\fb(B)$ on both sides of $\fb(\Zb)\eqd \Zb'$ and write 
\begin{align}
&\sum_{j\in J} \lambda_j\fb\left(N(\mub_j, \Sigmab_j)_{B} \right) =  \sum_{j=1}^{J'} \lambda_j' N(\mub_j', \Sigmab_j')_{\fb(B)} \, ,\\
&\sum_{j\in J} \lambda_j\phi\left(N(\mub_j, \Sigmab_j)_{B} \right) =  \sum_{j=1}^{J'} \lambda_j' N(\mub_j', \Sigmab_j')_{\fb(B)} \, .
\end{align}
By Lemma~\ref{lemma: pushforward} and $\fb$ is injective, we have 
\begin{align}
\sum_{j\in J} \lambda_j\left(N(H\mub_j+\bb, H\Sigmab_jH^T) \right)_{\fb(B)} =  \sum_{j=1}^{J'} \lambda_j' N(\mub_j', \Sigmab_j')_{\fb(B)} \, .
\end{align}
Since $B$ is open set and $\fb$ is an open map on $B$ (all surjective linear maps are), $\fb(B)$ is an open set as well. Thus we can apply \cref{thm:iden-pdgmm-openset} and get there must exists an $j'_0\in[J']$, such that $\fb(N(\mub_{j_0},\Sigmab_{j_0}))= N(\mub'_{j'_0},\Sigmab'_{j'_0})$ on $\fb(B)$.

\paragraph{Case 2.}
Case 2 is much more technical and will thus be the focus of the rest of the proof. Since $\Pi$ covers $\RR^n$, there must be at least one $P_1 \in \Pi$ such that $\zb_0 \in P_1$. Since $\zb_0$ is not in the interior of $P_1$, it must be in its boundary, $\partial P_1$. We now argue that $\zb_0$ must also be in a distinct polyhedron. We argue by contradiction. Suppose that is false. This means $\zb_0$ is in the complement of all polyhedra $P \in \Pi \setminus \{P_1\}$. Since these polyhedra are closed, their complements are open, thus we can find an open set $U$ containing $\zb_0$ that do not intersect with any $P \in \Pi \setminus \{P_1\}$. But since $\zb_0$ is on the boundary of $P_1$, this means a non-empty portion of $U$ lies outside $P_1$. But this would mean this portion is outside of all $P \in \Pi$, which is a contradiction since $\Pi$ covers $\RR^n$. This means there must be at least one other polyhedron $P_{2} \not= P_1$ containing $\zb_0$. Again, since $\zb_0$ is not in the interior of $P_2$, it must be on its boundary, $\partial P_2$. 

To summarize, we found two distinct polyhedra, $P_1$ and $P_2$, such that $\zb_0 \in P_1 \cap P_2$ and such that $\zb_0 \in \partial P_1$ and $\zb_0 \in \partial P_2$. In general, $\zb_0$ could be in more than two polyhedra. Let $P_1, \dots, P_L \in \Pi$ be the polyhedra that contain $\zb_0$. We know from the previous steps that $L \geq 2$.

Define the set of all components that intersects the point $\zb_0$.
\begin{align}
    J_{\zb_0}:=\{j\in [J]:\zb_0\in \Zcal_j\},\quad 
\end{align}
Of course, $j_0 \in J_{\zb_0}$. Since the $\Zcal_j$ are affine subspaces, they are closed sets. This means we can find an open set $B'$ containing $\zb_0$ small enough such that $B'$ intersects only the components $j \in J_{\zb_0}$. Furthermore, since all $P \in \Pi$ are closed, we can choose an open set $B''$ containing $\zb_0$ that is small enough so has to not intersect with the polyhedra other than $P_1, \dots, P_L$. Take $B_{\zb_0} := B' \cap B''$. Since $\zb_0$ is on the boundaries of each $P_1, \dots, P_L$, any open set must intersect with each $P_l$ for $l \in [L]$. We can thus define the non-empty sets 
$S_l:=B_{\zb_0}\cap P_l$ for all $l \in [L]$. Define the index set 
\begin{align}
\label{equ: s0}
J_l:=\{j\in [J]: S_l\cap \Zcal_j\neq \emptyset\}\, ,
\end{align}
i.e. this is the set of components $\Zcal_j$ that intersects the polyhedron $P_l$ and the open set $B_{\zb_0}$. 
Notice that $J_l \subseteq J_{\zb_0}$.

Then our goal for the following part is to show that there exists an $j'_0\in[J']$, such that $\fb(N(\mub_{j_0},\Sigmab_{j_0}))= N(\mub'_{j'_0},\Sigmab'_{j'_0})$ on $S_l$. 

But we do not need to go through all $l\in[L]$. Instead, since we are focusing on component $j_0$ now, in Case 2, we first explore the one on which $N(\mub_{j_0},\Sigmab_{j_0})$ has positive probability, and later come back to discuss the rest of the polyhedra. For this reason, we introduce one more index notation:
\begin{align}
\label{equ: L_j_0}
   L_{j_0}:=\{l\in [L]: N(\mub_{j_0},\Sigmab_{j_0}) (S_l)>0)\}\, ,
\end{align}
where $N(\mub_{j_0},\Sigmab_{j_0}) (S_l)$ means the probability that measure $N(\mub_{j_0},\Sigmab_{j_0})$ take on $S_l$. 

Pick some polyhedron $l \in [L_{j_0}]$. By condition we have $\Zb'\eqd \fb(\Zb)$, we restrict the measure on both sides to $\fb(S_l)$ and then we get 
\begin{align}
&\left(\fb(\sum_{j=1}^J \lambda_j N(\mub_j, \Sigmab_j))\right)_{\fb(S_l)}=  \left(\sum_{j=1}^{J'} \lambda_j' N(\mub_j', \Sigmab_j')\right)_{\fb(S_l)}
\end{align}
Since $\fb$ is injective, we can use Lemma~\ref{lemma: pushforward} to switch constrain and $\fb$ on the left hand side and we get %
\begin{align}
\fb\left(\left(\sum_{j=1}^J \lambda_j N(\mub_j, \Sigmab_j)\right)_{S_l}\right)&=  \left(\sum_{j=1}^{J'} \lambda_j' N(\mub_j', \Sigmab_j')\right)_{\fb(S_l)},\\
\fb\left(\sum_{j=1}^J \lambda_j N(\mub_j, \Sigmab_j)_{S_l}\right)&=  \sum_{j=1}^{J'} \lambda_j' N(\mub_j', \Sigmab_j')_{\fb(S_l)},\\
\sum_{j=1}^J \lambda_j\fb\left(N(\mub_j, \Sigmab_j)_{S_l} \right) &=  \sum_{j=1}^{J'} \lambda_j' N(\mub_j', \Sigmab_j')_{\fb(S_l)} \, ,
\end{align}
where the second line uses the fact that the restriction of a mixture is the mixture of the restriction and the third line uses the fact that the pushforward of a mixture is the mixture of the pushforwards. %

Notice that if $\Zcal_j \cap S_l = \emptyset$, we have that $N(\mub_j, \Sigmab_j)_{S_l} = 0$, i.e. it is the zero measure. The pushforward of the zero measure is also the zero measure, thus 
\begin{align}
    \sum_{j \in J_l} \lambda_j\fb\left(N(\mub_j, \Sigmab_j)_{S_l} \right) =  \sum_{j=1}^{J'} \lambda_j' N(\mub_j', \Sigmab_j')_{\fb(S_l)}
\end{align}

Since $\fb$ is piecewise affine and injective, we know there exists an invertible matrix $H^l$ and weight vector $\bb^l$ such that $\fb(\zb)=H^l\zb+\bb^l$, $\forall \zb\in S_l$. Define $\phi^l(\zb) := H^l\zb+\bb^l$ for all $\zb \in \RR^n$. We can thus write 
\begin{align}
&\sum_{j\in J_l} \lambda_j\phi^l\left(N(\mub_j, \Sigmab_j)_{S_l} \right) =  \sum_{j=1}^{J'} \lambda_j' N(\mub_j', \Sigmab_j')_{\fb(S_l)} \, ,
\end{align}

since $\fb = \phi^l$ on $S_l$. Since $\phi^l$ is injective, we can apply Lemma~\ref{lemma: pushforward} to get
\begin{align}
\sum_{j\in J_l} \lambda_j\phi^l\left(N(\mub_j, \Sigmab_j)\right)_{\phi^l(S_l)} &=  \sum_{j=1}^{J'} \lambda_j' N(\mub_j', \Sigmab_j')_{\fb(S_l)} \\
\sum_{j\in J_l} \lambda_jN(H^l\mub_j+\bb^l, H^l\Sigmab_j(H^l)^T)_{\fb(S_l)} &=  \sum_{j=1}^{J'} \lambda_j' N(\mub_j', \Sigmab_j')_{\fb(S_l)} \,.
\end{align}

Define $J'_l:=\{j\in [J']:\Zcal'_j\cap \fb(S_l)\neq \emptyset\}$. This allows us to write 
\begin{align}
\label{equ: constrain on S_l}
&\sum_{j\in J_l} \lambda_j N(H^l\mub_j+\bb^l, H^l\Sigmab_j(H^l)^T)_{\fb(S_l)}
=  \sum_{j\in J'_l} \lambda_j' N(\mub_j', \Sigmab_j')_{\fb(S_l)}.
\end{align}
Now we discuss the possible relations between $N(\mub_{j_0},\Sigmab_{j_0})$ and $S_l$. If $N(\mub_{j_0},\Sigmab_{j_0})$ takes probability zero on $S_l$, then we can skip it. If not, then we face two cases:
\begin{itemize}
\item Case 2.1 $\Zcal_{j_0}$ intersects with the interior of $S_l$, i.e. $\Zcal_{j_0}\cap S_l^{\circ}\neq \emptyset$, where $S_l^{\circ}:=B_{\zb_0}\cap P^{\circ}_l$.
\item Case 2.2 $\Zcal_{j_0}$ only intersects with the boundary of $S_l$ with positive probability, i.e. $\Zcal_{j_0}\cap S_l^{\circ}= \emptyset$ and $N(\mub_{j_0},\Sigmab_{j_0})( \partial S_l)>0$, where $\partial S_l:=B_{\zb_0}\cap \partial P_l$.
\end{itemize}

\textbf{Case 2.1.}
For Case 2.1, define the index set that also intersects with the interior part of $S_l$: $$J_l^{\circ}:=\{j\in J_l: \Zcal_j\cap S_l^{\circ} \neq \emptyset\}, \quad J_l'^{\circ}:=\{j\in J'_l: \Zcal'_j\cap \fb(S_l^{\circ}) \neq \emptyset\}.$$ 
We further constrain \cref{equ: constrain on S_l} on $\fb(S_l^{\circ})$ and cancel out the purely zero measure terms on both sides, i.e. the terms are not in $J_l^{\circ}$ and $J_l'^{\circ}$

\begin{align}
&\sum_{j\in J_l^{\circ}} \lambda_j N(H^l\mub_j+\bb^l, H^l\Sigmab_j(H^l)^T)_{\fb(S_l^{\circ})}
=  \sum_{j\in J_l'^{\circ}} \lambda_j' N(\mub_j', \Sigmab_j')_{\fb(S_l^{\circ})}.
\end{align}

$\fb$ is an open map on $S_l^{\circ}$ since it is affine on $S_l^{\circ}$. Therefore, $\fb(S_l^{\circ})$ is an open set as well. 

By the definition of $J_l^{\circ}$, we have $\forall j\in J_l^{\circ}$, $\Zcal_j\cap S_l^{\circ}\neq \emptyset$, then we can derive $\fb(\Zcal_j\cap S_l^{\circ})\neq \emptyset, \forall j\in J_l^{\circ}$. Since $\fb$ is invertible function, we can further derive $\forall j\in J_l, \fb(\Zcal_j)\cap \fb(S_l)\neq \emptyset$. Therefore, we can use \cref{thm:iden-pdgmm-openset} and get 
\begin{align}
\label{equ: interior eqd}
&\sum_{j\in J_l^{\circ}} \lambda_j N(H^l\mub_j+\bb^l, H^l\Sigmab_j(H^l)^T)
=  \sum_{j\in J_l'^{\circ}} \lambda_j' N(\mub_j', \Sigmab_j'),
\end{align}
which means $|J_l^{\circ}|=|J_l'^{\circ}|$ and there exists a permutation $\pi^l: J_l^{\circ}\rightarrow J_l'^{\circ}$ such that 
\begin{align}
\label{equ: cov and u in P_l}
H^l\Sigmab_j(H^l)^T=\Sigmab_{\pi^l(j)}', \quad H^l\mub_j+\bb^l=\mub_{\pi^l(j)}'.      
\end{align}

\textbf{Case 2.2.}
For Case 2.2, 
because of the definition of $L_{j_0}$ (\cref{equ: L_j_0}) where we pick up the $S_l$, we know $N(\mub_{j_0},\Sigmab_{j_0})$ has positive probability on $\partial S_l$. Since we know $\Zcal_{j_0}$ is an affine space, it can only intersect with one facet of $\partial S_l$. We denote this facet $\partial S_l^{j_0}$. 
Define the index sets that the supports intersect with this facet with positive probability: 
$$J_l^{j_0}:=\{j\in J_l\backslash J_l^{\circ}: N(\mub_{j},\Sigmab_{j})  (\partial S_l^{j_0}) >0\}, \quad J_l'^{j_0}:=\{j\in J'_l\backslash J_l'^{\circ}: N(\mub'_{j},\Sigmab'_{j}) (\fb(\partial S_l^{j_0})) >0\}.$$ 

Constrain \cref{equ: constrain on S_l} further on $\fb(\partial S_l^{j_0})$ and eliminate the probability zero terms then we get 
\begin{align}
\sum_{j\in J_l^{j_0}} \lambda_j N(H^l\mub_j+\bb^l, H^l\Sigmab_j(H^l)^T)_{\fb(\partial S_l^{j_0})}
=  \sum_{j\in J_l'^{j_0} } \lambda_j' N(\mub_j', \Sigmab_j')_{\fb(\partial S_l^{j_0})}.
\end{align}

Since all the Gaussian components in $J_l^{j_0}$ take a positive probability on $\partial S_l^{j_0}$, which implies their rank is lower than or equal to the rank of $\partial S_l^{j_0}$ and their support is parallel with $\partial S_l^{j_0}$. Then it must have no intersection with the interior of all adjacent polyhedra. Thus, we can construct an open set $E\subset \RR^n$, such that $\partial S_l^{j_0}\subset E$. And expand the constrained measure from $\fb(\partial S_l^{j_0})$ to $\fb(E)$ and then we have
\begin{align}
& \sum_{j\in J_l^{j_0}} \lambda_j N(H^l\mub_j+\bb^l, H^l\Sigmab_j(H^l)^T)_{\fb(E)}
=  \sum_{j\in J_l'^{j_0}}  \lambda_j' N(\mub_j', \Sigmab_j')_{\fb(E)}.
\end{align}
$\fb$ is an open map on $E$. Therefore, $\fb(E)$ is also an open set.
Therefore, we can use \cref{thm:iden-pdgmm-openset} and get 
\begin{align}
&\sum_{j\in J_l^{j_0}} \lambda_j N(H^l\mub_j+\bb^l, H^l\Sigmab_j(H^l)^T)
=  \sum_{j\in J_l'^{j_0}} \lambda_j' N(\mub_j', \Sigmab_j'),
\end{align}
which means $|J_l^{j_0}|=|J_l'^{j_0}|$ and there exists a permutation $\pi^l: J_l^{j_0}\rightarrow J_l'^{j_0}$ such that 
\begin{align}
\label{equ: cov and u in P_l 2}
H^l\Sigmab_j(H^l)^T=\Sigmab_{\pi^l(j)}', \quad H^l\mub_j+\bb^l=\mub_{\pi^l(j)}'.  
\end{align}

To sum up, \cref{equ: cov and u in P_l} holds for all $S_l$ in Case 2.1, and \cref{equ: cov and u in P_l 2} holds for all $S_l$ in Case 2.2, which covers all $\Zcal_{j_0}\cap B_{\zb_0}$. Thus, we can summarize it as for all $l\in [L_{j_0}]$ %
\begin{align}
H^l\Sigmab_j(H^l)^T=\Sigmab_{\pi^l(j)}', \quad H^l\mub_j+\bb^l=\mub_{\pi^l(j)}'.      
\end{align}
In the following part, we focus on the points $\zb_0$ satisfying Ass.~\ref{ass: no equ norm}, denote this set as $\Scal_{\Zcal}$. We aim to show that, consider $\zb_0\in \Scal_{\Zcal}$, for all $l\in [L_{j_0}]$%
, permutations $\pi^l$ map $j_0$ to the same index in $\Zcal'$.

We prove the claim for all $\zb_0\in \Scal_{\Zcal}$ by contradiction: suppose there exist $l_1, l_2\in [L_{j_0}]$, such that $\pi^{l1}(j_0)\neq \pi^{l2}(j_0)$. Denote $\pi^{l1}(j_0):=j_0'$, this means, there exists a polyhedron $l*\in [L_{j_0}]$, such that $(\pi^{l*})^{-1}(j_0')\neq (\pi^{l_1})^{-1}(j_0')$. Denote $j* :=(\pi^{l*})^{-1}(j_0')$, we have $j*\neq j_0$.

We can derive this 
\begin{align}
&
\begin{cases}
H^{l_1}\Sigmab_{j_0}(H^{l_1})^T=\Sigmab_{j_0'}', \quad H^{l_1}\mub_{j_0}+\bb^{l_1}=\mub_{j_0'}',\\
H^{l*}\Sigmab_{j*}(H^{l*})^T=\Sigmab_{j_0'}', \quad H^{l*}\mub_{j*}+\bb^{l*}=\mub_{j_0'}'.
\end{cases}\\
\label{equ: same mean0}
\Rightarrow &
\begin{cases}
H^{l_1}\Sigmab_{j_0}(H^{l_1})^T=H^{l*}\Sigmab_{j*}(H^{l*})^T,\\
H^{l_1}\mub_{j_0}+\bb^{l_1}=H^{l*}\mub_{j*}+\bb^{l*}.\\
\end{cases}
\end{align}

Since $\zb_0$ is in both polyhedra $l_1$ and $l*$,  we get
\begin{align}
H^{l_1}\zb_0 + \bb^{l_1} &= \fb(\zb_0) = H^{l*}\zb_0 + \bb^{l*} 
\label{eq:jfjdnsjsk}
\end{align}
By taking the difference between \cref{eq:jfjdnsjsk} and \cref{equ: same mean0}, we have
\begin{equation}
\label{equ: conti0}
    H^{l_1}(\zb_0-\mub_{j_0})=H^{l*}(\zb_0-\mub_{j*}).
\end{equation}

Suppose $rank(\Sigmab'_{j_0'})=k$. By \cref{equ: same mean0}, since both $H^{l_1}$ and $H^{l_2}$ are invertible matrix, we have 
\begin{equation}
\label{equ: same rank0}
rank(\Sigmab'_{j_0'})=rank(\Sigmab_{j_0})=rank(\Sigmab_{j*})=k.
\end{equation}
By the definition of degenerate Gaussian random variable, we can find a $n\times k$ matrix $A_1$ where $A_1 A_1^T=\Sigmab_{j_0}$ and $rank(A_1)=k$, and have 
$$A_1 \vepsb+\mub_{j_0}\sim N(\mub_{j_0}, \Sigmab_{j_0}),$$
where $\vepsb$ follows a standard $k$-dimensional normal distribution. Hence, there exists a unique value $\vepsb_1\in \RR^k$, such that $\zb_0=A_1\vepsb_1+\mub_{j_0}$. The same analysis for $N(\mub_{j*}, \Sigmab_{j*})$, there exists a unique $\vepsb_2\in \RR^k$ such that $\zb_0=A_2\vepsb_2+\mub_{j*}$.  %

Substitute into \cref{equ: conti0}, we get%
\begin{equation}
\label{equ: conti00}
H^{l_1}A_{1}\vepsb_{1}=H^{l_2}A_{2}\vepsb_{2}.
\end{equation}

Additionally, we have $A_1 A_1^T=\Sigmab_{j_0}$ and $A_2 A_2^T=\Sigmab_{j*}$. Substitute into \cref{equ: same mean0}, we get $$H^{l_1}A_{1} (H^{l_1}A_{1})^T=H^{l_2}A_{2} (H^{l_2})A_{2})^T.$$
First we have $rank(H^{l_1}A_{1})=k$ as $H^{l_1}$ is invertible. 

By Lemma~\ref{lemma: orthogonal matrix}, this implies there exists an orthogonal matrix $Q\in \RR^{k\times k}$, such that $H^{l_1}A_{1}=H^{l_2}A_{2}Q$. Substitute into \cref{equ: conti00}, we have 
$$H^{l_1}A_{1}(\vepsb_{1}-Q^T\vepsb_{2})=0.$$
Then we can derive $\vepsb_{1}-Q^T\vepsb_{2}=0$, which implies $||\vepsb_1||=||\vepsb_{2}||$.Note that $||\vepsb_1||=\sqrt{(\zb_0-\mub_{j_0})^T\Sigmab_{j_0}^{-1}(\zb_0-\mub_{j_0})}$ and $||\vepsb_{2}||=\sqrt{(\zb_0-\mub_{j_*})^T\Sigmab_{j_*}^{-1}(\zb_0-\mub_{j_*})}$.
This means for all $\zb_0\in \Scal_{\Zcal}$, the Mahalanobis distance of $\zb_0$ to $N(\mub_{j_0},\Sigmab_{j_0})$ and $N(\mub_{j_*},\Sigmab_{j_*})$ are the same.

We can find it will be contradictory to Assumption \ref{ass: no equ norm} if $\pi^{l_1}({j_0})\neq \pi^{l_2}({j_0})$. Therefore, we can claim there exists $j'_0\in[J']$, such that $\pi^{l}({j_0})=j'_0$ for all $l\in [L_{j_0}]$. %
Due to \cref{equ: cov and u in P_l} and the identifiability of pdMVN, we have 
\begin{align}
\label{equ: affine z_i z_j'}
N(H^l\mub_{j_0}+\bb^l, H^l\Sigmab_{j_0}(H^l)^T)
=   N(\mub_{j_0'}', \Sigmab_{j_0'}'), \quad \forall l\in [L_{j_0}].
\end{align}

Now we finished all the discussion of $l\in [L_{j_0}]$, i.e. all the $S_l$ where $N(\mub_{j_0},\Sigmab_{j_0})$ can take positive probability. However, as our goal is to show $\fb(N(\mub_{j_0},\Sigmab_{j_0}))= N(\mub'_{j'_0},\Sigmab'_{j'_0})$ over $\Zcal_{j_0}\cap B$, there is a probability-zero subdomain missing in the above analysis. Next, we discuss the rest part $l\in [L]\backslash [L_{j_0}]$. These cases can only happen if $\Zcal_{j_0}$ intersects with the boundary of $S_l$ only, i.e. $\Zcal_{j_0}\cap S_l^{\circ}= \emptyset$, $\Zcal_{j_0} \cap \partial S_l\neq \emptyset$ and $N(\mub_{j_0},\Sigmab_{j_0})(\partial S_l)=0$. For this case, since $\Zcal_{j_0} \cap \partial S_l\neq \emptyset$ and $N(\mub_{j_0},\Sigmab_{j_0})(\partial S_l)=0$, we know for every point $\zb*\in \Zcal_{j_0} \cap \partial S_l$, there must exist an adjecent $S_{l*}$ such that $\zb*\in S_{l*}$ and $l*\in [L_{j_0}]$. Therefore, we can conclude that $\Zcal_{j_0}\cap B \subseteq \cup _{l\in[L_{j_0}]}S_l$.

We now define a partition of the set $B_{\zb_0}\cap \Zcal_{j_0}$ by $S_l$, $l\in [L_{j_0}]$.

For any non-empty subset $I\subseteq [L_{j_0}]$, we define
$$D_I:=\left(\cap_{{l'}\in I} S_{l'}\right)\backslash \left(\cup_{{l'}\not\in I} S_{l'}\right).$$
Note that all the affine functions $H^l\zb + \bb^l$ with $l \in I$ agree on $D_I$ since $D_I \subseteq \cap_{l \in I} P_l$. For simplicity, we can directly use $l_I:=\min I$, i.e. the smallest index in $I$. When $l=l_I$ we constrain the measures on both sides of \cref{equ: affine z_i z_j'} on subdomain $\fb(D_I)$ and get 
\begin{align}
\label{equ: constrained D_I}
&N(H^{l_I}\mub_{j_0}+\bb^{l_I}, H^{l_I}\Sigmab_{j_0}(H^{l_I})^T)_{\fb(D_I)}
=   N(\mub_{j_0'}', \Sigmab_{j_0'}')_{\fb(D_I)}
\end{align}

Sum up all the $I\subseteq [L_{j_0}]$, we have 
\begin{align}
&\sum_{I\subseteq [L_{j_0}]} N(H^{l_I}\mub_{j_0}+\bb^{l_I}, H^{l_I}\Sigmab_{j_0}(H^{l_I})^T)_{\fb(D_I)}
=  \sum_{I\subseteq [L_{j_0}]} N(\mub_{j_0'}', \Sigmab_{j_0'}')_{\fb(D_I)},\\
&\sum_{I\subseteq [L_{j_0}]} \fb\left(N(\mub_{j_0}, \Sigmab_{j_0})_{D_I}\right)
=  \sum_{I\subseteq [L_{j_0}]} N(\mub_{j_0'}', \Sigmab_{j_0'}')_{\fb(D_I)}.
\end{align}
Note that $D_I$'s are disjoint from each other. Thus, we can switch sum into union and get
\begin{align}
&\fb\left(N(\mub_{j_0}, \Sigmab_{j_0})_{\cup_{I\subseteq [L_{j_0}]} D_I}\right)
=  N(\mub_{j_0'}', \Sigmab_{j_0'}')_{\fb(\cup_{I\subseteq [L_{j_0}]}D_I)}.
\label{equ: sumup boundary0}
\end{align}
Since $D_I$ is partition set of $B_{\zb_0}\cap \Zcal_{j_0}$, which means $\cup_{I\subseteq [L_{j_0}]} D_I=B_{\zb_0}\cap \Zcal_{j_0}$, substitute into \cref{equ: sumup boundary0} and then we have

\begin{align}
&\fb\left(N(\mub_{j_0}, \Sigmab_{j_0})_{B_{\zb_0}\cap \Zcal_{j_0}}\right)
=  N(\mub_{j_0'}', \Sigmab_{j_0'}')_{\fb(B_{\zb_0}\cap \Zcal_{j_0})},
\label{equ: sumup boundary}\\
&N(\mub_{j_0}, \Sigmab_{j_0})_{B_{\zb_0}\cap \Zcal_{j_0}}
=  \fb^{-1}\left(N(\mub_{j_0'}', \Sigmab_{j_0'}')_{\fb(B_{\zb_0}\cap \Zcal_{j_0})}\right).
\end{align}
Since $\fb$ is bijective on $\RR^n$, we have $\fb^{-1}$ is injective. Thus by apply Lemma~\ref{lemma: pushforward}, we can conclude Case 2:
\begin{align}
&N(\mub_{j_0}, \Sigmab_{j_0})_{B_{\zb_0}\cap \Zcal_{j_0}}
=  \fb^{-1}\left(N(\mub_{j_0'}', \Sigmab_{j_0'}')\right)_{B_{\zb_0}\cap \Zcal_{j_0}}.
\label{equ: z_i to z_j'}
\end{align}

Note that \cref{equ: z_i to z_j'} is true for all $\zb_0\in \Scal_{\Zcal}$, i.e. we can construct an open set $B_{\zb_0}$ containing $\zb_0$ for all $\zb_0\in \Scal_{\Zcal}$. Notice that $\{B_{\zb_0} \mid \zb_0 \in \in \Scal_{\Zcal}\}$ is thus an open cover of the support of the $\Zcal_{j_0}$. We can thus apply Lemma~\ref{lemma: pasting measures} to get 
\begin{align}
N(\mub_{j_0}, \Sigmab_{j_0})
&= \fb^{-1}\left( N(\mub_{j_0'}', \Sigmab_{j_0'}')\right) \\
\fb\left(N(\mub_{j_0},\Sigmab_{j_0})\right)
&= N(\mub'_{j'_0},\Sigmab'_{j'_0}).
\label{equ: f z_i to z_j}
\end{align}
Then, since $\fb$ is injective and
piecewise affine, by \cref{adap_thm:identif-inv-affine} from \cite{xu2024sparsity}, there
exists an invertible affine transformation
$\hb_{j_0}:\RR^n \rightarrow \RR^n$ such that $\hb_{j_0}(N(\mub_{j_0},\Sigmab_{j_0}))=N(\mub'_{j'_0},\Sigmab'_{j'_0})$. Substitute into \cref{equ: f z_i to z_j} then we have 
\begin{align}
\fb\left(N(\mub_{j_0},\Sigmab_{j_0})\right)
=  \hb_{j_0}(N(\mub_{j_0},\Sigmab_{j_0})) \quad \Rightarrow \quad  \hb_{j_0}^{-1}\circ \fb\left(N(\mub_{j_0},\Sigmab_{j_0})\right)
=  N(\mub_{j_0},\Sigmab_{j_0}).
\end{align}
As $\hb_{j_0}^{-1}\circ \fb$ is a continuous piecewise affine function, by Lemma~\ref{lemma: linearity of f for degenerate} from \cite{xu2024sparsity}, we can conclude that $\hb_{j_0}^{-1}\circ \fb$ is affine on $\Zcal_{j_0}$. Furthermore, we can conclude that $\fb$ is affine on $\Zcal_{j_0}$. This result holds for all ${j_0}\in[J]$.
\end{proof}

\thmaffineidenticomp*

\begin{proof}
Since $\Xb = \fb(\Zb)$, we can rewrite Equation \eqref{eq:zero_reconstruction} (perfect reconstruction) as
    \begin{align}
        \mathbb{E}||\fb(\Zb) - \hat\fb(\gb(\fb(\Zb)))||^2_2 = 0\,.
    \end{align}

This means that $\fb(\Zb)$ and $\hat\fb(\gb(\fb(\Zb)))$ are equal $\mathbb{P}_{\Zb}$-almost everywhere, which implies $\fb(\Zb)$ and $\hat\fb(\gb(\Xb)))$ are equally distributed with the same image space. Therefore, we can define $\hb := \hat\fb^{-1}\circ \fb$ on $\Zcal$, which is a continuous and invertible piecewise affine function, and we have $\Zb\eqd \hb^{-1}(\gb(\Xb))$, where both $\Zb$ and $\gb(\Xb)$ are pdGMMs.

Since $\Zb$ satisfies Ass.~\ref{ass: no equ norm}, and $\hb^{-1}$ is a continuous and invertible piecewise affine function, we can apply \cref{thm: affine for components} and get $\hb^{-1}$ is affine on $\Zcal'_j$, where $\Zcal'_j$ is the support of $N(\mub'_j, \Sigmab'_j)$. One step more, we can conclude that $\hb$, i.e., $\hat\fb^{-1}\circ \fb$, is affine on $\Zcal_j$.

\end{proof}

\subsubsection{From affine on components to affine everywhere}
\label{app: global affine}
While \cref{thm:affine-identi-comp} %
proves that $\gb\circ \fb$ 
is affine within each individual component, it does not guarantee that there is a global affine function across all components. In this section, we seek conditions under which $\gb\circ \fb$ is globally affine, i.e., affine on the whole support for the pdGMM. In the following proposition (Prop.~\ref{prop:globally_affine}), %
we show that if the supports of all components share a common basis (i.e., each support can be spanned by a subset of this common basis), then it is possible to construct a consistent affine function across the entire domain.

\asscommonbasis*

\begin{proposition}\label{prop:globally_affine}
Let $\Zcal_1, \dots, \Zcal_J \subseteq \RR^n$ be affine spaces that satisfy Ass.~\ref{ass: common basis}, and let $f:\RR^n \rightarrow \RR$ be a function such that, for all $j \in [J]$, the restriction of $f$ to $\Zcal_j$ is affine. Then the restriction of $f$ to $\Zcal := \cup_{i \in [m]}\Zcal_i$ is affine.%
\end{proposition}
\begin{proof}
    Our goal is to construct an affine function $\ab\zb + b$ with $\ab \in \RR^{1\times n}$ and $b \in \RR$ such that this function agrees with $f$ on $\Zcal$.

    First notice that, since $f$ is affine on each $\Zcal_j$, we have: for each $j \in [J]$, there exists $\ab^j \in \RR^{1 \times n}$ and $b^j \in \RR$ such that, for all $\zb \in \Zcal_j$, $f(\zb) = \ab^j\zb + b^j$. 
    
    Since $\zb^0 \in \cap_{j \in [J]} \Zcal_j$, we have that, for all $j \in [J]$, $f(\zb^0) = \ab^j \zb^0 + b^j$ for all $j \in [J]$.
    
    Let us define $\Bb := [\zb^1\ \cdots\ \zb^n] \in \RR^{n \times n}$, which is invertible since its columns form a basis by assumption. 

    Fix $j \in [J]$ and $k \in \Kcal_j$. We know that 
    \begin{align}
        f(\zb^k+\zb^0) &= \ab^j(\zb^k+\zb^0) + b^j \\
        &= \ab^j\zb^k + \ab^j\zb^0 + b^j\\
        &= \ab^j\zb^k + f(\zb^0)\\
        &= \underbrace{\ab^j\Bb}_{\tilde\ab^j :=}\eb^k + f(\zb^0) \\    
        &= \tilde\ab^j\eb^k + f(\zb^0) \\
        &= \tilde\ab^j_k + f(\zb^0) \\
        \implies \tilde\ab^j_k &= f(\zb^k+\zb^0) - f(\zb^0) \quad \forall k \in \Kcal_j \,.
    \end{align}
    
    Let us define 
    \begin{align}
        \ab &:= \begin{bmatrix}
            \left(f(\zb^1+\zb^0) - f(\zb^0)\right) & \cdots & \left(f(\zb^n+\zb^0) - f(\zb^0)\right) 
        \end{bmatrix} \Bb^{-1} \in \RR^{1 \times n} \text{and}\\
        b &:= f(\zb^0)-\ab\zb^0 \,.
    \end{align}

    We must now show that, for all $\zb \in \cup_{j \in [J]} \Zcal_j$, $f(\zb) = \ab \zb + b$. Fix some arbitrary $\zb \in \cup_{j \in [J]} \Zcal_j$. We know there has to be one $j_0 \in [J]$ s.t. $\zb \in \Zcal_{j_0}$. Hence $f(\zb) = \ab^{j_0} \zb + b^{j_0}$. We showed earlier that $b^j = f(\zb^0)-\ab^j\zb^0$ for all $j \in [J]$, hence we have $f(\zb) = \ab^{j_0} \zb + f(\zb^0)-\ab^{j_0}\zb^0$. Since $\Bb$ is invertible, there exists $\tilde\zb \in \RR^n$ such that $\zb-\zb^0 = B\tilde\zb$. Since $\{\zb^k \mid k \in \Kcal_j\}$ forms a basis for $\Zcal_j-\zb^0$, we have that whenever $k \not\in \Kcal_j$, $\tilde \zb_k = 0$. We can thus write
    \begin{align}
        f(\zb) &= \ab^{j_0} \zb + b^{j_0} \\
        &= \ab^{j_0} (B\tilde\zb + \zb^0) + b^{j_0} \\
        &= \ab^{j_0}B\tilde\zb + \ab^{j_0}\zb^0 + b^{j_0}\\
        &= \tilde\ab^{j_0}\tilde\zb + f(\zb^0)\\
        &= \sum_{k \in \Kcal_j}\tilde\ab_k^{j_0} \tilde \zb_k + \sum_{k \not\in \Kcal_j}\tilde\ab_k^{j_0} \underbrace{\tilde \zb_k}_{0} + f(\zb^0)\\
        &= \sum_{k \in \Kcal_j}\tilde\ab_k^{j_0} \tilde \zb_k + f(\zb^0)\\
        &= \sum_{k \in \Kcal_j}\left(f(\zb^k+\zb^0) - f(\zb^0)\right) \tilde \zb_k + f(\zb^0)\\
        &= \sum_{k \in \Kcal_j}\left(f(\zb^k+\zb^0) - f(\zb^0)\right) \tilde \zb_k + \sum_{k \not\in \Kcal_j}\left(f(\zb^k+\zb^0) - f(\zb^0)\right) \underbrace{\tilde \zb_k}_{0} +f(\zb^0)\\
        &= \begin{bmatrix}
            (f(\zb^1+\zb^0) - f(\zb^0)) & \cdots & (f(\zb^n+\zb^0) - f(\zb^0))  
        \end{bmatrix}\tilde \zb + f(\zb^0) \\
        &= \begin{bmatrix}
            (f(\zb^1+\zb^0) - f(\zb^0)) & \cdots & (f(\zb^n+\zb^0) - f(\zb^0))  
        \end{bmatrix} \Bb^{-1}(\zb - \zb^0) + f(\zb^0) \\
         &= \ab(\zb -\zb^0) + f(\zb^0)\\
        &= \ab\zb + \underbrace{f(\zb^0)-\ab\zb^0}_{b :=} \\
        &= \ab\zb + b \,.
    \end{align}
    Since $\zb \in \cup_{j \in [J]} \Zcal_j$ was arbitrary, the above equality holds for all $\zb$ in $\Zcal = \cup_{j \in [J]} \Zcal_j$. This concludes the proof.
\end{proof}

It is easy to construct examples of collections of subspaces $\Zcal_1, \dots, \Zcal_J$ satisfying the condition of Prop.~\ref{prop:globally_affine}. One can simply choose some basis $\zb^1, \cdots \zb^n$ of $\RR^n$ and construct each $\Zcal_j$ as the span of some subset of this basis. In particular, one can use the standard basis $\eb^1, \dots, \eb^n$, which will yield axis-aligned subspaces of $\RR^n$. We provide a counter-example in \cref{app: assump}.

Then by combining Ass.~\ref{ass: common basis} with \cref{thm: affine for components}, we can directly conclude the identifiability up to affine transformation in the following theorem.

\thmglobalaffineidenti*
\begin{proof}
Due to \cref{thm:affine-identi-comp}, we have $\gb\circ \fb$ is affine restricted on $\Zcal_1,...,\Zcal_J$ individually. In addition, since we assume $\Zcal_1,...,\Zcal_J$ satisfying Ass.~$\ref{ass: common basis}$, by applying Prop.~\ref{prop:globally_affine}, we can conclude $\gb\circ \fb$ is affine on $\Zcal=\cup_{j=1}^J\Zcal_j$.
\end{proof}

\subsubsection{From affine identifiability to element-wise identifiability via sparsity principle}
\label{app: PD identi}
In this section, we take a step further to investigate the possibility of achieving a stronger level of identifiability—namely, up to permutation and scaling indeterminacies only. We first introduce the following assumption, which is a special case of Ass.~\ref{ass: common basis} that assumes zero translation vector and standard common basis of $\RR^n$.

\assstandcommonbasis* 

We then adapt a key argument due to \citet{lachapelle2023synergies} in the context of sparse multitask learning and later adapted by \citet{xu2024sparsity} in the context of partial observability. In our case, the masked latent variables exactly follow a piecewise-degenerate Gaussian mixture model (pdGMM) with supports that share a common standard basis. The adapted result is formalized as follows:

\begin{lemma}[Element-wise Identifiability from sparsity \citep{lachapelle2023synergies,xu2024sparsity}] Assume that the latent variables $\Zb$ with support $\Zcal$ follows a pdGMM (Def.~\ref{def: reduce gmm}). The supports of its components satisfy Ass.~\ref{ass: standard common basis}. Let the function $\vb: \Zcal \rightarrow \RR^n$ be invertible and affine on $\Zcal$, and
\begin{align}
\mathbb{E}\norm{\vb(\Zb)}_0 &\leq \mathbb{E}\norm{\Zb}_0.
\label{equ: sparsity_constraint0}
\end{align}
Then $\vb$ is a permutation composed with an element-wise invertible affine transformation on $\Zcal$.
\label{lemma: element_wise v}
\end{lemma}
Since from \cref{thm:global-affine-identi} the mixing process is successfully reduced from the piecewise affine to affine, we can directly apply Lemma~\ref{lemma: element_wise v} on this results under Ass.~\ref{ass: standard common basis} to conclude the following theorem.

\thmPDidenti*
\begin{proof}
Since satisfying Ass.~\ref{ass: standard common basis} implies that Ass.~\ref{ass: common basis} holds, by \cref{thm:global-affine-identi}, we have $\gb\circ \fb$ is affine and invertible on $\Zcal$. By Lemma.~\ref{lemma: element_wise v}, we can conclude $\gb\circ \fb$ is a permutation composed with an element-wise invertible affine transformation on $\Zcal$.
\end{proof}

\section{Implementation details 
and Licences}
\label{app: implementation}

This section provides further details about the experiment implementation in \cref{sec: implementation}. The implementation is built on the open-source code by \citet{lachapelle2022disentanglement} released under Apache 2.0 License; the code by \citet{ahuja2022weakly} released under MIT License; the code by \citet{zheng2018dags} released under Apache 2.0 License. We list all the hyperparameters we used for training in \cref{tab:train_hyper}, and provide all the code for our method, and the experiments at 
\href{https://github.com/danrux/Identi-pdGMM}{GitHub repository}.

\begin{table}[h]
\caption{Parameters for experiments results in Sec.~\ref{sec: experiments} and App.~\ref{app: experiments}.}
\label{tab:train_hyper}
\centering
\begin{tabular}{llllll}
\hline
           & Numerical & Numerical & Balls  & Balls  & Balls  \\ 
           & Stage 1   &  Stage 2  & CNN    & Stage 1 & Stage 2\\ 
\hline
Sparsity level $\epsilon$ &        & 0.01      &   &   & 0.01   \\
Optimizer  & Adam      & ExtraAdam & Adam & Adam & ExtraAdam \\
Primal optimizer learning rate   & 1e-4    & 1e-4   & 1e-3 & 1e-3 & 1e-4 \\
Dual optimizer learning rate    &    & 5e-5  &  &  & 5e-5 \\
Batch size       & 6144   & 6144    & 256  & 256  & 256       \\
Group number $J$ & $5n$      & $5n$   & $2^{\bb}$  & $2^{\bb}$ & $2^{\bb}$ \\
\# Seeds   & 5      & 5      & 5  & 5  & 5 \\
\# Iterations   & 5000     & 80000   & 20000   & 10000  & 5000     \\
\hline
\end{tabular}

\end{table}

\section{Full experimental results}
\label{app: experiments}
\subsection{Metrics}
\label{app: metric}
\paragraph{Metric for identifiability up to affine transformation} $R^2$ metric is the \emph{coefficient of determination} of the linear regression model with ground truth latents $\Zb$ as explanatory and our learned representations $\gb(\Xb)$ as response. $R^2$ normally ranges from 0 to 1, and the closer it is to 1, the better its affine identifiability.

\paragraph{Metric for identifiability up to permutation and scaling.} MCC metric is based on the Pearson correlation matrix $\mathrm{Corr}^{n\times n}$ between the learned representations $\hat{\Zb}=\gb(\Xb)$ and ground truth latent variables $\Zb$. 
Since our results are up to permutation, we compute the MCC on the permutation $\pi$ that maximizes the average of $|\mathrm{Corr}_{i,\pi(i)}|$ for each index of a ground truth variable $i \in [n]$, i.e. MCC=$\frac{1}{n}\max_{\pi \in perm([n])} \sum_{i=1}^{n}|\mathrm{Corr}_{i,\pi(i)}|$. We denote the correlation matrix with the permutation $\pi$ as $\mathrm{Corr}^{n\times n}_{\pi}$. MCC can take value in $[0,1]$, and the closer it is to 1, the better element-wise affine identifiability.

\subsection{Numerical experiments}
We first show the full version (including standard deviation) of \cref{tab: numerical}.

\begin{table}[H]
\caption{Results for the numerical experiments (Average $\pm$ Standard Deviation).
The bold font indicates which parameters are varying in each block of rows. %
$\Rb^2$(Stage 1) shows the identifiability up to \emph{AT} after Stage 1 (Thm.~\ref{thm:global-affine-identi}); %
\textbf{MCC}(Stage 1) reflects the neccesarity of sparsity to achieve \emph{PTS} identifiability; %
\textbf{MCC}(Stage 2) presents the identifiability up to \emph{PTS} after Stage 2 (Thm.~\ref{thm: PD identi}).}
\label{tab: numerical_full}
\begin{center}
\begin{sc}
\begin{tabular}{cccccccccc}
\hline
$n$ & $k$ & $m$ &  $\rho$ &  $\delta$ &  $\theta$ & $\Rb^2$ $\uparrow$(Stage1)   & \textbf{MCC} (Stage1) &  \textbf{MCC} (Stage1) & \textbf{MCC} $\uparrow$(VaDE)\\ 
\hline
\textbf{5}  & 1 & 10 & 50\% & 0 & 0 & 0.93$\pm$0.02 & 0.75$\pm$0.03 & 0.96$\pm$0.01 `& 0.45$\pm$0.05\\
\textbf{10} & 1 & 10 & 50\% & 0 & 0 & 0.94$\pm$0.02 & 0.56$\pm$0.03 & 0.97$\pm$0.01 & 0.32$\pm$0.03\\
\textbf{20} & 1 & 10 & 50\% & 0 & 0 & 0.93$\pm$0.02 & 0.46$\pm$0.01 & 0.92$\pm$0.05 & 0.23$\pm$0.15\\
\textbf{40} & 1 & 10 & 50\% & 0 & 0 & 0.94$\pm$0.01 & 0.37$\pm$0.01 & 0.94$\pm$0.03 & 0.25$\pm$0.11\\
\hline
10 & \textbf{0} & 10 & 50\% & 0 & 0 & 0.94$\pm$0.02 & 0.51$\pm$0.11 & 0.97$\pm$0.01 & 0.33$\pm$0.05\\
10 & \textbf{1} & 10 & 50\% & 0 & 0 & 0.94$\pm$0.02 & 0.56$\pm$0.03 & 0.97$\pm$0.01 & 0.32$\pm$0.03 \\
10 & \textbf{2} & 10 & 50\% & 0 & 0 & 0.93$\pm$0.01 & 0.58$\pm$0.01 & 0.95$\pm$0.01 & 0.31$\pm$0.03\\
10 & \textbf{3} & 10 & 50\% & 0 & 0 & 0.92$\pm$0.03 & 0.56$\pm$0.02 & 0.91$\pm$0.04 & 0.27$\pm$0.03 \\
\hline
10 & 1 & \textbf{3} & 50\% & 0 & 0 & 0.99$\pm$0.01 & 0.50$\pm$0.07 & 0.96$\pm$0.03 & 0.33$\pm$0.04 \\
10 & 1 & \textbf{10} & 50\% & 0 & 0 & 0.94$\pm$0.02 & 0.56$\pm$0.03 & 0.97$\pm$0.01 & 0.32$\pm$0.03\\
10 & 1 & \textbf{20} & 50\% & 0 & 0 & 0.88$\pm$0.02 & 0.50$\pm$0.07 & 0.93$\pm$0.01 & 0.35$\pm$0.07\\
\hline
10 & 1 & 10 & \textbf{1 var} & 0 & 0 & 0.75$\pm$0.03 & 0.46$\pm$0.02 & 0.45$\pm$0.04 & 0.28$\pm$0.03 \\
10 & 1 & 10 & \textbf{50\%} & 0 & 0 & 0.94$\pm$0.02 & 0.56$\pm$0.03 & 0.97$\pm$0.01 & 0.32$\pm$0.03 \\
10 & 1 & 10 & \textbf{75\%} & 0 & 0 & 0.95$\pm$0.02 & 0.57$\pm$0.03 & 0.93$\pm$0.04 & 0.31$\pm$0.04\\
\hline
10 & 1 & 10 & 50\% & \textbf{0} & 0 & 0.94$\pm$0.02 & 0.56$\pm$0.03 & 0.97$\pm$0.01 & 0.32$\pm$0.03 \\
10 & 1 & 10 & 50\% & \textbf{1} & 0 & 0.95$\pm$0.01 & 0.59$\pm$0.01 & 0.88$\pm$0.04 & 0.30$\pm$0.02 \\
10 & 1 & 10 & 50\% & \textbf{2} & 0 & 0.96$\pm$0.01 & 0.59$\pm$0.03 & 0.61$\pm$0.04 & 0.27$\pm$0.03\\
10 & 1 & 10 & 50\% & \textbf{3} & 0 & 0.96$\pm$0.01 & 0.58$\pm$0.01 & 0.49$\pm$0.02 & 0.28$\pm$0.03\\
\hline
10 & 1 & 10 & 50\% & 0 & \textbf{0} & 0.94$\pm$0.02 & 0.56$\pm$0.03 & 0.97$\pm$0.01 & 0.32$\pm$0.03\\
10 & 1 & 10 & 50\% & 0 & \textbf{15} & 0.93$\pm$0.02 & 0.58$\pm$0.03 & 0.83$\pm$0.02 & 0.32$\pm$0.03 \\
10 & 1 & 10 & 50\% & 0 & \textbf{30} & 0.93$\pm$0.02 & 0.57$\pm$0.03 & 0.60$\pm$0.02 & 0.31$\pm$0.04 \\
10 & 1 & 10 & 50\% & 0 & \textbf{45} & 0.93$\pm$0.02 & 0.57$\pm$0.03 & 0.62$\pm$0.01 & 0.32$\pm$0.03\\
\hline

\end{tabular}
\end{sc}
\end{center}
\end{table}

We show the Pearson Correlation matrix $\mathrm{Corr}^{n\times n}_{\pi}$ with the permutation $\pi$ between ground truth latent variables $\Zb$ and the learned representations $\hat{\Zb}=\gb(\Xb)$. One figure represents one ablation study in \cref{tab: numerical}. We choose one random seed to plot for each setup. Ground truth Pear. Corr. matrix on the left shows the original linear correlation inside $\Zb$, compared with the estimator on the right-hand side.
\begin{figure*}[!htbp]
    \begin{center}
        \includegraphics[width=0.18\textwidth]{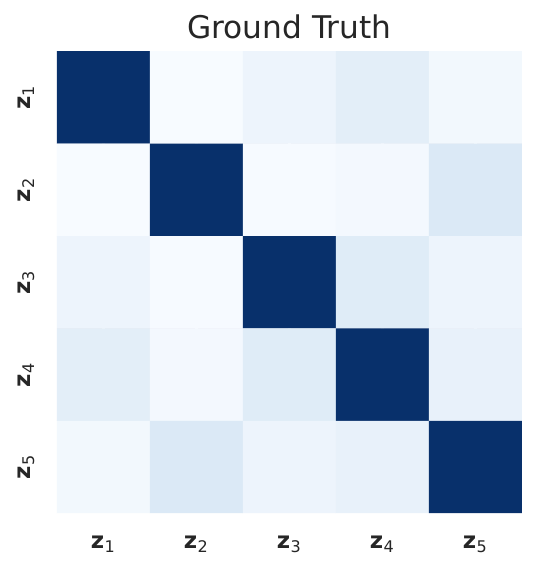}
        \hspace{0cm} 
        \includegraphics[width=0.18\textwidth]{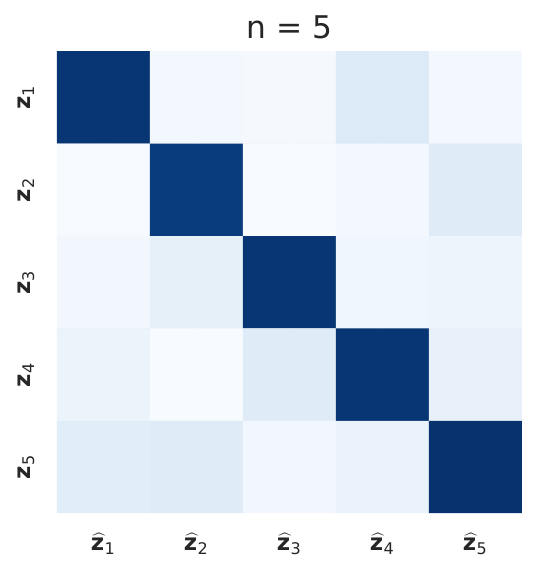}
        \hspace{0.5cm} 
         \includegraphics[width=0.18\textwidth]{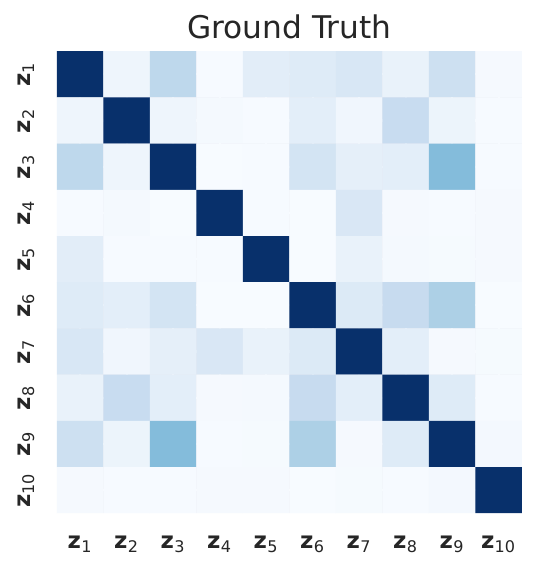}
        \hspace{0cm} 
        \includegraphics[width=0.18\textwidth]{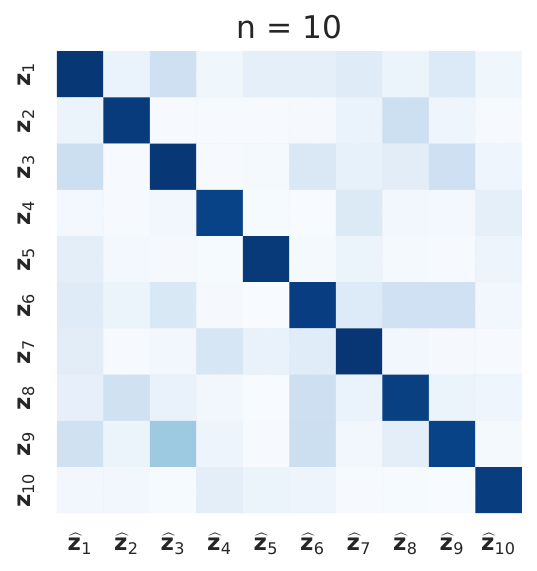}
        \includegraphics[width=0.0375\textwidth]{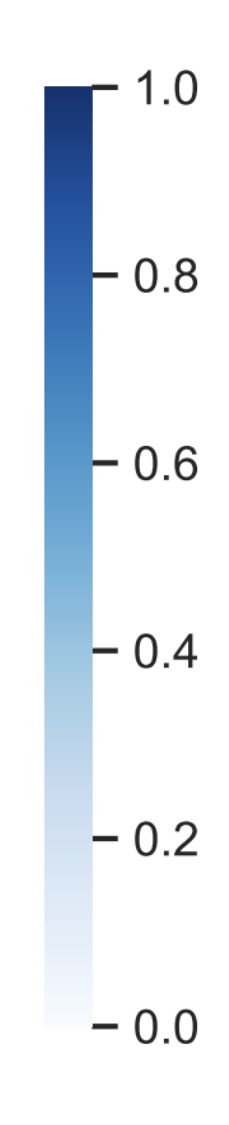}
        \hspace{0.5cm} 
        \includegraphics[width=0.18\textwidth]{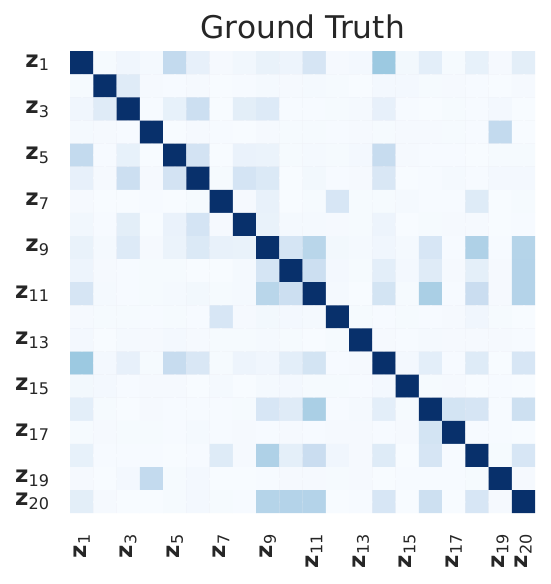}
        \hspace{0cm} 
        \includegraphics[width=0.18\textwidth]{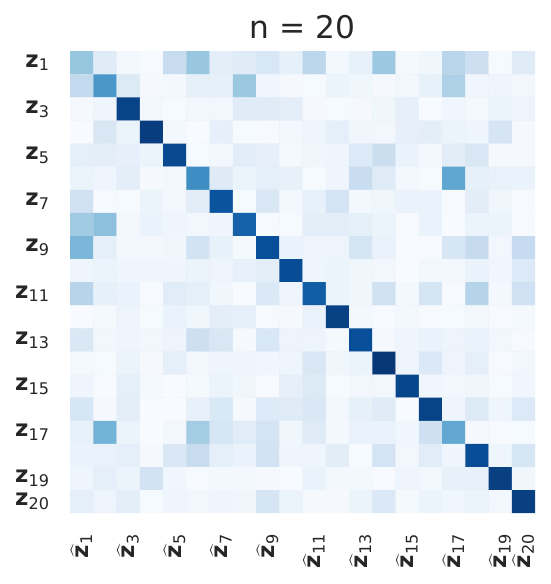}
        \hspace{0.5cm} 
        \includegraphics[width=0.18\textwidth]{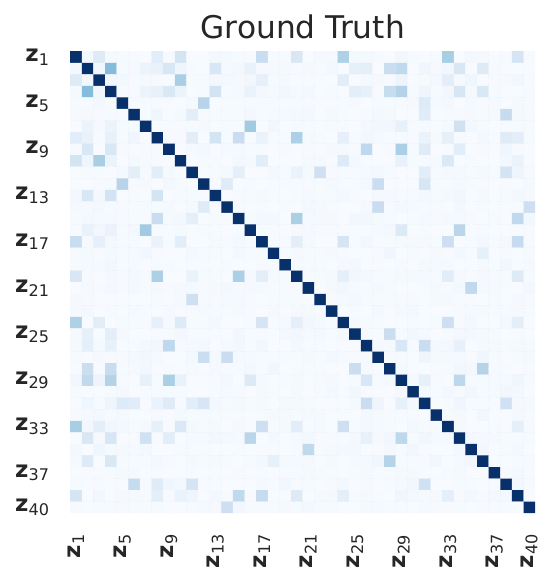}
        \hspace{0cm} 
        \includegraphics[width=0.18\textwidth]{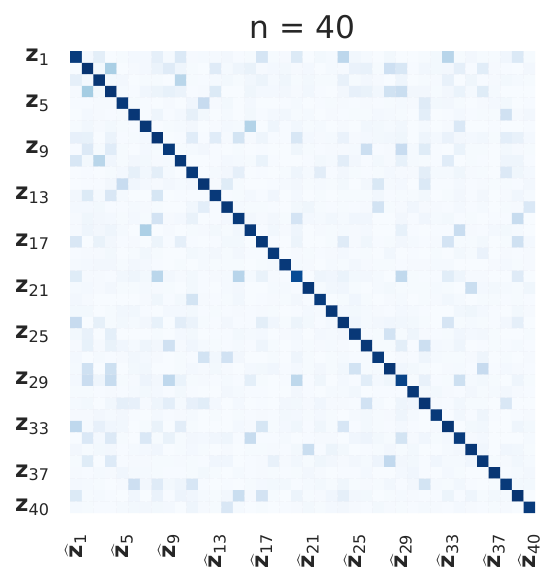}
        \includegraphics[width=0.0375\textwidth]{figures/cbar.pdf}
        \caption{Ablation study on increasing the latent dimension $n$ from 5 to 40 and fixing  $k=1$, $m=10$, $\rho=50\%$, $\delta=0.0\sigma$, $\theta=0$. The case with a higher $n$ is more complicated to learn the latent variables.}
        \label{fig:pw_ablation_n}
    \end{center}
\end{figure*}

\begin{figure*}[!htbp]
    \begin{center}
        \includegraphics[width=0.18\textwidth]{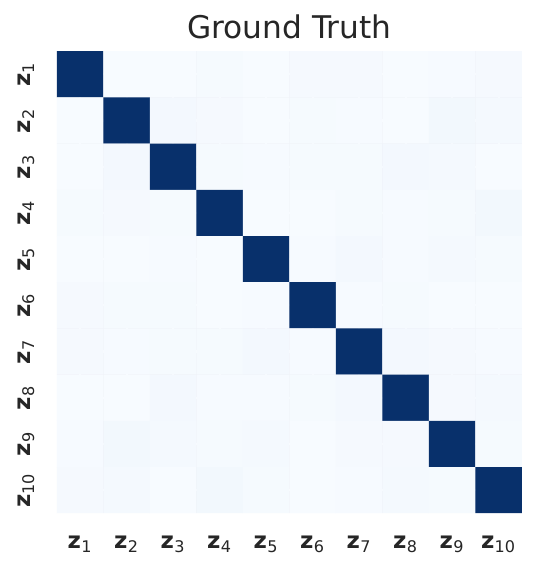}
        \hspace{0cm} 
        \includegraphics[width=0.18\textwidth]{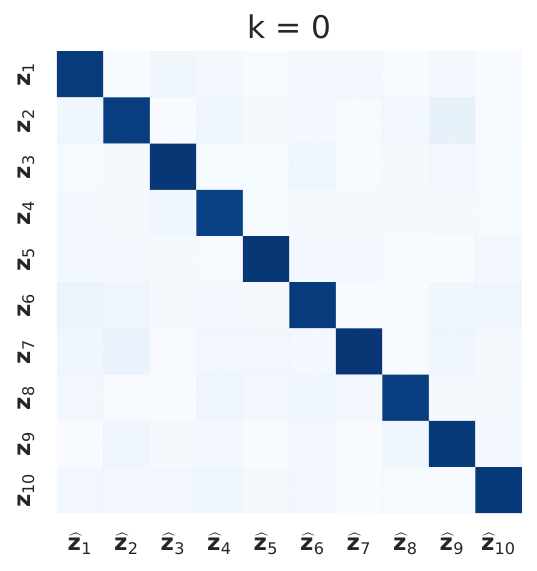}
        \hspace{0.5cm} 
         \includegraphics[width=0.18\textwidth]{figures/heatmaps_pw/true/pw10_d0.0_n10_nn10_M50_G1_rs2_Corr_heatmap.pdf}
        \hspace{0cm} 
        \includegraphics[width=0.18\textwidth]{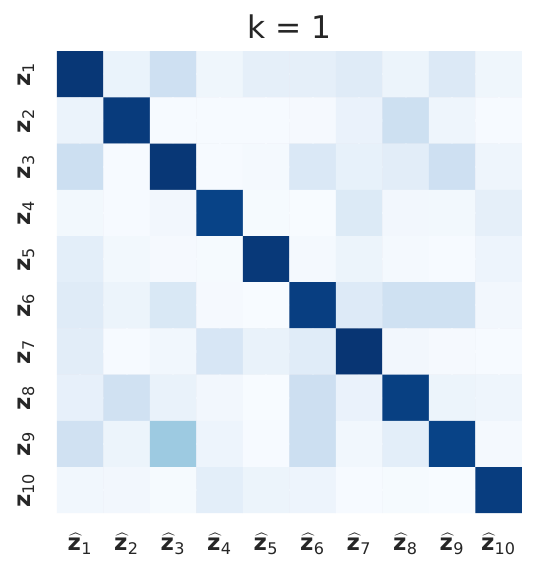}
        \includegraphics[width=0.0375\textwidth]{figures/cbar.pdf}
        \hspace{0.5cm} 
        \includegraphics[width=0.18\textwidth]{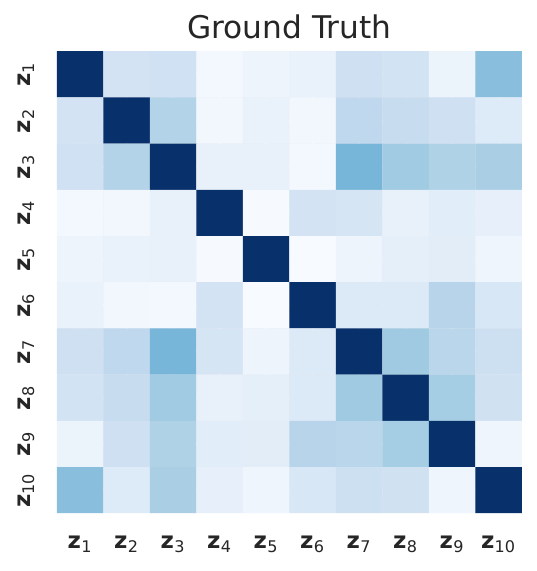}
        \hspace{0cm} 
        \includegraphics[width=0.18\textwidth]{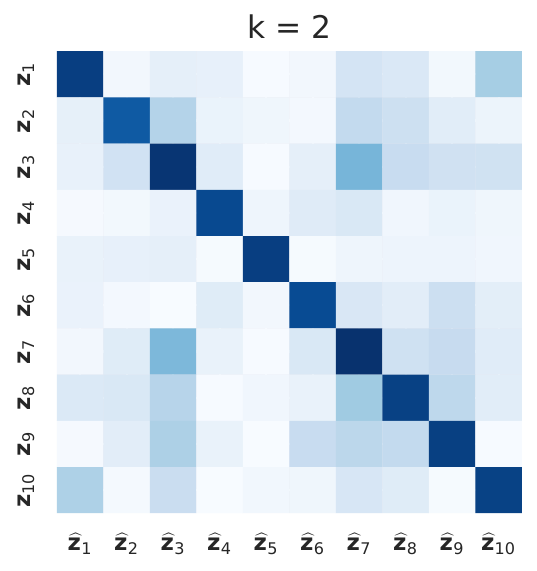}
        \hspace{0.5cm} 
        \includegraphics[width=0.18\textwidth]{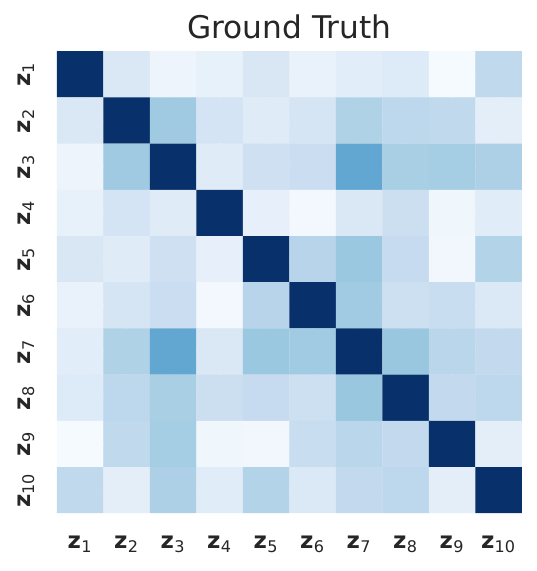}
        \hspace{0cm} 
        \includegraphics[width=0.18\textwidth]{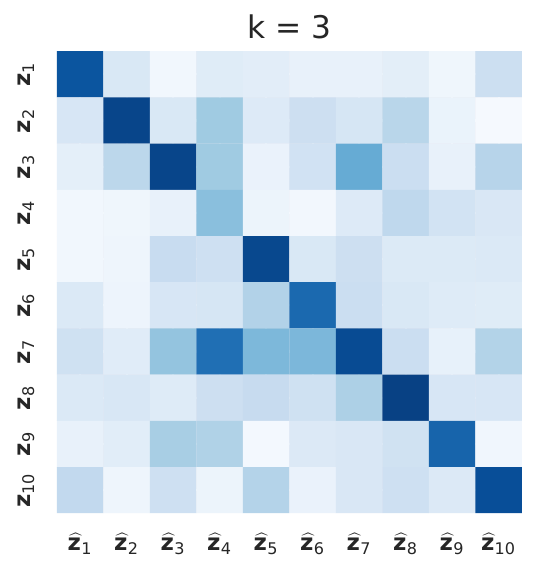}
        \includegraphics[width=0.0375\textwidth]{figures/cbar.pdf}
        \caption{Ablation study on increasing the density of causal graphs $k$ from 0 to 3 and fixing $n=10$, $m=10$, $\rho=50\%$, $\delta=0.0\sigma$, $\theta=0$. In the case with a denser graph, it is more complicated to learn the latent variables.}
        \label{fig:pw_ablation_k}
    \end{center}
\end{figure*}

\begin{figure*}[!htbp]
    \begin{center}
        \includegraphics[width=0.18\textwidth]{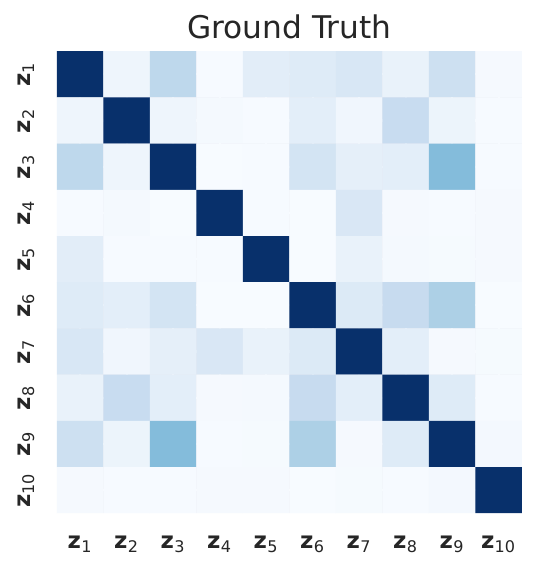}
        \hspace{0cm} 
        \includegraphics[width=0.18\textwidth]{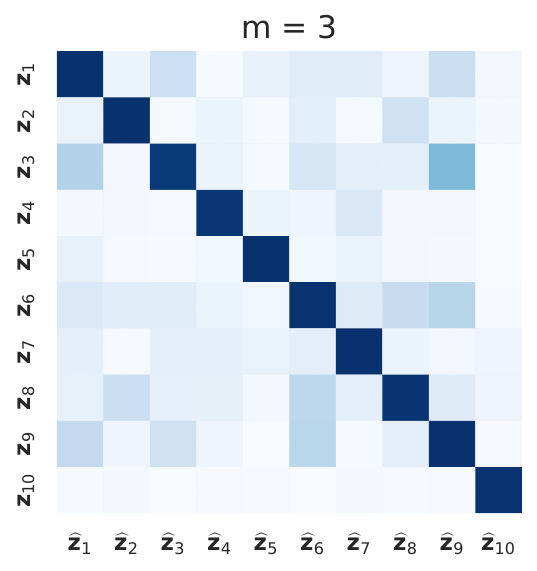}
        \hspace{0.5cm} 
         \includegraphics[width=0.18\textwidth]{figures/heatmaps_pw/true/pw10_d0.0_n10_nn10_M50_G1_rs2_Corr_heatmap.pdf}
        \hspace{0cm} 
        \includegraphics[width=0.18\textwidth]{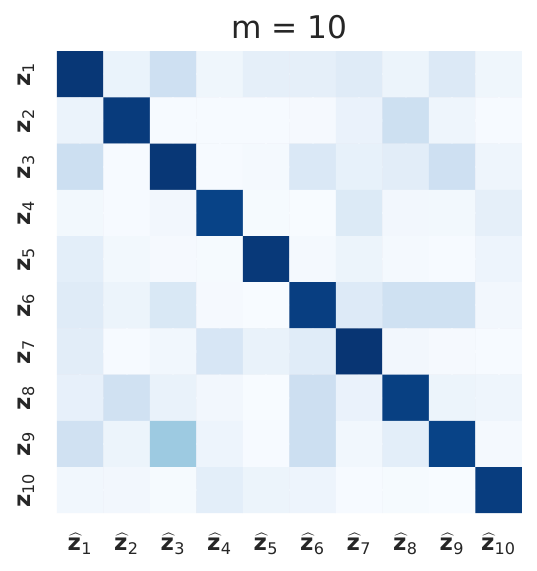}
        \includegraphics[width=0.0375\textwidth]{figures/cbar.pdf}
        \hspace{0.5cm} 
        \includegraphics[width=0.18\textwidth]{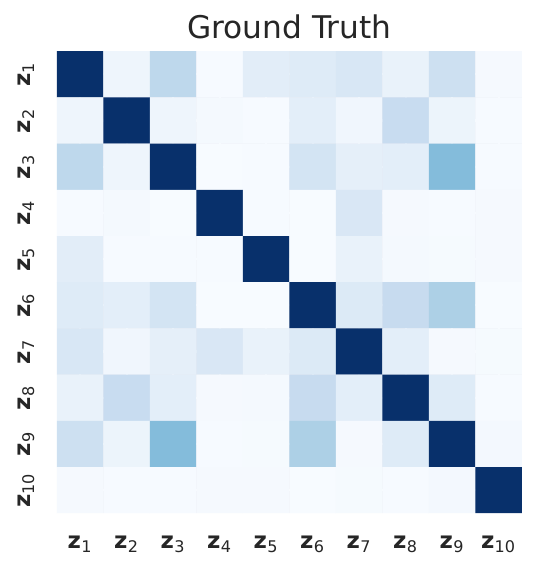}
        \hspace{0cm} 
        \includegraphics[width=0.18\textwidth]{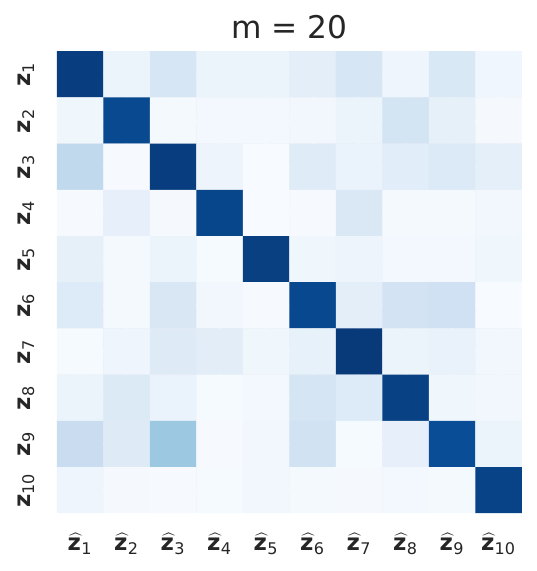}
        \includegraphics[width=0.0375\textwidth]{figures/cbar.pdf}
        \caption{Ablation study on increasing the number of Leaky-ReLU layers $(m-1)$ from 3 to 20 and fixing $n=10$, $k=1$, $\rho=50\%$, $\delta=0.0\sigma$, $\theta=0$. The case with a higher $n$ is more complicated to learn the latent variables. The case with larger $m$ is more complicated to learn the latent variables due to a greater extent of nonlinearity.}
        \label{fig:pw_ablation_m}
    \end{center}
\end{figure*}

\begin{figure*}[!htbp]
    \begin{center}
        \includegraphics[width=0.18\textwidth]{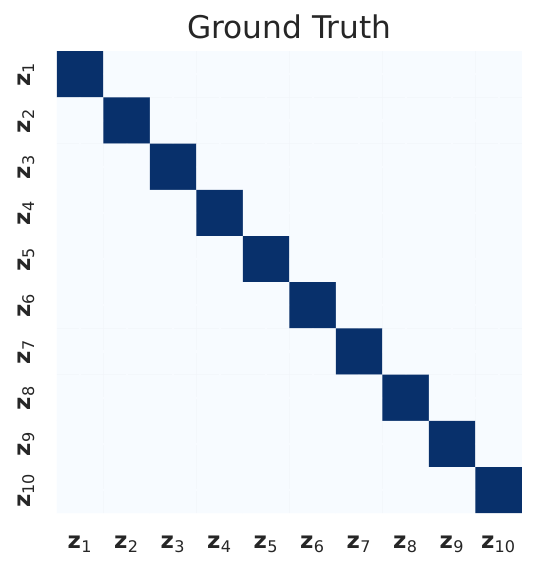}
        \hspace{0cm} 
        \includegraphics[width=0.18\textwidth]{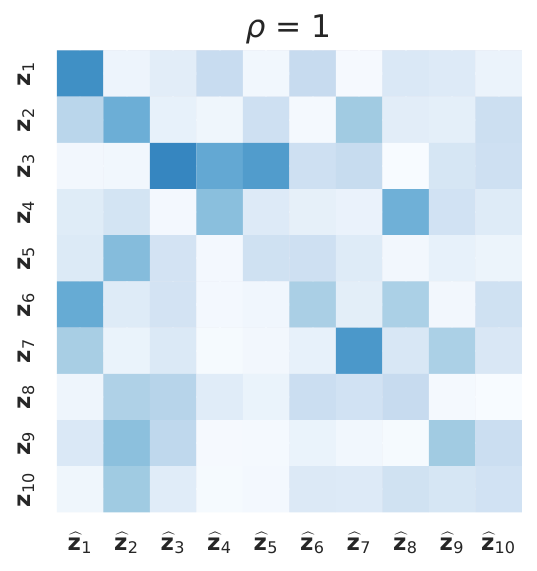}
        \hspace{0.5cm} 
         \includegraphics[width=0.18\textwidth]{figures/heatmaps_pw/true/pw10_d0.0_n10_nn10_M50_G1_rs2_Corr_heatmap.pdf}
        \hspace{0cm} 
        \includegraphics[width=0.18\textwidth]{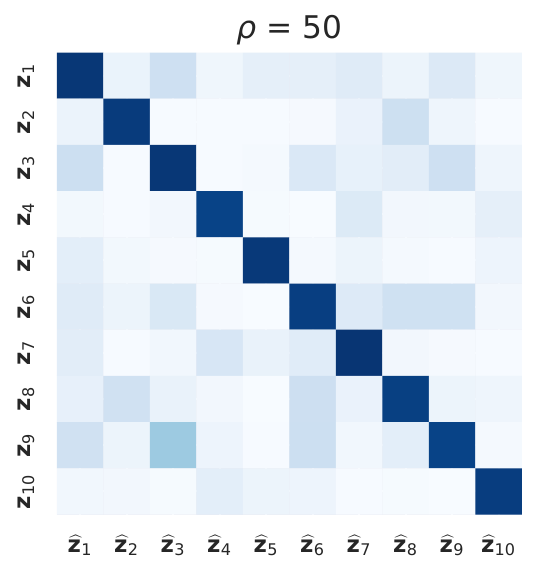}
        \includegraphics[width=0.0375\textwidth]{figures/cbar.pdf}
        \hspace{0.5cm} 
        \includegraphics[width=0.18\textwidth]{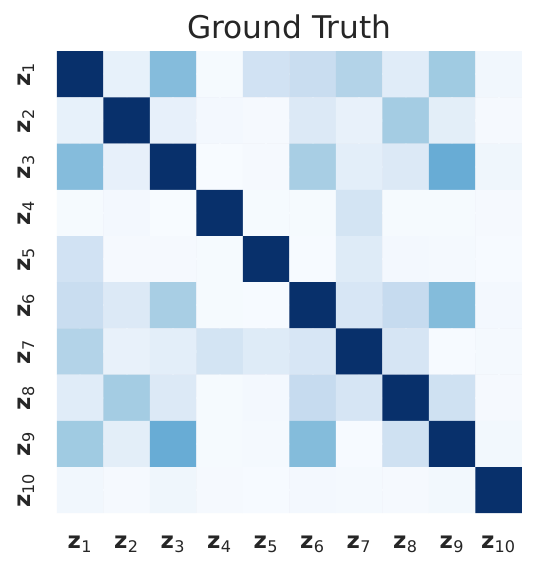}
        \hspace{0cm} 
        \includegraphics[width=0.18\textwidth]{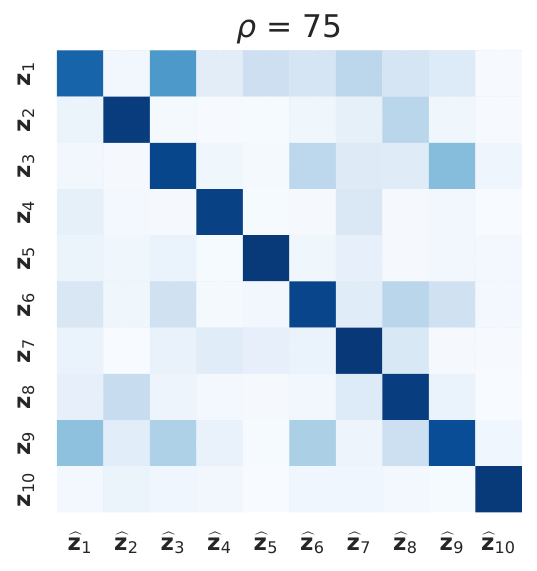}
        \includegraphics[width=0.0375\textwidth]{figures/cbar.pdf}
        \caption{Piecewise linear mixing function with linear causal relation and Gaussian noise: ablation study on increasing the ratio of active (unmasked) variables $\rho$ from 1 variable only to 75\% and fixing $n=10$, $k=1$, $m=10$, $\delta=0.0\sigma$, $\theta=0$. Learning the latent variables is more complicated in the case with a larger portion of active variables.}
        \label{fig:pw_ablation_rho}
    \end{center}
\end{figure*}

\begin{figure*}[!htbp]
    \begin{center}
        \includegraphics[width=0.18\textwidth]{figures/heatmaps_pw/true/pw10_d0.0_n10_nn10_M50_G1_rs2_Corr_heatmap.pdf}
        \hspace{0cm} 
        \includegraphics[width=0.18\textwidth]{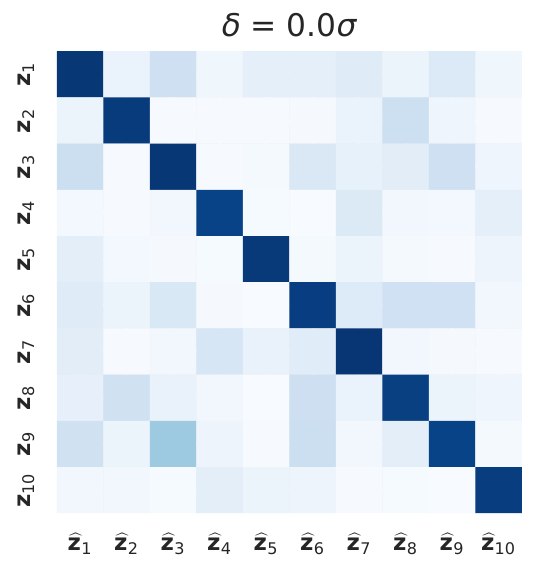}
        \hspace{0.5cm} 
         \includegraphics[width=0.18\textwidth]{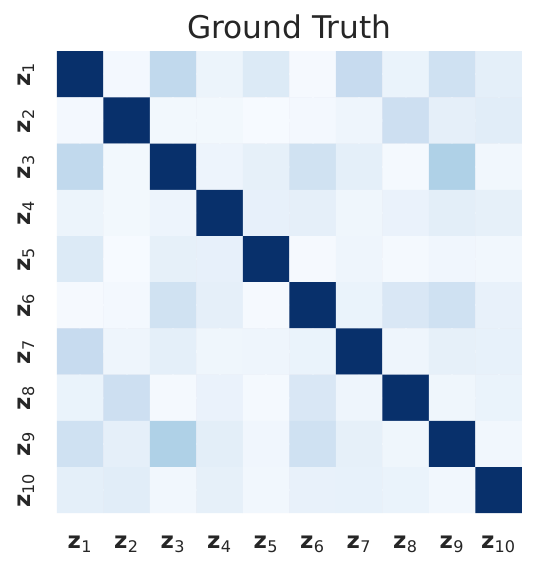}
        \hspace{0cm} 
        \includegraphics[width=0.18\textwidth]{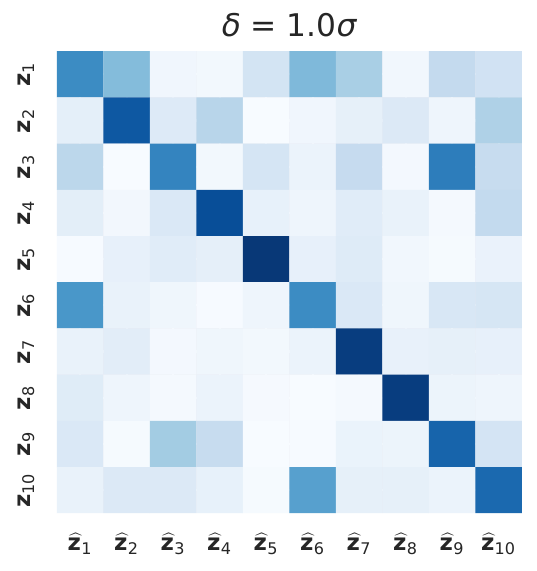}
        \includegraphics[width=0.0375\textwidth]{figures/cbar.pdf}
        \hspace{0.5cm} 
         \includegraphics[width=0.18\textwidth]{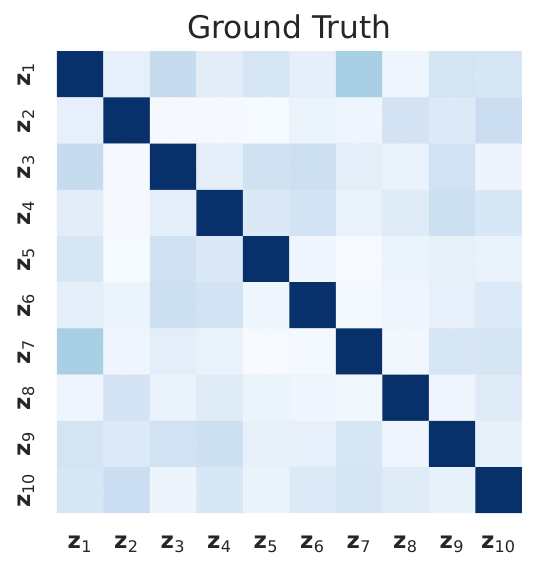}
        \hspace{0cm} 
        \includegraphics[width=0.18\textwidth]{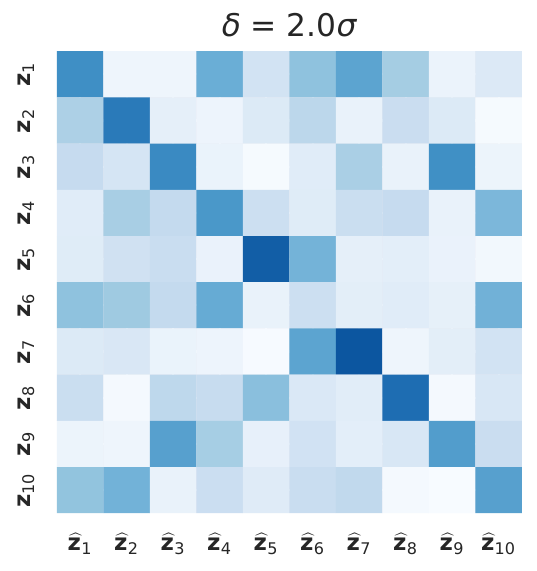}
        \hspace{0.5cm} 
         \includegraphics[width=0.18\textwidth]{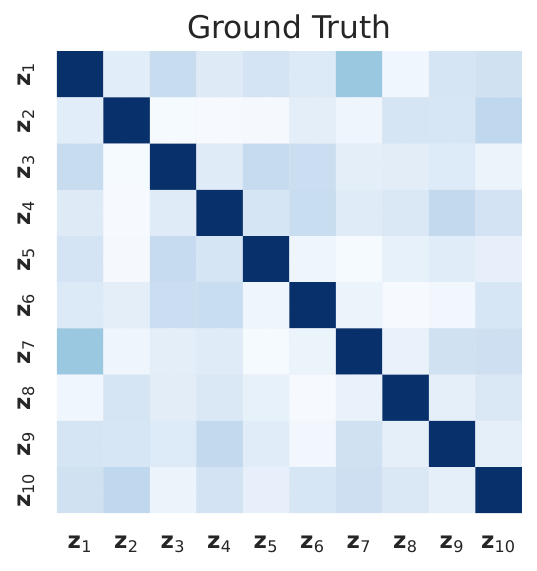}
        \hspace{0cm} 
        \includegraphics[width=0.18\textwidth]{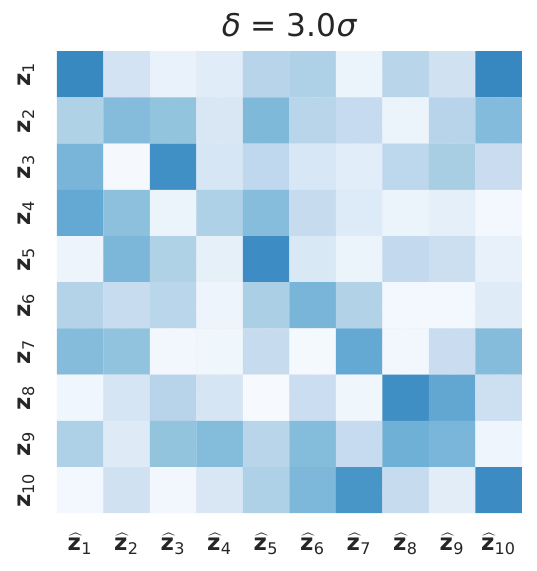}
        \includegraphics[width=0.0375\textwidth]{figures/cbar.pdf}
        \caption{Ablation study on increasing the distance between masked value and mean of latents from $0.0 \sigma$ to $3.0 \sigma$, where $\sigma$ is the standard deviation of latents, and fixing $n=10$, $k=1$, $m=10$, $\rho=50\%$, $\theta=0$. Disentangling the latent variables is more complicated in the case of a larger deviation from zero translation.}
        \label{fig:pw_ablation_delta}
    \end{center}
\end{figure*}

\begin{figure*}[!htbp]
    \begin{center}
        \includegraphics[width=0.18\textwidth]{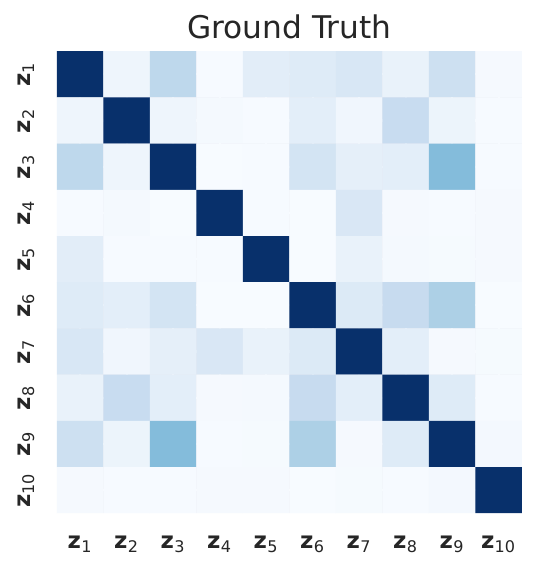}
        \hspace{0cm} 
        \includegraphics[width=0.18\textwidth]{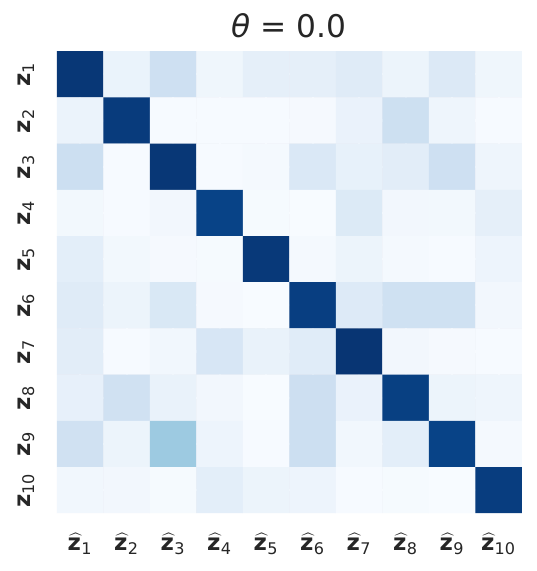}
        \hspace{0.5cm} 
         \includegraphics[width=0.18\textwidth]{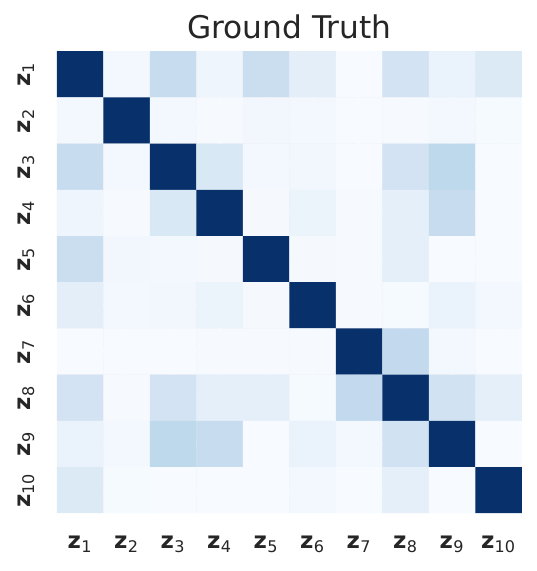}
        \hspace{0cm} 
        \includegraphics[width=0.18\textwidth]{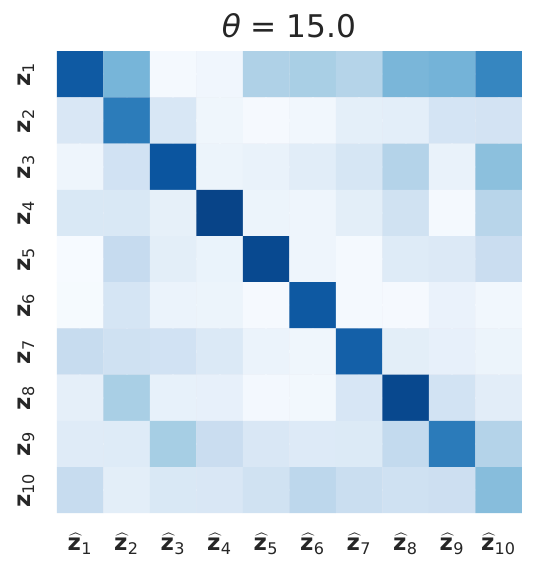}
        \includegraphics[width=0.0375\textwidth]{figures/cbar.pdf}
        \hspace{0.5cm} 
         \includegraphics[width=0.18\textwidth]{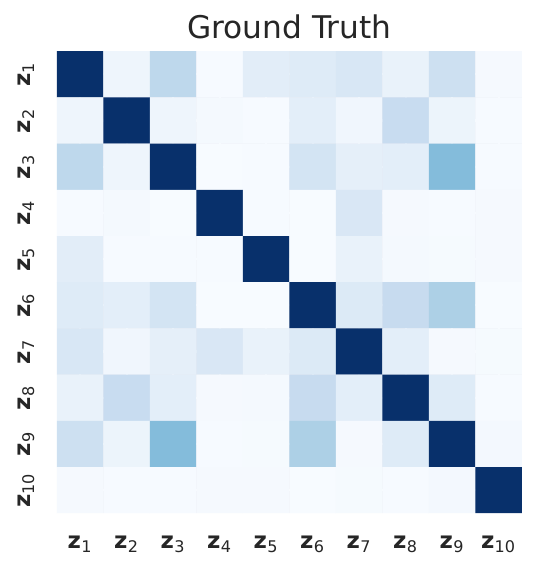}
        \hspace{0cm} 
        \includegraphics[width=0.18\textwidth]{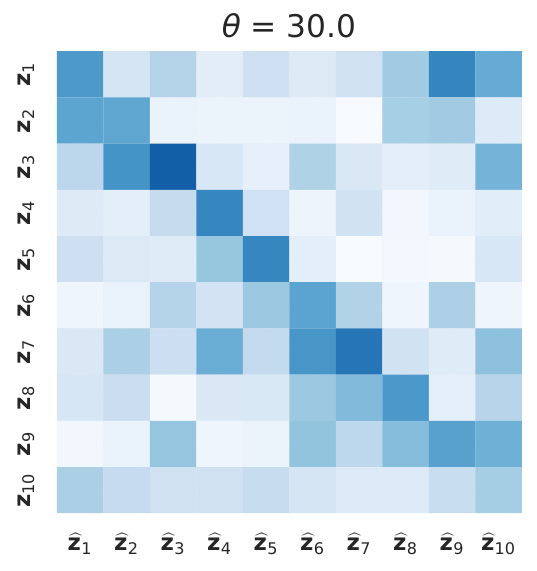}
        \hspace{0.5cm} 
         \includegraphics[width=0.18\textwidth]{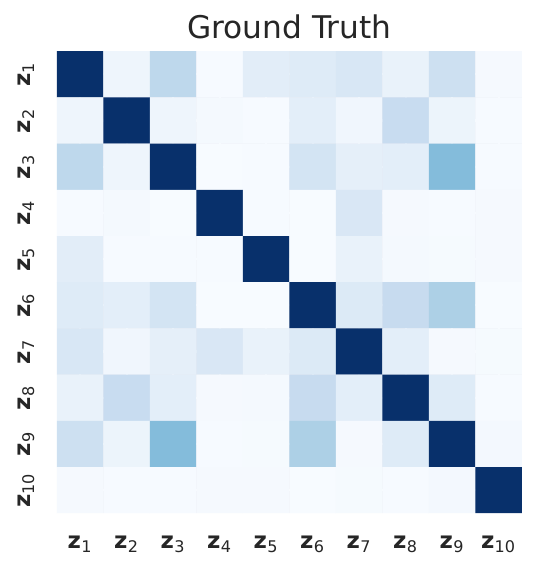}
        \hspace{0cm} 
        \includegraphics[width=0.18\textwidth]{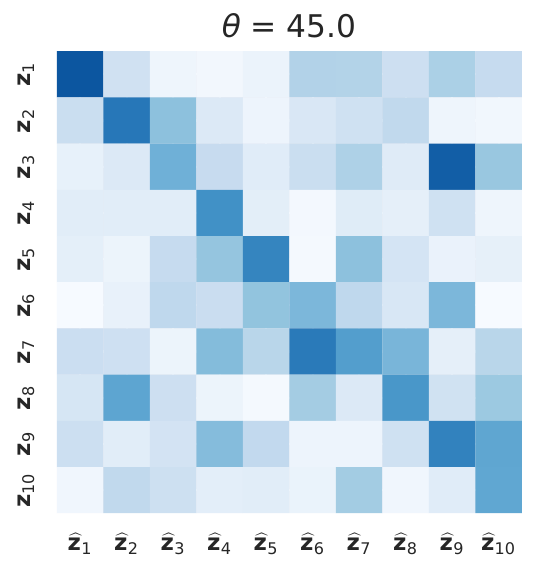}
        \includegraphics[width=0.0375\textwidth]{figures/cbar.pdf}
        \caption{Ablation study on increasing the rotation of common basis from $0.0^{\circ}$ to $45.0^{\circ}$, and fixing $n=10$, $k=1$, $m=10$, $\rho=50\%$, $\delta=0.0\sigma$. Disentangling the latent variables is more complicated in the case of a larger deviation from the standard basis.}
        \label{fig:pw_ablation_theta}
    \end{center}
\end{figure*}
\FloatBarrier

\subsection{Test on independent variables}
\label{app: num independent}
To exclude the natural dependency between latent variables that potentially increases the value of metric MCC that might cause a false positive measurement, we provide additional experiment results with independent latent variables using five random seeds in \cref{tab: numerical indep}. As we can see, it keeps the consistent performance as those in \cref{tab: numerical}, where we repeat the results in the right column of \cref{tab: numerical indep} for easier comparison. This provides evidence that the MCC obtained by our method does not come from the inherent dependencies among latents.

\begin{table}[H]
\caption{Results for the numerical experiments (Stage 2) tested on independent data.
The bold font indicates which parameters are varying in each block of rows.}
\label{tab: numerical indep}
\begin{center}
\begin{small}
\begin{sc}
\begin{tabular}{cccccccc}
\hline
$n$ & $k$ & $m$ &  $\rho$ &  $\delta$ &  $\theta$ & \textbf{MCC} Independent   & \textbf{MCC} in Tab.~\ref{tab: numerical}  \\ \hline
\textbf{5}    & 1 & 10  & 50 \%  & 0 & 0 &    0.96$\pm$0.01  & 0.96$\pm$0.01 \\
\textbf{10}   & 1 & 10  & 50 \%  & 0 & 0 &  0.97$\pm$0.01 & 0.97$\pm$0.01 \\
\textbf{20}   & 1 & 10  & 50 \%  & 0 & 0 &   0.93$\pm$0.05 & 0.92$\pm$0.05 \\
\textbf{40}   & 1 & 10  & 50 \%  & 0 & 0 & 0.94$\pm$0.03  & 0.94$\pm$0.03 \\
\hline
10    & \textbf{0} & 10  & 50 \% & 0 & 0 &  0.97$\pm$0.01  & 0.97$\pm$0.01\\
10    & \textbf{1} & 10 & 50 \%  & 0 & 0 &  0.97$\pm$0.01 & 0.97$\pm$0.01\\
10   &  \textbf{2} & 10 & 50 \%  & 0 & 0 &  0.97$\pm$0.01 & 0.95$\pm$0.01 \\
10   &  \textbf{3} & 10 & 50 \%  & 0 & 0 &  0.91$\pm$0.04 & 0.91$\pm$0.04\\
\hline
10   &  1 & \textbf{3}  & 50 \%  & 0 & 0 &  0.96$\pm$0.03 & 0.96$\pm$0.03\\
10   &  1 & \textbf{10}  & 50 \% & 0 & 0 & 0.97$\pm$0.01  & 0.97$\pm$0.01 \\
10   &  1 & \textbf{20}  & 50 \% & 0 & 0 & 0.94$\pm$0.02 & 0.93$\pm$0.01\\
\hline
10    & 1 & 10 & \textbf{1\text{var}}  & 0 &0 & 0.46 $\pm$0.04 & 0.45$\pm$0.04 \\
10   &  1 & 10 & \textbf{50 \%}   & 0 &0 &  0.97$\pm$0.01 & 0.97$\pm$0.01\\
10   &  1 & 10 & \textbf{75 \%}  & 0 &0 &      0.93 $\pm$0.03  & 0.93$\pm$0.04 \\
\hline
10   & 1 & 10 & 50 \%   & \textbf{0} & 0 &  0.97$\pm$0.01 & 0.97$\pm$0.01 \\
10   & 1 & 10 & 50 \%   & \textbf{1} & 0 &  0.87$\pm$0.04 & 0.88$\pm$0.04\\
10   & 1 & 10 & 50 \%   & \textbf{2} & 0 &  0.61$\pm$0.04 & 0.61$\pm$0.04 \\
10   & 1 & 10 & 50 \%   & \textbf{3} & 0 &  0.49$\pm$0.02  & 0.49$\pm$0.02 \\
\hline
10   &  1 & 10  & 50 \%  & 0 & \textbf{0} &  0.96$\pm$0.03 & 0.97$\pm$0.01\\
10   &  1 & 10  & 50 \% & 0 & \textbf{15} & 0.84 $\pm$0.01  & 0.83$\pm$0.02 \\
10   &  1 & 10  & 50 \% & 0 & \textbf{30} & 0.61$\pm$0.01 & 0.60$\pm$0.02\\
10   &  1 & 10  & 50 \% & 0 & \textbf{45} & 0.63$\pm$0.01 & 0.62$\pm$0.01 \\
\hline
\end{tabular}
\end{sc}
\end{small}
\end{center}
\end{table}

\subsection{Empirical study on scalability}
\label{subapp: complexity}

As shown in Tab.~\ref{tab:train_hyper}, Stage 1 only need 5000 epoch whereas Stage 2 need 80000. Consequently, the second stage of optimization constitutes the primary computational bottleneck of our method. To quantify scalability, we report both the number of epochs needed for convergence and the average execution time per 1000 epochs in seconds for the second stage. We can observe an exponential complexity increase in the table:

\begin{table}[H]
\centering
\caption{Number of epochs for convergence and training time}
\begin{tabular}{c|ccccc}
\hline
$n$ & 5 & 10 & 20 & 40 & 50 \\
\hline
Epoch & 2,500 & 20,000 & 125,000 & 300,000 & 800,000 \\
\hline
Execution time per 1000 epoch in Stage2 (s) & 13.4 & 20.2 & 45.6 & 145.4 & 219.7 \\
\hline
\end{tabular}
\end{table}
In particular we report results for $n=50$ and limited results for $n=100$. For comparisons, few works in the CRL literature are able to scale reliably to this number of latent variables, and most of the methods typically focus on 10-20 variables.
For $n=50$, we got $MCC=0.93 \pm 0.03$ averaged over 3 random seeds.
For $n=100$, full convergence of the second stage was not reached. However, after 15,000 epochs of the first stage, we observed an average $R^2=0.82 \pm 0.13$, indicating a meaningful affine identifiability which validates \cref{thm:global-affine-identi}.

\subsection{Sensitivity Analysis of Hyper-parameters}
\label{subapp: sensitive analysis}
We performed a sensitivity analysis on two hyperparameters: the sparsity level $\epsilon$ in Equation (5), which controls the sparsity level, and the learning rate. Bold options in the tables are the ones we used in the main paper, where $\epsilon=1e-2$ is inspired by \citet{xu2024sparsity}, which shows that the results are not sensitive to $\epsilon$, but they are sensitive to the learning rate. 

\begin{table}[H]
\centering
\caption{Sensitivity analysis of sparsity level $\epsilon$ in Eq.~(5): Average MCCs over 5 random seeds (Setup: $n=10$, $k=1$, $m=10$, $\rho=50$, $\delta=0$, $\theta=0$)}
\begin{tabular}{c|ccccc}
\hline
$\epsilon$ & 1000 & 100 & 10 & 1 & 1e-1 \\
\hline
MCC & 0.50$\pm$0.01 & 0.50$\pm$0.01 & 0.50$\pm$0.01 & 0.50$\pm$0.01 & 0.90$\pm$0.08 \\
\hline
\hline
$\epsilon$ & \textbf{1e-2} & 1e-3 & 1e-4 & 1e-5 &  \\
\hline
MCC & \textbf{0.97$\pm$0.01} & 0.94$\pm$0.03 & 0.90$\pm$0.07 & 0.88$\pm$0.07 &  \\
\hline
\end{tabular}
\end{table}

The final hyperparameters we used in the experiments are chosen based on the MCC and  $R^2$ computed with the ground truth variables.
If we did not know the ground truth, tuning the parameters would be more difficult, but in line with other DL approaches. For example, we can choose the learning rate as usual in DL by observing the behavior of the loss, while for the sparsity level $\epsilon$
 we can choose it based on potential violations of our assumptions in the reconstructed variables, e.g., if the reconstructed variables are non-Gaussian.

\begin{table}[H]
\centering
\caption{Sensitivity analysis of learning rate: Average MCCs over 5 random seeds}
\begin{tabular}{c|cccc}
\hline
Learning rate & 1e-2 & 1e-3 & \textbf{1e-4} & 1e-5 \\
\hline
MCC & 0.25$\pm$0.03 & 0.53$\pm$0.08 & \textbf{0.97$\pm$0.01} & 0.92$\pm$0.02 \\
\hline
\end{tabular}
\end{table}

\subsection{Ablation study on sparsity constraint and comparison with close works}
\label{subapp: baselines}
We ran the experiments on all settings without a sparsity constraint in the second stage. As summarized \cref{tab: two bench}, the significant drop in MCC shows the important role that sparsity plays in the disentanglement process.

Additionally, we compared two works \citep{xu2024sparsity, kivva2022identifiability} which are closely related to our work, where \citet{xu2024sparsity} serves as an oracle since it provides both group information and degenerate dimension, which is learned in our method. \citet{xu2024sparsity} can be seen as a special case which can be explained in our theorem and since it provides stronger information so we do not aim to outperform it. \citet{kivva2022identifiability} only works for non-degenerate cases, which shows the importance of our method to solve the unexplored problem in the degenerate latent space.

\begin{table}[H]
\centering
\label{tab: two bench}
\caption{Ablation study without sparsity constraints and baselines: Average MCCs over 5 random seeds. 
Ours = proposed method, WO = without sparsity constraints, Oracle \citep{xu2024sparsity}, VaDE \citep{kivva2022identifiability}}

\begin{subtable}{\textwidth}
\centering
\begin{tabular}{c|cccc}
\hline
$n$ & 5 & 10 & 20 & 40 \\
\hline
Ours & 0.96$\pm$0.01 & 0.97$\pm$0.01 & 0.92$\pm$0.05  & 0.94$\pm$0.03 \\
WO & 0.57$\pm$0.02 & 0.57$\pm$0.04 & 0.44$\pm$0.01 & 0.36$\pm$0.01 \\
Oracle & 0.91$\pm$0.05 & 0.98$\pm$0.00 & 0.96$\pm$0.01 & 0.94$\pm$0.02 \\
VaDE & 0.45$\pm$0.05 & 0.32$\pm$0.03 & 0.23$\pm$0.15 & 0.25$\pm$0.11 \\
\hline
\end{tabular}
\subcaption{Varying $n$}
\end{subtable}

\vspace{0.3cm}

\begin{subtable}{\textwidth}
\centering
\begin{tabular}{c|cccc}
\hline
$k$ & 0 & 1 & 2 & 3 \\
\hline
Ours & 0.97$\pm$0.01 & 0.97$\pm$0.01 & 0.95$\pm$0.01 & 0.91$\pm$0.04 \\
WO & 0.56$\pm$0.04 & 0.57$\pm$0.04 & 0.54$\pm$0.01 & 0.55$\pm$0.02 \\
Oracle & 0.98$\pm$0.01 & 0.98$\pm$0.00 & 0.95$\pm$0.04 & 0.95$\pm$0.03 \\
VaDE & 0.33$\pm$0.05 & 0.32$\pm$0.03 & 0.31$\pm$0.03 & 0.27$\pm$0.03 \\
\hline
\end{tabular}
\subcaption{Varying $k$}
\end{subtable}

\vspace{0.3cm}

\begin{subtable}{\textwidth}
\centering
\begin{tabular}{c|ccc}
\hline
$m$ & 3 & 10 & 20 \\
\hline
Ours & 0.96$\pm$0.03 & 0.97$\pm$0.01 & 0.93$\pm$0.01 \\
WO & 0.55$\pm$0.01 & 0.57$\pm$0.04 & 0.56$\pm$0.02 \\
Oracle & 0.99$\pm$0.00 & 0.98$\pm$0.00 & 0.95$\pm$0.01 \\
VaDE & 0.33$\pm$0.04 & 0.32$\pm$0.03 & 0.35$\pm$0.07 \\
\hline
\end{tabular}
\subcaption{Varying $m$}
\end{subtable}

\vspace{0.3cm}

\begin{subtable}{\textwidth}
\centering
\begin{tabular}{c|ccc}
\hline
$\rho$ & 1 & 50 & 75 \\
\hline
Ours & 0.45$\pm$0.04 & 0.97$\pm$0.01 & 0.93$\pm$0.04 \\
WO & 0.42$\pm$0.03 & 0.57$\pm$0.04 & 0.56$\pm$0.03 \\
Oracle & 0.86$\pm$0.04 & 0.98$\pm$0.00 & 0.98$\pm$0.01 \\
VaDE & 0.28$\pm$0.03 & 0.32$\pm$0.03 & 0.31$\pm$0.04 \\
\hline
\end{tabular}
\subcaption{Varying $\rho$}
\end{subtable}

\vspace{0.3cm}

\begin{subtable}{\textwidth}
\centering
\begin{tabular}{c|cccc}
\hline
$\delta$ & 0 & 1 & 2 & 3 \\
\hline
Ours & 0.97$\pm$0.01 & 0.88$\pm$0.04 & 0.61$\pm$0.04 & 0.49$\pm$0.02 \\
WO & 0.57$\pm$0.04 & 0.54$\pm$0.02 & 0.55$\pm$0.04 & 0.55$\pm$0.03 \\
Oracle & 0.98$\pm$0.00 & 0.73$\pm$0.16 & 0.41$\pm$0.04 & 0.38$\pm$0.02 \\
VaDE & 0.32$\pm$0.03 & 0.30$\pm$0.02 & 0.27$\pm$0.03 & 0.28$\pm$0.03 \\
\hline
\end{tabular}
\subcaption{Varying $\delta$}
\end{subtable}

\vspace{0.3cm}

\begin{subtable}{\textwidth}
\centering
\begin{tabular}{c|cccc}
\hline
$\theta$ & 0 & 15 & 30 & 45 \\
\hline
Ours & 0.97$\pm$0.01 & 0.83$\pm$0.02 & 0.60$\pm$0.02 & 0.62$\pm$0.01 \\
WO & 0.57$\pm$0.04 & 0.56$\pm$0.03 & 0.58$\pm$0.04 & 0.56$\pm$0.02 \\
Oracle & 0.98$\pm$0.00 & 0.85$\pm$0.01 & 0.62$\pm$0.02 & 0.63$\pm$0.01 \\
VaDE & 0.32$\pm$0.03 & 0.32$\pm$0.03 & 0.31$\pm$0.04 & 0.32$\pm$0.03 \\
\hline
\end{tabular}
\subcaption{Varying $\theta$}
\end{subtable}

\end{table}

\subsection{Comparison with Baseline VaDE without misspecification}
\label{app: without misspecifiction}

We provide results where we evaluate our method in a setting that is fair to the baseline VaDE, where we generate the latent with a non-degenerate GMM. In this setting, we can only compare the identifiability up to $AT$ measured by metric $R^2$, but we cannot compare for $PS$ for any of the two methods, because our method needs degeneracy, and VaDE needs an auxiliary variable and pairwise conditional independence. As shown below, our method outperforms VaDE for all settings.

\begin{table}[H]
\caption{Results for the numerical experiments (Stage 1) on nondegenerate GMM data.
The bold font indicates which parameters are varying in each block of rows.}
\label{tab: numerical nonde}
\begin{center}
\begin{small}
\begin{sc}
\begin{tabular}{cccccccc}
\hline
$n$ & $k$ & $m$ &  $\rho$ &  $\delta$ &  $\theta$ & \textbf{MCC} of VaDE   & \textbf{MCC} of Our Method  \\ \hline
\textbf{5}    & 1 & 10  & 50 \%  & 0 & 0 &    0.66$\pm$0.09  & 0.94$\pm$0.01 \\
\textbf{10}   & 1 & 10  & 50 \%  & 0 & 0 &  0.78$\pm$0.05 & 0.94$\pm$0.01 \\
\textbf{20}   & 1 & 10  & 50 \%  & 0 & 0 &   0.86$\pm$0.04 & 0.93$\pm$0.01 \\
\textbf{40}   & 1 & 10  & 50 \%  & 0 & 0 & 0.84$\pm$0.07  & 0.93$\pm$0.01 \\
\hline
10    & \textbf{0} & 10  & 50 \% & 0 & 0 &  0.86$\pm$0.05  & 0.95$\pm$0.01\\
10    & \textbf{1} & 10 & 50 \%  & 0 & 0 &  0.78$\pm$0.05 & 0.94$\pm$0.01\\
10   &  \textbf{2} & 10 & 50 \%  & 0 & 0 &  0.74$\pm$0.12 & 0.94$\pm$0.01 \\
10   &  \textbf{3} & 10 & 50 \%  & 0 & 0 &  0.69$\pm$0.11 & 0.93$\pm$0.02\\
\hline
10   &  1 & \textbf{3}  & 50 \%  & 0 & 0 &  0.78$\pm$0.10 & 0.99$\pm$0.00\\
10   &  1 & \textbf{10}  & 50 \% & 0 & 0 & 0.78$\pm$0.05  & 0.94$\pm$0.01 \\
10   &  1 & \textbf{20}  & 50 \% & 0 & 0 & 0.70$\pm$0.10 & 0.89$\pm$0.01\\
\hline
10    & 1 & 10 & \textbf{1\text{var}}  & 0 &0 & 0.77 $\pm$0.05 & 0.94$\pm$0.01 \\
10   &  1 & 10 & \textbf{50 \%}   & 0 &0 &  0.78$\pm$0.05 & 0.94$\pm$0.01\\
10   &  1 & 10 & \textbf{75 \%}  & 0 &0 &      0.75 $\pm$0.08  & 0.94$\pm$0.02 \\
\hline
10   &  1 & 10  & 50 \%  & 0 & \textbf{0} &  0.78$\pm$0.05 & 0.94$\pm$0.01\\
10   &  1 & 10  & 50 \% & 0 & \textbf{15} & 0.77 $\pm$0.01  & 0.94$\pm$0.01 \\
10   &  1 & 10  & 50 \% & 0 & \textbf{30} & 0.78$\pm$0.02 & 0.94$\pm$0.02\\
10   &  1 & 10  & 50 \% & 0 & \textbf{45} & 0.78$\pm$0.02 & 0.94$\pm$0.01 \\
\hline
\end{tabular}
\end{sc}
\end{small}
\end{center}
\end{table}

\subsection{Over-parameterization of latent dimension $n$}
\label{subapp: over n}

In our previous analysis, the latent dimension $n$ is treated as fixed. To the best of our knowledge, prior work typically assumes $n$ is known in theoretical analysis, but then relies on over-parameterization strategies in empirical analysis. We also consider an over-parameterized choice and show empirically that our method performs well under these settings. While a formal proof of identifiability for remains open, our experimental findings suggest that moderate over-specification does not undermine identifiability guarantees in practice.  We use our basic setup: $n=10$, $k=1$, $m=10$, $\rho=50$, $\delta=0$, $\theta=0$, varying estimation $\hat{n}\in\{12, 14, 16,18, 20\}$.
\begin{table}[h!]
\centering
\caption{Average R2 and MCC over 5 random seeds}
\begin{tabular}{lccccc}
\hline
 $\hat{n}$& 12(n+2) & 14(n+4) & 16(n+6) & 18(n+8) & 20(n+10) \\
\hline
R2  & 0.91 & 0.89 & 0.88 & 0.88 & 0.83 \\
MCC & 0.96 & 0.96 & 0.93 & 0.96 & 0.96 \\
\hline
\end{tabular}
\end{table}

\subsection{Causal graph learning analysis}
\label{subapp: graph} 

The aim of causal representation learning (CRL) is to recover latent causal variables from high-dimensional observations. Once these causal variables are recovered, one can learn the equivalence class of the causal graph on these variables, e.g. by applying a standard causal discovery method on top of these recovered variables.

We provide results in which we apply a standard causal discovery method, PC \citep{spirtes2000causation}, on our representations. We consider the same setting as in \cref{tab: numerical_full}. We generate $5000$ samples for each graph. We use the implementation from the pcalg package and use partial correlation tests with a significance threshold $\alpha=0.01$.

In Tab.~\ref{tab: shd}, we report the average SHD on the CPDAG learned over our $\hat{\zb}$ w.r.t. the oracle CPDAG (the ground truth CPDAG with oracle tests). We report the average SHD on the CPDAG learned over the ground truth variables $\zb$ w.r.t. the oracle CPDAG. This is the best result we could obtain with a perfect reconstruction of the ground truth variables. We report the delta between them as an alternative measure of the accuracy of our reconstruction of the causal variables.

\begin{table}[H]
\caption{The average SHD on the CPDAG learned over our $\hat{\zb}$ and $\zb$ w.r.t. the oracle CPDAG.}
\label{tab: shd}
\begin{center}
\begin{small}
\begin{sc}
\begin{tabular}{ccccccccc}
\hline
$n$ & $k$ & $m$ &  $\rho$ &  $\delta$ &  $\theta$ & \textbf{SHD} ($\zb$)   &  \textbf{SHD} ($\hat{\zb}$)  & $\Delta$\textbf{SHD} \\ \hline
\textbf{5}    & 1 & 10  & 50 \%  & 0 & 0 &    1.4$\pm$1.67  & 3.8$\pm$1.79 & 2.4 \\
\textbf{10}   & 1 & 10  & 50 \%  & 0 & 0 &  1.8$\pm$1.79 & 4.2$\pm$1.10 & 2.4\\
\textbf{20}   & 1 & 10  & 50 \%  & 0 & 0 &   7.8$\pm$2.86 & 22$\pm$12.90 & 14.2\\
\textbf{40}   & 1 & 10  & 50 \%  & 0 & 0 & 13.6$\pm$4.04  & 39.8$\pm$17.75 & 26.2\\
\hline
10    & \textbf{1} & 10 & 50 \%  & 0 & 0 &  1.8$\pm$1.79 & 4.2$\pm$1.10 & 2.4\\
10   &  \textbf{2} & 10 & 50 \%  & 0 & 0 &  13.2$\pm$2.17 & 16$\pm$2.64 & 2.8\\
10   &  \textbf{3} & 10 & 50 \%  & 0 & 0 &  25$\pm$1.87 & 26.6$\pm$3.85 & 1.6\\
\hline
10   &  1 & \textbf{3}  & 50 \%  & 0 & 0 &  1.8$\pm$1.79 & 3$\pm$1.22 & 1.2\\
10   &  1 & \textbf{10}  & 50 \% & 0 & 0 & 1.8$\pm$1.79  & 4.2$\pm$1.10 & 2.4\\
10   &  1 & \textbf{20}  & 50 \% & 0 & 0 & 1.8$\pm$1.79 & 8.8$\pm$6.05 & 7\\
\hline
10    & 1 & 10 & \textbf{1\text{var}}  & 0 &0 & 1.8$\pm$1.79 & 21.4$\pm$3.05 & 19.6\\
10   &  1 & 10 & \textbf{50 \%}   & 0 &0 &  1.8$\pm$1.79 & 4.2$\pm$1.10 & 2.4\\
10   &  1 & 10 & \textbf{75 \%}  & 0 &0 &  1.8$\pm$1.79  & 7.4$\pm$4.04 & 5.6\\
\hline
10   & 1 & 10 & 50 \%   & \textbf{0} & 0 &  1.8$\pm$1.79 & 4.2$\pm$1.10 &2.4 \\
10   & 1 & 10 & 50 \%   & \textbf{1} & 0 &  1.8$\pm$1.79 & 16.8$\pm$5.12 & 15\\
10   & 1 & 10 & 50 \%   & \textbf{2} & 0 &  1.8$\pm$1.79 & 21.8$\pm$2.59 & 20\\
10   & 1 & 10 & 50 \%   & \textbf{3} & 0 &  1.8$\pm$1.79  & 22.2$\pm$2.68 & 20.4\\
\hline
10   &  1 & 10  & 50 \%  & 0 & \textbf{0} &  1.8$\pm$1.79  & 4.2$\pm$1.10 & 2.4\\
10   &  1 & 10  & 50 \% & 0 & \textbf{15} & 2.2$\pm$1.24  & 18$\pm$2.83 & 15.8\\
10   &  1 & 10  & 50 \% & 0 & \textbf{30} & 2.6$\pm$2.41 & 20.8$\pm$1.64 &18.2\\
10   &  1 & 10  & 50 \% & 0 & \textbf{45} & 4.5$\pm$0.71 & 21.4$\pm$2.30 & 16.9\\
\hline
\end{tabular}
\end{sc}
\end{small}
\end{center}
\end{table}

These results show that our learned causal graph is similar to the learned graphs on the ground truth variables.

\subsection{Image dataset: Multiple Balls}
\label{app: balls}
\paragraph{Data generation process }We use PyGame %
to render images with size $64 \times 64 \times 3$ as shown in \cref{fig: balls varying number from 2 to 6}.
The number of balls $\bb$ varies from 2 to 6. In this setting, the latent variables we consider are the $(x, y)$ coordinates of each ball with support $[0,1]^2$. We sample the $(x, y)$ coordinates of each ball $i\in[\bb]$ independently from a truncated 2-dimensional normal distribution $\Ncal(\mub_i,\Sigmab_i)$ bounded within the square $(0.05,0.95)^2$ with randomized $\mub_i$ and fixed $\Sigmab_i$: $\mub_i\sim \text{Unif}(0.3,0.7)^2$, and $\Sigmab_i= ((0.01,0.00)(0.00,0.01))$. To avoid the high ratio of occultation, we use rejection sampling to sample $\mub_i$ to ensure the Euclidean distance between any pair of $\mub_i$ values is at least $0.2$. To model the degeneracy in this setting, we set each ball at a fixed position $\mub_i$ with probability $p=0.1$ by sampling a binomial variable to indicate the degeneracy status.
\begin{figure}[H]
\centering
\includegraphics[width=0.8\textwidth]{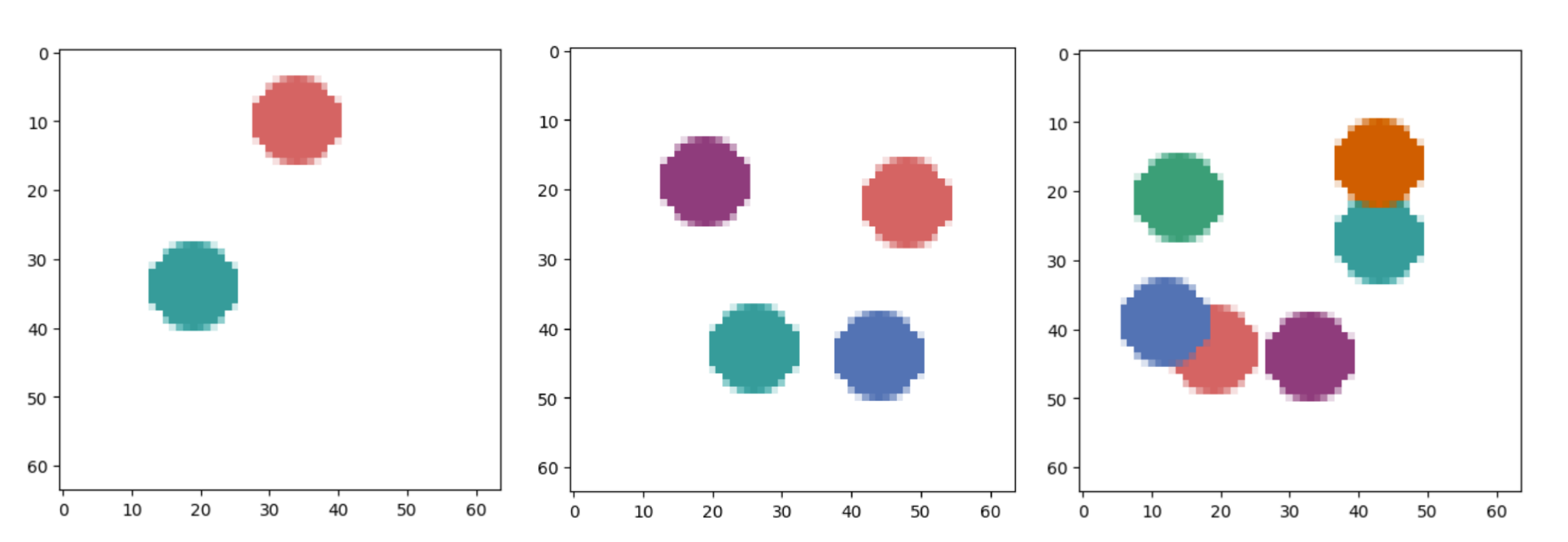}
\caption{Example images of multiple ball dataset, modified based on the rendering code provided by \citet{ahuja2022weakly}. Balls can move freely inside the frame, and each ball's coordinates can be fixed randomly to a constant value, which causes their latent dimensions to be degenerate; Varying the number of balls from 2 to 6.}
\label{fig: balls varying number from 2 to 6}   
\end{figure}

\paragraph{CNN architecture}
Before we train the model described in Sec.~\ref{sec: implementation}, we pre-train a CNN model to reduce the input image dimension from $ 64\times 64 \times 3$ to a vector of size $128\times1$. To control the information loss, we use a symmetric convolutional autoencoder to reconstruct the images. All convolutional layers use kernel size 3 with stride 1 and padding 1, and a ReLU activation follows each convolutional or transposed convolutional layer, except for the final output layer which uses a Sigmoid activation function. We list the concrete network structure in \cref{tab: cnn architecture}.
\begin{table}[H]
\label{tab: cnn architecture}
\centering
\caption{Architecture of the convolutional autoencoder used in Sec.~\ref{sec:multiballexp} (input size: $3 \times 64 \times 64$).}
\begin{tabular}{l|l|c}
\hline\hline
\textbf{Layer Type} & \textbf{Configuration} & \textbf{Output Shape} \\
\hline
Input               & RGB image              & (3, 64, 64)            \\
\hline
\multicolumn{3}{c}{\textbf{Encoder}} \\
\hline
Conv2D + ReLU       & 3 $\rightarrow$ 32, kernel=3, stride=1, pad=1   & (32, 64, 64) \\
MaxPool2D           & 2$\times$2                                      & (32, 32, 32) \\
Conv2D + ReLU       & 32 $\rightarrow$ 64, kernel=3, stride=1, pad=1  & (64, 32, 32) \\
MaxPool2D           & 2$\times$2                                      & (64, 16, 16) \\
Conv2D + ReLU       & 64 $\rightarrow$ 64, kernel=3, stride=1, pad=1  & (64, 16, 16) \\
MaxPool2D           & 2$\times$2                                      & (64, 8, 8)   \\
Conv2D + ReLU       & 64 $\rightarrow$ 128, kernel=3, stride=1, pad=1 & (128, 8, 8)  \\
MaxPool2D           & 2$\times$2                                      & (128, 4, 4)  \\
Conv2D + ReLU       & 128 $\rightarrow$ 128, kernel=3, stride=1, pad=1& (128, 4, 4)  \\
MaxPool2D           & 4$\times$4                                      & (128, 1, 1)  \\
\hline
\multicolumn{3}{c}{\textbf{Decoder}} \\
\hline
ConvTranspose2D + ReLU     & 128 $\rightarrow$ 128, kernel=2, stride=2       & (128, 2, 2)  \\
ConvTranspose2D + ReLU     & 128 $\rightarrow$ 64, kernel=2, stride=2        & (64, 4, 4)   \\
ConvTranspose2D + ReLU     & 64 $\rightarrow$ 64, kernel=2, stride=2         & (64, 8, 8)   \\
ConvTranspose2D + ReLU     & 64 $\rightarrow$ 32, kernel=2, stride=2         & (32, 16, 16) \\
ConvTranspose2D + Sigmoid  & 32 $\rightarrow$ 3, kernel=4, stride=4          & (3, 64, 64)  \\
\hline\hline
\end{tabular}
\end{table}

To evaluate the performance, we perform linear regression from each representation separately to the corresponding ground-truth ball (x,y)-position. For $\bb=2,4,6$, we provide the average $R^2$ over 5 random seeds for each ball in \cref{tab: multi_balls_2,tab: multi_balls_4,tab: multi_balls_6}. The results show that for the settings $\bb=2,4$, the position of each ball can be well recovered through pairs of representations, which is consistent with our theoretical results. However, as shown in the example images, when $\bb=6$, even though we use rejection sampling to ensure the distance between the mean values of any pair of balls is at least 0.2, due to the limited space of the frame, a high portion of occultation cannot be avoided, which also cause a performance drop as shown in \cref{tab: multi_balls_6}.
\begin{table}[H]
\caption{Results for Multiple Balls with $\bb=2$: averaged $R^2$ of linear regression between ground truth ball positions and learned representations.  }
\label{tab: multi_balls_2}
\begin{center}
\begin{small}
\begin{sc}
\begin{tabular}{ccc}
\hline\hline
    & Ball 1 & Ball 2 \\ 
\hline 
MCC & 0.958 $\pm$ 0.052 & 0.914 $\pm$ 0.148\\
\hline\hline
\end{tabular}
\end{sc}
\end{small}
\end{center}
\end{table}

\begin{table}[H]
\caption{Results for Multiple Balls with $\bb=4$: averaged $R^2$ of linear regression between ground truth ball positions and learned representations. }
\label{tab: multi_balls_4}
\begin{center}
\begin{small}
\begin{sc}
\begin{tabular}{ccccc}
\hline\hline
    & Ball 1 & Ball 2 & Ball 3 & Ball 4 \\ 
\hline 
MCC & 0.953 $\pm$ 0.033 & 0.853 $\pm$ 0.194 & 0.932 $\pm$ 0.041 & 0.880 $\pm$ 0.137\\
\hline\hline
\end{tabular}
\end{sc}
\end{small}
\end{center}
\end{table}

\begin{table}[H]
\caption{Results for Multiple Balls with $\bb=6$: averaged $R^2$ of linear regression between ground truth ball positions and learned representations. }
\label{tab: multi_balls_6}
\begin{center}
\begin{small}
\begin{sc}
\begin{tabular}{ccccccc}
\hline\hline
    & Ball 1 & Ball 2 & Ball 3 \\
\hline 
MCC & 0.743 $\pm$ 0.106 & 0.105 $\pm$ 0.046 & 0.293 $\pm$ 0.229 \\
\hline\hline
    & Ball 4 & Ball 5 & Ball 6 \\ 
\hline 
MCC & 0.407 $\pm$ 0.375 & 0.474 $\pm$ 0.183 & 0.595 $\pm$ 0.241\\
\hline\hline
\end{tabular}
\end{sc}
\end{small}
\end{center}
\end{table}

\subsection{Identify $x-$ and $y-$positions with fined-grained masks}
we can define finer-grained sparsity patterns. 
In particular, our method can disentangle the  $x-$ and $y-$positions of each ball, if the dataset allows freezing $x-$ and $y-$positions individually. 

We ran experiments on 5 random seeds with 2 balls (in total 4 latent variables), and got an average MCC = $0.97\pm0.03$. In \cref{tab: R2 fined grained}, we show that the $R^2$ of linear regression between ground truth ball positions $\bb_1^x$, $\bb_1^y$, $\bb_2^x$, $\bb_2^y$, and learned representations $\hat{\zb}_1$,...,$\hat{\zb}_4$. As we can see, both $x-$ and $y-$ can be disentangled well.

\begin{table}[H]
\centering
\label{tab: R2 fined grained}
\caption{Example results of $\hat{z}_i$ assignments across blocks}
\begin{tabular}{|c|cccc|}
\hline
 & $\bb_1^x$ & $\bb_1^y$ & $\bb_2^x$ & $\bb_2^y$ \\
\hline
$\hat{\zb}_1$ & 0.03 & 0.00 & \textbf{0.98} & 0.00 \\
$\hat{\zb}_2$ & \textbf{0.95} & 0.01 & 0.01 & 0.00 \\
$\hat{\zb}_3$ & 0.00 & 0.00 & 0.00 & \textbf{0.97} \\
$\hat{\zb}_4$ & 0.02 & \textbf{0.99} & 0.00 & 0.01 \\
\hline
\end{tabular}
\end{table}

\subsection{Compute information}
\label{subapp: compute}
The experiments in \cref{sec:numericalexp} are implemented on NVIDIA A100 GPU. The experiments in \cref{sec:multiballexp} on multiple balls are implemented on NVIDIA A40 GPU. Each job used 1 GPU per node and 20 GB of memory.

\section{Discussion on Assumptions}
\label{app: assump}
In this section, we assess the applicability and robustness of our assumptions and discuss to what extent they hold or may fail in practical scenarios. We begin with empirical evidence on violations of Gaussianity and piecewise affine mixing, followed by a closer examination of each theoretical assumption, supported by proofs and illustrative examples.

\subsection{Parametric assumptions: Gaussian latent and piecewise affine mixing}

\label{app:misspecified}

Mixture models capture a significantly broader range of latent structures than standard single-component (e.g., standard Gaussian) priors used in many prior works (e.g. VAE and $\beta$ VAE). In fact, mixture models, especially GMMs, are commonly used to approximate real-world latent spaces in applications. Our use of a degenerate mixture is especially relevant in high-dimensional sparse regimes, where not all latent variables are active in every sample.

The piecewise affine mixing function assumption aligns with a broad class of practical transformations (e.g., ReLU networks). Moreover, it can approximate any diffeomorphism functions up to arbitrary precision given a sufficient number of pieces.

In addition, to test the robustness of our method, we report some results for violations of these assumptions in \cref{tab: numerical miss}. We experiment with the violation of 1) the Gaussian assumption (replaced by Exponential and Gumbel), and 2) the piecewise affine (replace LeakyReLU by 1. smooth LeakyReLU 2. Sigmoid ). Except for the Sigmoid setting, we did not observe a performance drop, indicating the robustness of our method under certain violations of our assumptions.

\begin{table}[H]
\caption{Results for the numerical experiments (Stage 2) on misspecified settings.}
\label{tab: numerical miss}
\begin{center}
\begin{small}
\begin{sc}
\begin{tabular}{cccccccc}
\hline
$n$ & $k$ & $m$ & $\rho$ 
& \multicolumn{4}{c}{\textbf{MCC}} \\ 
 &  &  &  
 & \textbf{Exponential} 
 & \textbf{Gumbel} 
 & \textbf{Smooth LeakyReLU} 
 & \textbf{Sigmoid} \\ \hline
\textbf{5}    & 1 & 10  & 50 \%  &  0.92$\pm$0.06 &  0.95$\pm$0.01 &    0.97$\pm$0.01  & 0.36$\pm$0.02 \\
\textbf{10}   & 1 & 10  & 50 \%  & 0.96$\pm$0.01 &  0.97$\pm$0.01 &  0.98$\pm$0.01 & 0.32$\pm$0.06 \\
\textbf{20}   & 1 & 10  & 50 \%  &  0.94$\pm$0.02 &  0.95$\pm$0.02 &   0.94$\pm$0.09 & 0.29$\pm$0.04 \\
\hline
10    & \textbf{0} & 10  & 50 \% &  0.95$\pm$0.01 &  0.97$\pm$0.01 &  0.95$\pm$0.08  & 0.37$\pm$0.06\\
10    & \textbf{1} & 10 & 50 \%  & 0.96$\pm$0.01 &  0.97$\pm$0.01 &  0.98$\pm$0.01 & 0.32$\pm$0.06\\
10   &  \textbf{2} & 10 & 50 \%  &  0.93$\pm$0.04 &  0.96$\pm$0.02 &  0.97$\pm$0.01 & 0.29$\pm$0.06 \\
10   &  \textbf{3} & 10 & 50 \%  &  0.91$\pm$0.04 &  0.91$\pm$0.06 &  0.93$\pm$0.03 & 0.36$\pm$0.09\\
\hline
10   &  1 & \textbf{3}  & 50 \%  &  0.98$\pm$0.01 &  0.99$\pm$0.01 &  0.99$\pm$0.00 & 0.60$\pm$0.03\\
10   &  1 & \textbf{10}  & 50 \% & 0.96$\pm$0.01 &  0.97$\pm$0.01 & 0.98$\pm$0.01  & 0.32$\pm$0.06 \\
10   &  1 & \textbf{20}  & 50 \% &  0.91$\pm$0.04 &  0.94$\pm$0.02 & 0.97$\pm$0.01 & 0.30$\pm$0.08\\
\hline
10    & 1 & 10 & \textbf{1\text{var}}  &  0.66$\pm$0.08 & 0.59$\pm$0.07 & 0.70 $\pm$0.07 & 0.36$\pm$0.01 \\
10   &  1 & 10 & \textbf{50 \%}   & 0.96$\pm$0.01 &  0.97$\pm$0.01 &  0.98$\pm$0.01 & 0.32$\pm$0.06\\
10   &  1 & 10 & \textbf{75 \%}  &  0.92$\pm$0.04 &  0.94$\pm$0.04 &      0.95 $\pm$0.03  & 0.32$\pm$0.07 \\

\hline
\end{tabular}
\end{sc}
\end{small}
\end{center}
\end{table}

\subsection{\cref{ass: no equ norm}: Genericity of pdGMMs}

\label{discussion:genericity}

In this section, we clarify the genericity of \cref{ass: no equ norm}. We begin by providing a proof that the set of parameters violating this assumption is a measure-zero subset of the parameter space under consideration, i.e., even when we consider only covariance matrices that are degenerate. Then, we introduce an example in two-dimensional space to support and illustrate the intuition behind it. We also extend this example to higher dimensions and conduct an ablation study to show the necessity of this assumption. Finally, we discuss the applicability of this assumption in practice.

 \assequalnorm*
 
\begin{lemma}

Consider a pdGMM in reduced form $\Zb \sim \sum_{j=1}^J \lambda_j N(\mub_j, \Sigmab_j)$.
The set of parameters $(\lambda_j, \mub_j, \Sigmab_j), J\in [J]$ violating \cref{ass: no equ norm} is measure-zero in the parameter space under consideration, i.e., when we consider $\Sigmab_j$ to be degenerate covariance matrices.
\end{lemma}

\begin{proof}

A violation of \cref{ass: no equ norm} means for some $0<k\leq n$, there exists an subset $\Jcal_k^0\in \Jcal_k$, such that all points $\zb_0\in \cap_{j\in \Jcal_k^0}\Zcal_j$, $\exists \, i,j\in \Jcal_k^0$ such that the Mahalonobis distance are the same, i.e.,
$\sqrt{(\zb_0-\mub_i)^T\Sigmab_i^{-1}(\zb_0-\mub_i)}=\sqrt{(\zb_0-\mub_j)^T\Sigmab_j^{-1}(\zb_0-\mub_j)}$.

To constrain the parameter space of $\Sigmab_i$ and $\Sigmab_j$ to the ones with rank $k$, we parameterize $\Sigmab_i=A_iA_i^T$ and $\Sigmab_j=A_jA_j^T$ via full-column-rank decompositions with $\rank(A_i)=\rank(A_j)=k$. Therefore, we have 
\begin{align}
\label{equ: maha to norm2}
    &\sqrt{(\zb_0-\mub_i)^T\Sigmab_i^{-1}(\zb_0-\mub_i)}=||(A_i^TA_i)^{-1}A_i^T(\zb_0-\mub_i)||\\
    &\sqrt{(\zb_0-\mub_j)^T\Sigmab_j^{-1}(\zb_0-\mub_j)}=||(A_j^TA_j)^{-1}A_j^T(\zb_0-\mub_j)||. 
\end{align}

In this way, we can define the subset of parameter space that violates the assumption for each pair of components $i$ and $j$.
\begin{equation}
\label{equ: vio}
    \Vcal^{\Jcal_k^0}_{ij}:=\{(A_i, A_j, \mub_i, \mub_j): ||(A_i^TA_i)^{-1}A_i^T(\zb_0-\mub_i)||=||(A_j^TA_j)^{-1}A_j^T(\zb_0-\mub_j)|| \quad \forall \,\zb_0\in \cap_{j\in \Jcal_k^0}\Zcal_j\}
\end{equation}

The total set of violating parameters for the pdGMM is $\Vcal=\cup_{k=1}^n\cup_{\Jcal_k^0 \subseteq \Jcal_k} \cup_{i,j\in \Jcal_k^0} \Vcal^{\Jcal_k^0}_{ij}$. Since a finite union of measure zero sets is still measure zero, if we can show that the measure of $\Vcal^{\Jcal_k^0}_{ij}$ is 0 for all pairs of components $i, j \in \Jcal_k^0$, the proof is complete.

Let $\fb^{\zb_0}_{ij}:=||(A_i^TA_i)^{-1}A_i^T(\zb_0-\mub_i)||-||(A_j^TA_j)^{-1}A_j^T(\zb_0-\mub_j)||$.
By \citet{mityagin2015zero}, a real analytic function from %
$\mathbb{R}^{m}$ 
to $\mathbb{R}$ for $m\geq 1$ is either zero everywhere or non-zero almost everywhere. We can apply this result to our setting and show that $\fb^{\zb_0}_{ij}$, which is a real analytic function in $(A_i, A_j, \mub_i, \mub_j)$, is non-zero almost everywhere, which means that for each point $\zb_0$ we can distinguish each pair of components. This means that once we can find one nonzero solution, we can conclude that the function is all nonzero almost everywhere. We construct the nonzero solution in this way: we can put the first term in the definition of $\fb^{\zb_0}_{ij}$ to zero by taking $\mub_i=\zb_0$ and using an arbitrary full-column matrix $A_i$. For the second term, we want to show that there exists a non-zero solution. We fix $A_j$ and show that if $\zb_0-\mub_j \not\in Im((A_j^TA_j)^{-1}A_j^T)$, then $\fb^{\zb_0}_{ij}$ is nonzero.

Thus, we conclude based on the result by \citet{mityagin2015zero} that $\fb^{\zb_0}_{ij}$ is non-zero almost everywhere for any $\zb_0$. This means that the set of violations 
$\Vcal^{\Jcal_k^0}_{ij}$ is a measure-zero set for any pair $i,j$ and therefore its finite union $\Vcal$ is a measure-zero set.
\end{proof}

\begin{example}
\label{exa: assum generic}
Consider $n=2$ and $J=2$ with components $N(\mub_1, \Sigmab_1)$ and $N(\mub_2, \Sigmab_2)$. Suppose
\begin{align}
\mub_1=\begin{pmatrix}
0 \\
0
\end{pmatrix}, \quad 
\Sigmab_1=
\begin{pmatrix}
1 & 0 \\
0 & 0
\end{pmatrix}.
\end{align}
There are only two ways to decompose $\Sigmab_1$ into the form $A_1A_1^T$: $A_1=(1,0)^T$ or $(-1,0)^T$. Let $\Sigmab_2=A_2A_2^T$ where $A_2=(a_1, a_2)^T$, and $\mub_2=(\mu_1,\mu_2)$. Consider $\vepsb_1=\vepsb \in \RR$, then $\vepsb_2=\vepsb$ or $-\vepsb$.
To violate
Ass.~\ref{ass: no equ norm}, $(\mub_2, \Sigmab_2)$ can must satisfy:

\begin{align}
\begin{cases}
\vepsb=a_1\vepsb+\mu_1\\
0=a_2\vepsb+\mu_2
\end{cases}
\text{ or }
\begin{cases}
-\vepsb=a_1\vepsb+\mu_1\\
0=a_2\vepsb+\mu_2
\end{cases}
\text{ or }
\begin{cases}
\vepsb=-a_1\vepsb+\mu_1\\
0=-a_2\vepsb+\mu_2
\end{cases}
\text{ or }
\begin{cases}
-\vepsb=-a_1\vepsb+\mu_1\\
0=-a_2\vepsb+\mu_2
\end{cases}
\end{align}
By solving the equations, we can conclude if and only if $(a_1-1)\mu_2=a_2 \mu_1$ or $|(a_1-1)\mu_2|=|a_2 \mu_1|$, Ass.~\ref{ass: no equ norm} is not satisfied, which is a measure-zero subset of parameter space.
\end{example}

We extend this example to dimensions $n=2,5,10$ and show the results of our method on it in \cref{tab: abla1}. As you can see, these results are much worse than the comparable results in \cref{tab: numerical} for $n=5,10$, where the $R^2$ was $\sim 0.93$, and the $MCC$ was $\sim 0.96$, showing why this assumption is needed.

\begin{table}[h]
\label{tab: abla1}
\centering
\caption{Average $R^2$ and MCC over 5 random seeds}
\begin{tabular}{lccc}
\hline
$n$ & 2 & 5 & 10 \\
\hline
$R^2$ & 0.87 & 0.70 & 0.64 \\
MCC & 0.92 & 0.63 & 0.47 \\
\hline
\end{tabular}
\end{table}

In summary, we have shown that the violations of this assumption occur only on a measure-zero subset, supporting its genericity. The two-dimensional example further illustrates this result by explicitly characterizing the exceptional cases, which correspond to highly fine-tuned parameter choices. Together, these arguments support the statement that \cref{ass: no equ norm} is a mild and reasonable assumption, holding almost everywhere in real-world settings.

\subsection{\cref{ass: common basis}: Common basis and translation vector}

\label{discussion:common basis}

This assumption states that the supports of all components should intersect and that a shared basis can represent them. This always holds for non-degenerate GMMs. For degenerate components, a violation can happen because the components are not intersecting, or if their individual bases do not form a global shared basis.

To show the necessity of \cref{ass: common basis}, we provide an example that violates this assumption, showing that such cases only local affine identifiability can be achieved, whereas global identifiability fails.
\begin{example}
\label{exam: local affine to global}
    Consider $n = 2$ with $\Zcal_1 := \textnormal{span}\{\eb^1\}$, $\Zcal_2 := \textnormal{span}\{\eb^2\}$ and $\Zcal_3 = \textnormal{span}\{\eb^1 + \eb^2\}$. In this case, all three subspaces are one-dimensional and the issue is that the family of vectors $\{\eb^1, \eb^2, \eb^1 + \eb^2\}$ is linearly dependent and thus do not form a basis. One can construct a function $f:\cup_{j=1}^3\Zcal_j \rightarrow \RR$ that is affine on each $\Zcal_j$ individually, but not on the union: 
    \begin{align}
        f(\zb) := \begin{cases}
0,\ \text{if}\ \zb \in \Zcal_1\cup \Zcal_2\\
\zb_1,\ \text{if}\ \zb \in \Zcal_3
\end{cases} \,.
\end{align}
Clearly this function is not affine on $\cup_{j=1}^3\Zcal_j$ since 
\begin{align}
    f(\eb^1 + \eb^2) = 1 \not= 0 = f(\eb^1) + f(\eb^2) \,.
\end{align}
\end{example}

Intuitively, \cref{ass: common basis} requires that there exists a unique shared basis for all Gaussian components. We conduct an ablation study to show the necessity. To show a violation of this, we generate the data with a mixture of two different bases: a standard basis and one with rotations. As shown in Tab.~\ref{tab: abla2}, these results are worse than the comparable results in Tab.~\ref{tab: numerical}, which shows why this assumption is needed for AT.

\begin{table}[h]
\centering
\caption{Average $R^2$ over 5 random seeds}
\begin{tabular}{lcccc}
\hline
$k$ & 0 & 1 & 2 & 3 \\
\hline
$R^2$ & 0.83 & 0.82 & 0.84 & 0.85 \\
\hline
\end{tabular}
\label{tab: abla2}
\end{table}

One concrete setting in which this assumption holds is the partially observable scenario in \citet{xu2024sparsity}. In that setting, there are $n$ causal variables, but only a subset of them are “active” in each observation sample, e.g., they are captured in an image, while others are “inactive”. The sets of “active” and “inactive” variables can change for each observation. For example, one could consider a camera that captures objects moving in and out of frame, or are occluded by another object. We consider the “inactive” objects as contributing to the observation with a specific constant value, i.e., their inputs to the mixing functions are constant (e.g. $\mathbf{0}$). In this setting, the intersection of the supports is the vector of all the constant values for each variable and since we are masking the inputs for the “inactive” variables, the individual basis of the supports will always be orthogonal and hence form a well-defined global shared basis, satisfying \cref{ass: common basis}.

Other possible applications could include other high-dimensional applications that rely on sparse representations, such as language models \citep{cunningham2023sparse,gao2024scaling,marks2024sparse}, where we also have “active” and “inactive” concepts.

If it does not hold strictly in a setting, e.g., it only holds for a subset of components, for this subset we can still achieve the same result as \cref{thm:global-affine-identi}, but for the other components we will learn a local linear identified representation individually as stated in \cref{thm:affine-identi-comp}, but not a global one.

\subsection{\cref{ass: standard common basis}a: Common Standard Basis}

\label{discussion:standard common basisa}

The only difference between the Common Standard Basis Assumption (\cref{ass: standard common basis} (a)) and \cref{ass: common basis} is the interception term, which can be canceled out easily by shifting the modeling by a constant in practice. 

Table~\ref{tab: numerical} already shows violations of Ass.~\ref{ass: standard common basis}a when $\delta$ (the masking value) and 
$\theta$ (the rotation of the standard basis) are non-zero, showcasing a decrease in performance when the assumption does not hold.

\subsection{\cref{ass: standard common basis}b: Sufficient Support Basis Index Variability}

\label{discussion:standard common basisb}

The Sufficient Support Basis Index Variability Assumption (\cref{ass: standard common basis} (b)), it is more likely to hold when the sparsity patterns across samples are sufficiently diverse. For example, in the partially observable CRL setting, this would mean that two variables are not always “inactive” and “active” at the same time, since we would otherwise not be able to disentangle them. 

Even if this assumption does not hold strictly in a setting, our framework still supports a weaker form of identifiability: block-wise identifiability, which allows disentanglement across blocks of latent variables, even if variables within each block may remain entangled. This relaxation still enables meaningful interpretability and structure in learned representations.

We ran experiments with different blocks of variables that were always active/inactive together. We use our basic setup: $n=10$, $k=1$, $m=10$, $\rho=50$, $\delta=0$, $\theta=0$ with three types of block to see whether the performance varies with different block sizes: $[[\zb_0, \zb_1], [\zb_2, \zb_3], [\zb_4, \zb_5], [\zb_6, \zb_7], [\zb_8, \zb_9]]$, 
$[[\zb_0], [\zb_1],[\zb_2], [\zb_3, \zb_4, \zb_5], [\zb_6, \zb_7, \zb_8, \zb_9]]$,
$[[\zb_0], [\zb_1],[\zb_2, \zb_3], [\zb_4, \zb_5, \zb_6, \zb_7, \zb_8, \zb_9]]$, and we show $R^2$ of the linear regression between each block and learned representations $\hat{\zb}_1, \dots, \hat{\zb}_10$. The results show that even though the variables are entangled within blocks, disentanglement can be achieved across blocks.

\begin{table}[H]
\centering
\label{tab: block_1}
\caption{Example results of $\hat{\zb}_i$ assignments across blocks}
\begin{tabular}{|c|ccccc|}
\hline
 & $[\zb_0, \zb_1]$ & $[\zb_2, \zb_3]$ & $[\zb_4, \zb_5]$ & $[\zb_6, \zb_7]$ & $[\zb_8, \zb_9]$ \\
\hline
$\hat{\zb}_0$ & \textbf{0.95}  & 0.01 & 0.01 & 0.02 & 0.01 \\
$\hat{\zb}_1$ & \textbf{0.90}  & 0.00 & 0.00 & 0.00 & 0.00  \\
$\hat{\zb}_2$ & 0.01 & \textbf{0.84} & 0.01 & 0.01 & 0.00\\
$\hat{\zb}_3$ & 0.00 & \textbf{0.84} & 0.00 & 0.00 & 0.01\\
$\hat{\zb}_4$ & 0.00 & 0.00 & \textbf{0.91} & 0.00 & 0.00 \\
$\hat{\zb}_5$ & 0.00 & 0.00 & \textbf{0.91} & 0.00 & 0.01 \\
$\hat{\zb}_6$ & 0.00 & 0.00 & 0.00 & \textbf{0.86} & 0.00 \\
$\hat{\zb}_7$ & 0.01 & 0.01 & 0.00 & \textbf{0.96} & 0.01\\
$\hat{\zb}_8$ & 0.01 & 0.00 & 0.00 & 0.01 & \textbf{0.91} \\
$\hat{\zb}_9$ & 0.00 & 0.00 & 0.01 & 0.02 & \textbf{0.88}\\
\hline
\end{tabular}
\end{table}

\begin{table}[H]
\centering
\label{tab: block_2}
\caption{Example results of $\hat{\zb}_i$ assignments across blocks}
\begin{tabular}{|c|ccccc|}
\hline
 & $[\zb_0]$& $[\zb_1]$& $[\zb_2]$ & $[\zb_3, \zb_4, \zb_5]$ & $[\zb_6, \zb_7, \zb_8, \zb_9]$ \\
\hline
$\hat{\zb}_0$ & \textbf{0.95} & 0.00 & 0.00 & 0.01 & 0.03 \\
$\hat{\zb}_1$ & 0.00& \textbf{0.89} & 0.00 & 0.00 & 0.00  \\
$\hat{\zb}_2$ & 0.01 & 0.00 & \textbf{0.92} & 0.01 & 0.02 \\
$\hat{\zb}_3$ & 0.00 & 0.00 & 0.00 & \textbf{0.91} & 0.02 \\
$\hat{\zb}_4$ & 0.01 & 0.00 & 0.00 & \textbf{0.94} & 0.01 \\
$\hat{\zb}_5$ & 0.00 & 0.00 & 0.00 & \textbf{0.92} & 0.00 \\
$\hat{\zb}_6$ & 0.00 & 0.00 & 0.00 & 0.00 & \textbf{0.89} \\
$\hat{\zb}_7$ & 0.01 & 0.00 & 0.01 & 0.01 & \textbf{0.96} \\
$\hat{\zb}_8$ & 0.01 & 0.00 & 0.00 & 0.01 & \textbf{0.92} \\
$\hat{\zb}_9$ & 0.00 & 0.00 & 0.00 & 0.02 & \textbf{0.93} \\
\hline
\end{tabular}
\end{table}

\begin{table}[H]
\centering
\label{tab: block_3}
\caption{Example results of $\hat{\zb}_i$ assignments across blocks}
\begin{tabular}{|c|cccc|}
\hline
 & $[\zb_0]$  & $[\zb_1]$  & $[\zb_2, \zb_3]$& $[\zb_4, \zb_5, \zb_6, \zb_7, \zb_8, \zb_9]$ \\
\hline
$\hat{\zb}_0$ & \textbf{0.92}  & 0.00 & 0.01 & 0.05 \\
$\hat{\zb}_1$ & 0.00  & \textbf{0.89} & 0.00 & 0.01 \\
$\hat{\zb}_2$ & 0.01 & 0.00 & \textbf{0.91} & 0.04\\
$\hat{\zb}_3$ & 0.00 & 0.00 & \textbf{0.88} & 0.01 \\
$\hat{\zb}_4$ & 0.01 & 0.00 & 0.00 & \textbf{0.94}  \\
$\hat{\zb}_5$ & 0.01 & 0.00 & 0.01 & \textbf{0.91}   \\
$\hat{\zb}_6$ & 0.02 & 0.00 & 0.00 & \textbf{0.92}  \\
$\hat{\zb}_7$ & 0.00 & 0.00 & 0.00 & \textbf{0.88} \\
$\hat{\zb}_8$ & 0.01 & 0.00 & 0.00 & \textbf{0.92}  \\
$\hat{\zb}_9$ & 0.00 & 0.00 & 0.00 & \textbf{0.82} \\
\hline
\end{tabular}
\end{table}

\end{document}